\documentclass{article}

     \PassOptionsToPackage{numbers, compress}{natbib}

     \usepackage[final]{neurips_2019}

\usepackage[utf8]{inputenc} 
\usepackage[T1]{fontenc}    
\usepackage{hyperref}       
\usepackage{url}            
\usepackage{booktabs}       
\usepackage{amsfonts}       
\usepackage{nicefrac}       
\usepackage{diagbox}    
\usepackage{microtype}
\usepackage{graphicx}
\usepackage{subfigure}
\usepackage{booktabs} 
\usepackage{amsmath,amssymb}
\usepackage{amsthm}
\newtheorem{lemma}{Lemma}
\newtheorem*{proof*}{Proof}
\newtheorem{theorem}{Theorem}
\newtheorem{example}{Example}
\newtheorem*{condition*}{Monotonicity Condition}
\newtheorem{con}{Condition}

\newtheorem{assumption}{Assumption}

\usepackage{mwe} 
\usepackage{xcolor}
\usepackage{enumitem}
\usepackage{float}
\usepackage[font=footnotesize,skip=2pt,figurename=Fig.]{caption}

\ifodd 0

\newcommand{\rev}[1]{{\color{blue}#1}} 
\newcommand{\mrev}[1]{{\color{magenta}#1}} 
\newcommand{\com}[1]{\textbf{\color{red}(COMMENT: #1)}} 
\newcommand{\mcom}[1]{\textbf{\color{purple}(Response: #1)}} 
\newcommand{\edt}[1]{\textbf{\color{magenta}#1}} 
\newcommand{\clar}[1]{\textbf{\color{green}(NEED CLARIFICATION: #1)}}

\else
\newcommand{\rev}[1]{#1}

\newcommand{\mrev}[1]{#1}
\newcommand{\com}[1]{}
\newcommand{\mcom}[1]{}
\newcommand{\edt}[1]{}
\newcommand{\clar}[1]{}
\fi

\DeclareMathOperator*{\argmin}{arg\,min}

\title{Group Retention when Using Machine Learning in Sequential Decision Making: the Interplay between User Dynamics and Fairness }

\author{%
   Xueru Zhang\thanks{Equal contribution} \\
  University of Michigan, AnnArbor, USA \\
   \texttt{xueru@umich.edu} \\
   \And
   Mohammad Mahdi Khalili\footnotemark[1]\\
   University of Michigan, AnnArbor, USA \\
   \texttt{khalili@umich.edu} \\
   \AND
   Cem Tekin \\
   Bilkent University, Ankara, Turkey\\
   \texttt{cemtekin@ee.bilkent.edu.tr} \\
   \And
   Mingyan Liu\\
   University of Michigan, AnnArbor, USA \\
   \texttt{mingyan@umich.edu} \\
}

\begin{document}

\maketitle

\begin{abstract}
 Machine Learning (ML) models trained on data from multiple demographic groups can inherit representation disparity \cite{pmlr-v80-hashimoto18a} that may exist in the data: the model may be less favorable to groups contributing less to the training process; this in turn can degrade population retention in these groups over time, and exacerbate representation disparity in the long run. In this study, we seek to understand the interplay between ML decisions and the underlying group representation, how they evolve in a sequential framework, and how the use of fairness criteria plays a role in this process. We show that the representation disparity can easily worsen over time under a natural user dynamics (arrival and departure) model when decisions are made based on a commonly used objective and fairness criteria, resulting in some groups diminishing entirely from the sample pool in the long run. It highlights the fact that fairness criteria have to be defined while taking into consideration the impact of decisions on user dynamics. Toward this end, we explain how a proper fairness criterion can be selected based on a general user dynamics model. 
 
\end{abstract}

\section{Introduction}
Machine learning models developed from real-world data can inherit pre-existing bias in the dataset. When
 these models are used to inform decisions involving humans, it may exhibit similar discrimination against sensitive attributes (e.g., gender and race) \cite{accent_bias,Face++,Score}. Moreover, these decisions can influence human actions, such that bias in the decision is then captured in the dataset used to train future models. This closed feedback loop becomes self-reinforcing and can lead to highly undesirable outcomes over time by allowing biases to perpetuate. For example, speech recognition products such as Amazon's Alexa and Google Home are shown to have accent bias against non-native speakers \cite{accent_bias}, with native speakers experience much higher quality than non-native speakers. If this difference leads to more native speakers using such products while driving away non-native speakers, then over time the data used to train the model may become even more skewed toward native speakers, with fewer and fewer non-native samples. Without intervention, the resulting model becomes even more accurate for the former and less for the latter, which then reinforces their respective user experience~\cite{pmlr-v80-hashimoto18a}.

To address the fairness issues, one commonly used approach is to impose fairness criteria such that certain statistical measures (e.g., positive classification rate, false positive rate, etc.) across different demographic groups are (approximately) equalized \cite{barocas-hardt-narayanan}. However, their effectiveness is studied mostly in a static framework,
where only the immediate impact of the learning algorithm is assessed but not their long-term
consequences. Consider an example where a lender decides whether or not to approve a loan application based on the applicant's credit score.  \rev{Decisions satisfying an identical true positive rate (equal opportunity) across different racial groups can make the outcome seem fairer \cite{hardt2016equality}.} However, this can potentially result in more loans issued to less qualified applicants in the group whose score distribution skews toward higher default risk. The lower repayment among these individuals causes their future credit scores to drop, which moves the score distribution of that group further toward high default risk \cite{pmlr-v80-liu18c}. This shows that intervention by imposing seemingly fair decisions in the short term can lead to undesirable results in the long run. 

In this paper we are particularly interested in understanding what happens to group representation over time when models with fairness guarantee are used, and how it is affected when the underlying feature distributions are also affected/reshaped by decisions. 
%
Toward this end, we introduce a user retention model to capture users' reaction (stay or leave) to the decision. We show that under relatively mild and benign conditions, group representation disparity exacerbates over time and eventually the disadvantaged groups may diminish entirely from the system. This condition unfortunately can be easily satisfied when decisions are made based on a typical algorithm (e.g., taking objective as minimizing the total loss) under some commonly used fairness criteria (e.g., statistical parity, equal of opportunity, etc.). Moreover, this exacerbation continues to hold and can accelerate when feature distributions are affected and change over time. A key observation is that if the factors equalized by the fairness criterion do not match what drives user retention, then the difference in (perceived) treatment will exacerbate representation disparity over time. Therefore, fairness has to be defined with a good understanding of how users are affected by the decisions, which can be challenging in practice as we typically have only incomplete/imperfect information. However, we show that if a model for the user dynamics is available, then it is possible to find the proper fairness criterion that mitigates representation disparity. 

The impact of fairness intervention on both individuals and society has been studied in \cite{pmlr-v80-hashimoto18a,heidari2019on,hu2018short,kannan2019downstream,pmlr-v80-liu18c} and \cite{pmlr-v80-hashimoto18a,heidari2019on,pmlr-v80-liu18c} are the most relevant to the present study.  Specifically, \cite{heidari2019on,pmlr-v80-liu18c} focus on the impact on reshaping features over two time steps, while we study the impact on group representation over an infinite horizon. \cite{pmlr-v80-hashimoto18a} studies group representation disparity in a sequential framework but without inspecting the impact of fairness criteria \rev{or considering feature distributions reshaped by decision}. More on related work can be found in Appendix \ref{app_related}.  

The remainder of this paper is organized as follows. Section \ref{sec:problem} formulates the problem. The impact of various fairness criteria on group representation disparity is analyzed and presented in Section \ref{sec:exacerbation}, as well as potential mitigation. Experiments are presented in Section \ref{sec:result}. Section \ref{sec:conclusion} concludes the paper. All proofs and a table of notations can be found in the \rev{appendices.}

\vspace{-0.2cm}
\section{Problem Formulation}\label{sec:problem}
\vspace{-0.2cm}
Consider two demographic groups $G_a$, $G_b$ distinguished based on some sensitive attribute $K\in\{a,b\}$ (e.g., gender, race).  An individual from either group has feature $X\in \mathbb{R}^d$ and label $Y\in \{0,1\}$, both \rev{can be time varying}. Denote by $G_k^j \subset G_k$ the subgroup with label $j$, $j\in \{0,1\}$, $k \in \{a,b\}$, $f_{k,t}^j(x)$ its feature distribution and $\alpha_k^j(t)$ the size of $G_k^j$ as a fraction of the entire population at time $t$. 
Then $\overline{\alpha}_k(t) := \alpha_k^0(t) + \alpha_k^1(t)$ is the size of $G_k$ as a fraction of the population and the difference between $\overline{\alpha}_a(t)$ and $\overline{\alpha}_b(t)$ measures the representation disparity between two groups at time step $t$. Denote by $g_{k,t}^j= \frac{\alpha_k^j(t)}{\overline{\alpha}_k(t)}$ the fraction of label $j\in\{0,1\}$ in group $k$ at time $t$, then the distribution of $X$ over $G_k$ is given by $f_{k,t}(x) = g_{k,t}^1 f_{k,t}^1(x)+g_{k,t}^0 f_{k,t}^0(x)$ and $f_{a,t}\neq f_{b,t}$.

Consider a sequential setting where the decision maker at each time makes a decision on each individual based on feature $x$. Let $h_{\theta}(x)$ be the decision rule parameterized by $\theta\in\mathbb{R}^d$ and $\theta_k(t)$ be the decision parameter for $G_k$ at time $t$, $k\in\{a,b\}$. The goal of the decision maker at time $t$ is to find the best parameters $\theta_a(t)$, $\theta_b(t)$ such that the corresponding decisions about individuals from $G_a$, $G_b$ maximize its utility (or minimize its loss) in the current time. 
Within this context, the commonly studied fair machine learning problem is the one-shot problem stated as follows, at time step $t$:
\begin{eqnarray}
\min_{\theta_a, \theta_b} ~~~ \pmb{O}_t(\theta_a,\theta_b;\overline{\alpha}_a(t),\overline{\alpha}_b(t)) = \overline{\alpha}_a(t) O_{a,t}(\theta_a) + \overline{\alpha}_b(t) O_{b,t}(\theta_b)~~~
\text{s.t.} ~~~ \Gamma_{\mathcal{C},t}(\theta_a,\theta_b)= 0~,\label{eq:opt}
\end{eqnarray}
where $\pmb{O}_t(\theta_a,\theta_b;\overline{\alpha}_a(t),\overline{\alpha}_b(t))$ is the overall objective of the decision maker at time $t$, which consists of sub-objectives from two groups weighted by their group proportions.\footnote{This is a typical formulation if the objective $\pmb{O}_t$ measures the average performance of decisions over all samples, i.e., $\pmb{O}_t = \frac{1}{|G_a|+|G_b|}(\sum_{i\in G_a}O_{t}^i + \sum_{i\in G_b}O_{t}^i) = \frac{1}{|G_a|+|G_b|}(|G_a|O_{a,t} + |G_b|O_{b,t})$, where $O_{t}^i$ measures the performance of each sample $i$ and $O_{k,t} = \frac{1}{|G_k|}\sum_{i\in G_k}O_{t}^i$ is the average performance of $G_k$.} $\Gamma_{\mathcal{C},t}(\theta_a,\theta_b)=0$ characterizes fairness constraint $\mathcal{C}$, which requires the parity of certain statistical measure (e.g., positive classification rate, false positive rate, etc.) across different demographic groups. 
Some commonly used criteria will be elaborated in Section \ref{subsec:prob}. 
Both $O_{k,t}(\theta_k)$ and  $\Gamma_{\mathcal{C},t}(\theta_a,\theta_b)=0$ depend on $f_{k,t}(x)$.  
The resulting solution $(\theta_a(t),\theta_b(t))$ will be referred to as the one-shot fair decision under fairness $\mathcal{C}$, where the optimality only holds for a single time step $t$. 


In this study, we seek to understand how the group representation evolves in a sequential setting over the long run when different fairness criteria are imposed. To do so, the impact of the current decision on the size of the underlying population is modeled by the following discrete-time retention/attrition dynamics. Denote by $N_k(t)\in\mathbb{R}_+$ the expected number of users in group $k$ at time $t$:
\begin{align}\label{eq:dynamic}
N_k(t+1) &= N_k(t) \cdot \pi_{k,t} (\theta_k(t)) + \beta_k ~,\forall k\in\{a,b\}, 
\end{align}
where $\pi_{k,t} (\theta_k(t))$ is the retention rate, i.e., the probability of a user from $G_k$ who was in the system at time $t$ remaining in the system at time $t+1$. \rev{This is assumed to be a  function of the user experience, which could be the actual accuracy of the algorithm or their perceived (mis)treatment.} 
This \rev{experience} is determined by the application and is different under different contexts. For instance, in domains of speaker verification and medical diagnosis, it can be considered as the average loss, i.e., a user stays if he/she can be classified correctly; in loan/job application scenarios, it can be the rejection rates, i.e., user stays if he/she gets approval. $\beta_k$ is the expected number of exogenous arrivals to $G_k$ and is treated as a constant in our analysis, though our main conclusion  holds when this is modeled as a random variable.  Accordingly, the relative group representation for time step $t+1$ is updated as
$\overline{\alpha}_k(t+1)= \frac{N_k(t+1)}{N_a(t+1) + N_b(t+1)} ,\forall k\in\{a,b\}.$

For the remainder of this paper, $\frac{\overline{\alpha}_a(t)}{\overline{\alpha}_b(t)}$ is used to measure the group representation disparity at time $t$. As $\overline{\alpha}_k(t)$ and $f_{k,t}(x)$ change over time, the one-shot problem \eqref{eq:opt} is also time varying.  
In the next section, we examine what happens to $\frac{\overline{\alpha}_a(t)}{\overline{\alpha}_b(t)}$ when one-shot fair decisions are applied in each step.

\section{Analysis of Group Representation Disparity in the Sequential Setting}\label{sec:exacerbation}

Below we present results on the monotonic change of $\frac{\overline{\alpha}_a(t)}{\overline{\alpha}_b(t)}$ when applying one-shot fair decisions in each step. It shows that the group representation disparity can worsen over time and may lead to the extinction of one group under a monotonicity condition stated as follows.

\begin{condition*}\label{con1}
Consider two one-shot problems defined in \eqref{eq:opt} with objectives $\widehat{\pmb{O}}(\theta_a,\theta_b;\widehat{\overline{\alpha}}_{a},\widehat{\overline{\alpha}}_{b})$ and $\widetilde{\pmb{O}}(\theta_a,\theta_b;\widetilde{\overline{\alpha}}_{a},\widetilde{\overline{\alpha}}_{b})$ over distributions $\widehat{f}_{k}(x)$, $\widetilde{{f}}_{k}(x)$ respectively. 
Let $(\widehat{\theta}_{a},\widehat{\theta}_{b})$, $(\widetilde{\theta}_{a},\widetilde{\theta}_{b})$ be the corresponding fair decisions. We say that two problems $\widehat{\pmb{O}}$ and $\widetilde{\pmb{O}}$ satisfy the monotonicity condition given a dynamic model if for any $\widehat{\overline{\alpha}}_{a}+\widehat{\overline{\alpha}}_{b}=1$ and $\widetilde{\overline{\alpha}}_{a}+\widetilde{\overline{\alpha}}_{b}=1$ such that 
$ \frac{\widehat{\overline{\alpha}}_{a}}{\widehat{\overline{\alpha}}_{b}}< \frac{\widetilde{\overline{\alpha}}_{a}}{\widetilde{\overline{\alpha}}_{b}}$,
the resulting retention rates satisfy $\widehat{\pi}_{a}(\widehat{\theta}_{a})<\widetilde{\pi}_{a}(\widetilde{\theta}_{a}) \text{ and }\widehat{\pi}_{b}(\widehat{\theta}_{b})>\widetilde{\pi}_{b}(\widetilde{\theta}_{b})$.
\end{condition*}
Note that this condition is defined over two one-shot problems and a given dynamic model. \rev{It is not limited to specific families of objective or constraint functions; nor is it limited to one-dimensional features.} 
 The only thing that matters is the group proportions within the system and the retention rates determined by the decisions and the dynamics. It characterizes a situation where when one group's representation increases, the decision becomes more in favor of this group and less favorable to the other, so that the retention rate is higher for the favored group and lower for the other.

\begin{theorem}\label{thm1}
[Exacerbation of representation disparity] Consider a sequence of one-shot problems \eqref{eq:opt} with objective $\pmb{O}_{t} (\theta_a,\theta_b;\overline{\alpha}_a(t), \overline{\alpha}_b(t))$ at each time $t$. Let $(\theta_a(t),\theta_b(t))$ be the corresponding solution and 
 $\pi_{k,t}(\theta_k(t))$ be the resulting retention rate of $G_k$, $k\in\{a,b\}$ under a dynamic model \eqref{eq:dynamic}. If the initial states satisfy $\frac{N_a(1)}{N_b(1)} = \frac{\beta_a}{\beta_b}$, $N_k(2)>N_k(1)$,\footnote{This condition will always be satisfied when the system starts from a near empty state.} and one-shot problems in any two consecutive time steps, i.e.,  $\pmb{O}_{t}$, $\pmb{O}_{t+1}$, satisfy monotonicity condition under the given dynamic model, then the following holds.  Let
 $\diamond$ denote
either $``<"$ or $``="$ or $``>"$, if $\pi_{a,1}(\theta_a(1))\diamond \pi_{b,1}(\theta_b(1))$, then $\frac{\overline{\alpha}_a(t+1)}{\overline{\alpha}_b(t+1)}\diamond\frac{\overline{\alpha}_a(t)}{\overline{\alpha}_b(t)}$ and $\pi_{a,t+1}(\theta_a(t+1))\diamond \pi_{a,t}(\theta_a(t))\diamond\pi_{b,t}(\theta_b(t))\diamond\pi_{b,t+1}(\theta_b(t+1)) $, $\forall t$.
\end{theorem}

Theorem \ref{thm1} says that once a group's proportion starts to change (increase or decrease), it will continue to change in the same direction. This is because under the monotonicity condition, there is a feedback loop between representation disparity and the one-shot decisions: the former drives the latter which results in different user retention rates in the two groups, which then drives future representation.

The monotonicity condition can be satisfied under some commonly used objectives, dynamics and fairness criteria.  This is characterized in the following theorem.

\begin{theorem}\label{thm:MC}
\mrev{[A case satisfying monotonicity condition]} Consider two one-shot problems defined in \eqref{eq:opt} with objectives $\widetilde{\mathbf{O}}(\theta_a,\theta_b;\widehat{\overline{\alpha}}_a , \widehat{\overline{\alpha}}_b) = \widehat{\overline{\alpha}}_a O_a(\theta_a)+\widehat{\overline{\alpha}}_b O_b(\theta_b)$ and $\widehat{\mathbf{O}}(\theta_a,\theta_b;\widetilde{\overline{\alpha}}_a , \widetilde{\overline{\alpha}}_b)= \widetilde{\overline{\alpha}}_a O_a(\theta_a)+\widetilde{\overline{\alpha}}_b O_b(\theta_b)$ \mrev{over the same} distribution $f_k(x)$ with $\widehat{\overline{\alpha}}_a + \widehat{\overline{\alpha}}_b=1$ and $\widetilde{\overline{\alpha}}_a + \widetilde{\overline{\alpha}}_b=1$. Let $(\widehat{\theta}_a,\widehat{\theta}_b)$, $(\widetilde{\theta}_a,\widetilde{\theta}_b)$ be the corresponding solutions. Under the condition that $O_k(\widehat{\theta}_k) \neq O_k(\widetilde{\theta}_k)$ for all possible \mrev{$\widehat{\overline{\alpha}}_k\neq  \widetilde{\overline{\alpha}}_k$}, if the dynamics satisfy $\pi_k(\theta_k) = h_k(O_k(\theta_k))$ for some decreasing function $h_k(\cdot)$, then $\widetilde{\mathbf{O}}$ and $\widehat{\mathbf{O}}$ satisfy the monotonicity condition. 
\end{theorem}
\mrev{
The above theorem identifies a class of cases satisfying the monotonicity condition; these are cases where whenever the group proportion changes, the decision will cause the sub-objective function value to change as well, and  the sub-objective function value drives user departure.  

}

\rev{
%
}

For the rest of the paper we will focus on the one-dimensional setting. 
Some of the cases we consider are special cases of Thm \ref{thm:MC} (Sec. \ref{subsec:dynamic}). Others such as the time-varying feature distribution $f_{k,t}(x)$ considered in Sec. \ref{subsec:reshape} also satisfy the monotonicity condition but are not captured by Thm \ref{thm:MC}.



\subsection{The one-shot problem}\label{subsec:prob}

Consider a binary classification problem based on feature $X\in\mathbb{R}$. 
Let decision rule $h_{\theta}(x) = \textbf{1}(x\geq \theta)$ be a threshold policy parameterized by $\theta\in\mathbb{R}$ and $L(y,h_{\theta}(x)) = \textbf{1}(y\neq h_{\theta}(x))$ the 0-1 loss incurred by applying decision $\theta$ on individuals with data $(x,y)$. 

The goal of the decision maker at each time is to find a pair $(\theta_a(t),\theta_b(t))$ subject to criterion $\mathcal{C}$ such that the total expected loss is minimized,
 i.e., $\pmb{O}_t(\theta_a,\theta_b;\overline{\alpha}_a(t),\overline{\alpha}_b(t)) = \overline{\alpha}_a(t) L_{a,t}(\theta_a) + \overline{\alpha}_b(t) L_{b,t}(\theta_b)$, where $L_{k,t}(\theta_k) =
  g_{k,t}^1\int_{-\infty}^{\theta_k}f_{k,t}^1(x)dx + g_{k,t}^0\int_{\theta_k}^{\infty}f_{k,t}^0(x)dx $ is the expected loss $G_k$ experiences at time $t$. Some examples of $\Gamma_{\mathcal{C},t}(\theta_a,\theta_b)$ are as follows and illustrated in Fig. \ref{fig1}.  
 \begin{enumerate}[noitemsep,topsep=0pt]
 	\item 
 	Simple fair (\texttt{Simple}): $\Gamma_{\texttt{Simple},t} = \theta_a-\theta_b$.  Imposing this criterion simply means we ensure the same decision parameter is used for both groups. 
 	\item 
 	Equal opportunity (\texttt{EqOpt}): $\Gamma_{\texttt{EqOpt},t} =	\int_{\theta_a}^{\infty} f_{a,t}^0(x) dx  -  \int_{\theta_b}^{\infty} f_{b,t}^0(x) dx. 
 	$ This requires the false positive rate (FPR) be the same for different groups (Fig. \ref{fig1:c}),\footnote{Depending on the context, this criterion can also refer to equal false negative rate (FNR), true positive rate (TPR), or true negative rate (TNR), but the analysis is essentially the same.} i.e., $\text{Pr}(h_{\theta_a}(X)=1|Y=0,K=a) = \text{Pr}(h_{\theta_b}(X)=1|Y=0,K=b)$. 
 	\item 
 	Statistical parity (\texttt{StatPar}): $	\Gamma_{\texttt{StatPar},t} = \int_{\theta_a}^{\infty} f_{a,t}(x) dx - \int_{\theta_b}^{\infty} f_{b,t}(x) dx$. This requires different groups be given equal probability of being labelled $1$ (Fig. \ref{fig1:b}), i.e., $\text{Pr}(h_{\theta_a}(X) = 1|K=a) = \text{Pr}(h_{\theta_b}(X) = 1|K=b)$. 
 	\item 
 	Equalized loss (\texttt{EqLos}): $\Gamma_{\texttt{EqLos},t} = L_{a,t}(\theta_a) - L_{b,t}(\theta_b)$. This requires that the expected loss across different groups be equal (Fig. \ref{fig1:d}). 
 \end{enumerate} 
 Notice that for \texttt{Simple}, \texttt{EqOpt} and \texttt{StatPar} criteria, \rev{the following holds:} $\forall t$, $(\theta_a,\theta_b)$, and $(\theta_a',\theta_b')$ that satisfy $\Gamma_{\mathcal{C},t}(\theta_a,\theta_b)=\Gamma_{\mathcal{C},t}(\theta_a',\theta_b')= 0$, \rev{we have} $\theta_a \geq \theta_a'$ if and only if $\theta_b \geq \theta_b'$.
\vspace{-0.2cm}
\begin{minipage}{0.72\textwidth}
Some technical assumptions on the feature distributions are in order. We assume
$f_{a,t}^0(x), f_{a,t}^1(x), f_{b,t}^0(x), f_{b,t}^1(x)$ have bounded support on $[\underline{a}_t^{0},\overline{a}_t^{0}]$, $[\underline{a}_t^{1},\overline{a}_t^{1}]$, $[\underline{b}_t^{0},\overline{b}_t^{0}]$ and $[\underline{b}_t^{1},\overline{b}_t^{1}]$ respectively, and that $f_{k,t}^1(x)$ and $f_{k,t}^0(x)$ overlap, i.e., $\underline{a}_t^{0}<\underline{a}_t^{1}<\overline{a}_t^{0}<\overline{a}_t^{1}$ and $\underline{b}_t^{0}<\underline{b}_t^{1}<\overline{b}_t^{0}<\overline{b}_t^{1}$. The main technical assumption is stated as follows.
\end{minipage}
\hspace{0.1cm}
\begin{minipage}{0.25\textwidth}
\includegraphics[trim={0cm 0cm 0cm 0.cm},clip=true,width=\textwidth]{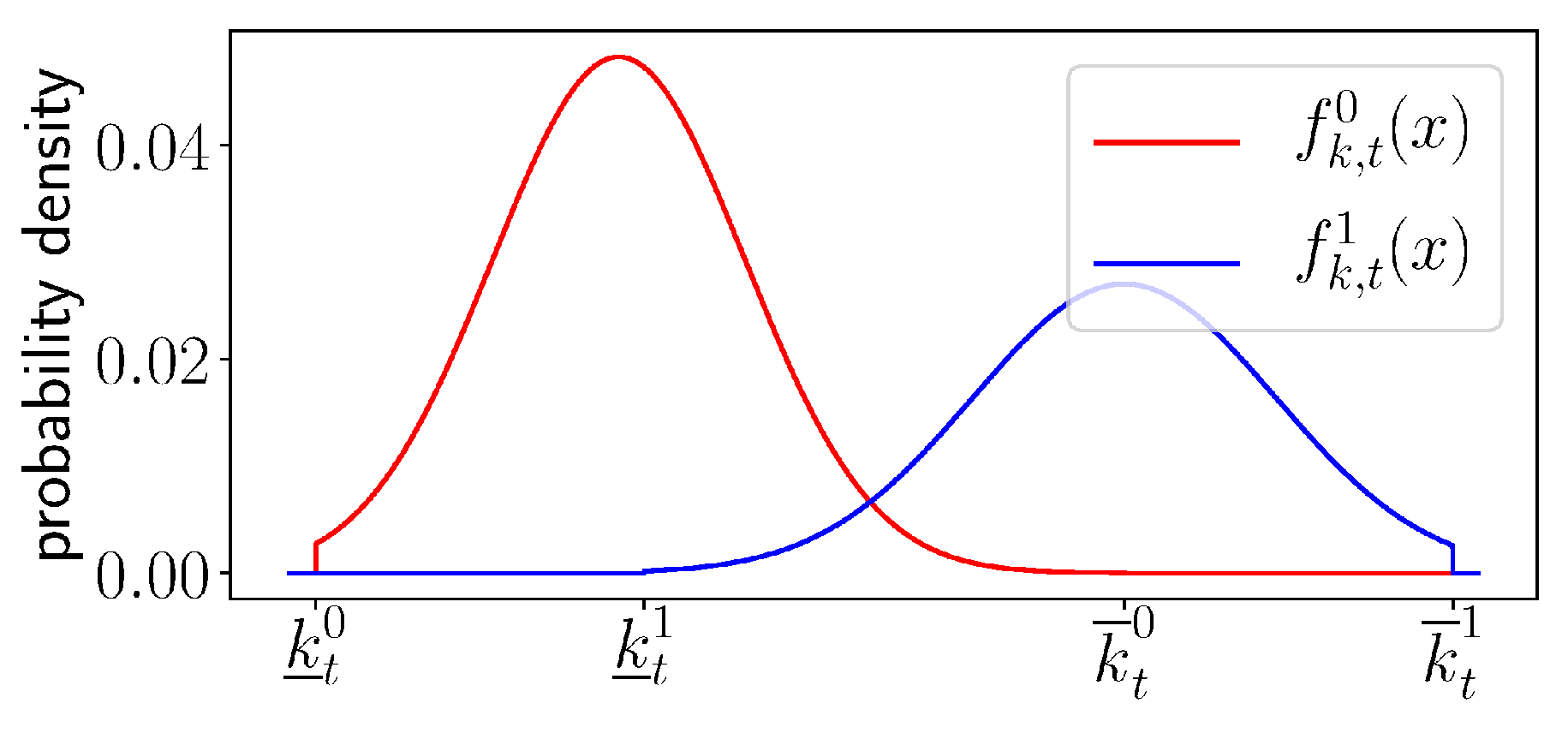}
\captionof{figure}{$f_{k,t}^j(x)$, $k\in\{a,b\}$\label{fig2}}
%
\end{minipage}
\begin{figure}[h]
	\centering   
	\subfigure[each $f_k^j(x)$ for $G_k^j$]{\label{fig1:a}\includegraphics[trim={1.95cm 0cm 2.3cm 0.8cm},clip=true,width=0.245\textwidth]{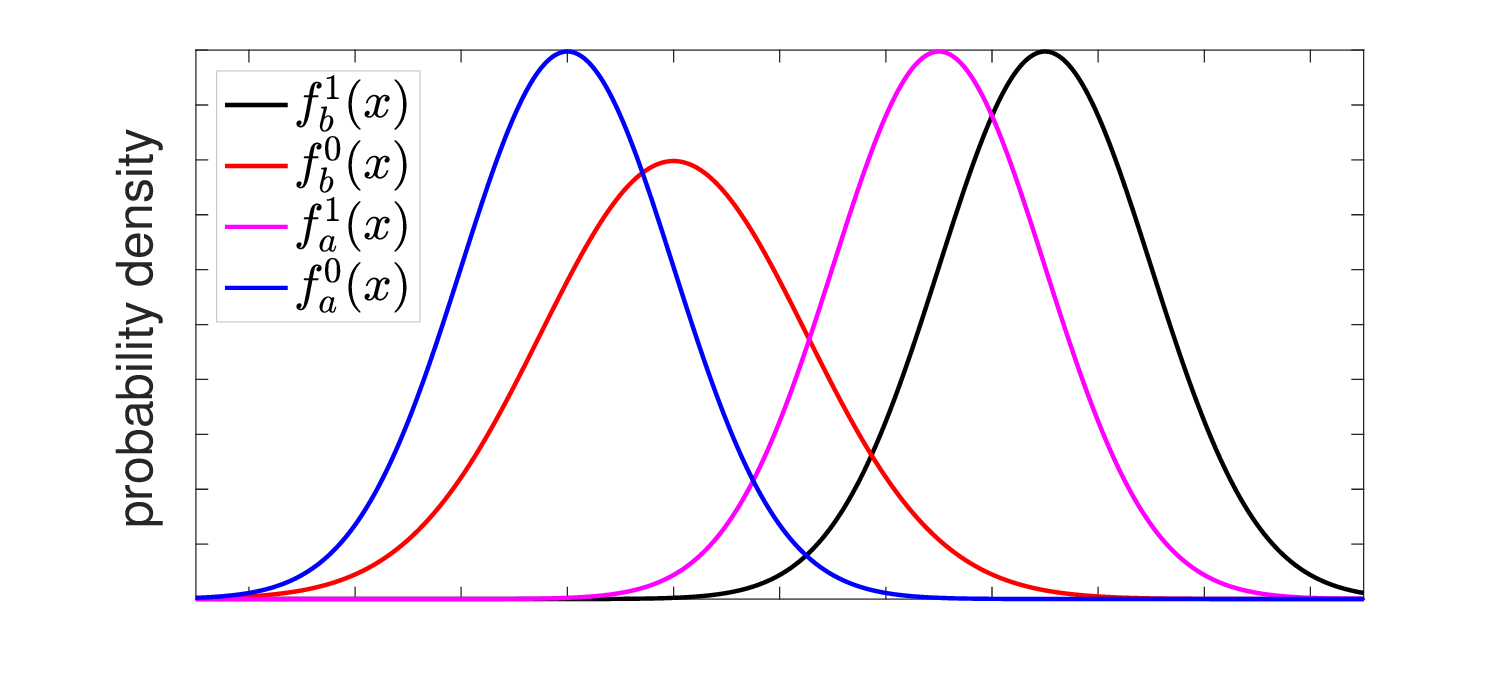}}
	\subfigure[Statistical parity]{\label{fig1:b}\includegraphics[trim={1.95cm 0cm 2.3cm 0.8cm},clip=true,width=0.245\textwidth]{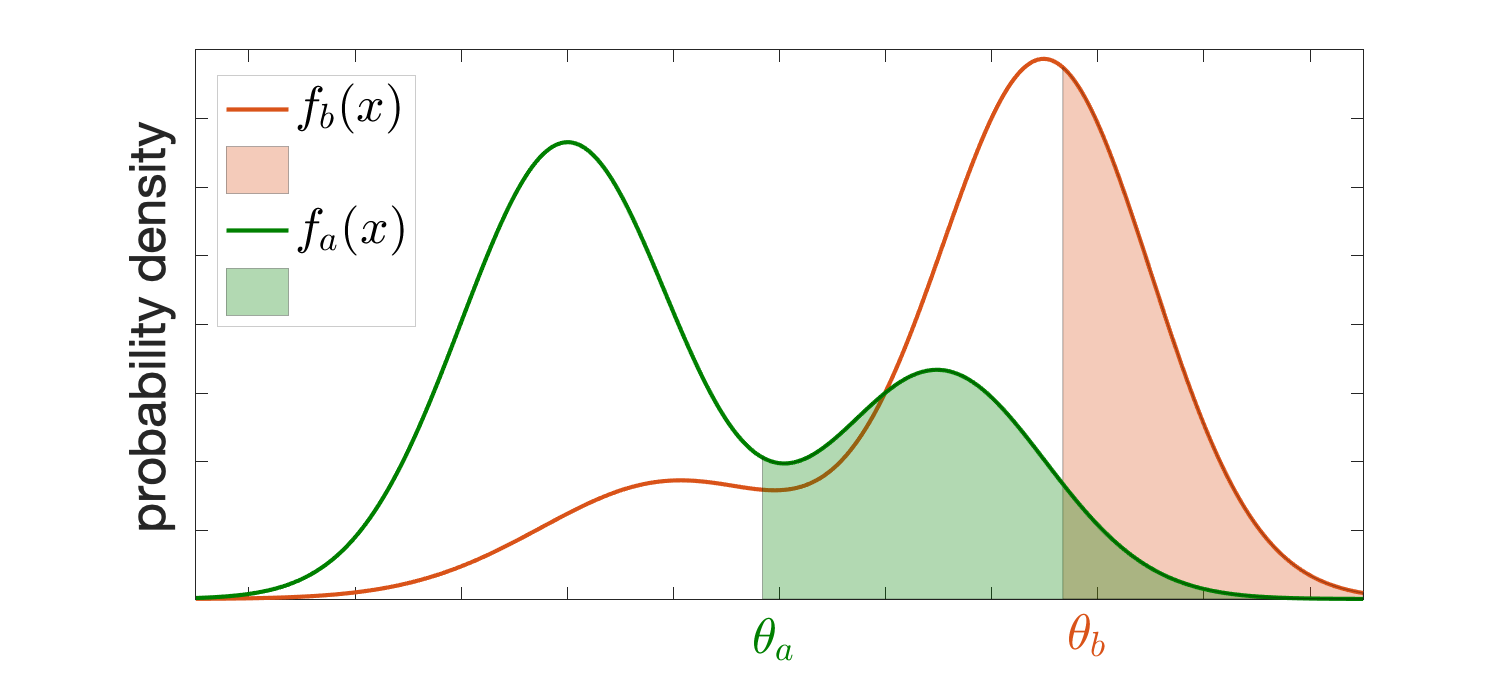}}
	\subfigure[Equal opportunity]{\label{fig1:c}\includegraphics[trim={1.95cm 0cm 2.3cm 0.8cm},clip=true, width=0.245\textwidth]{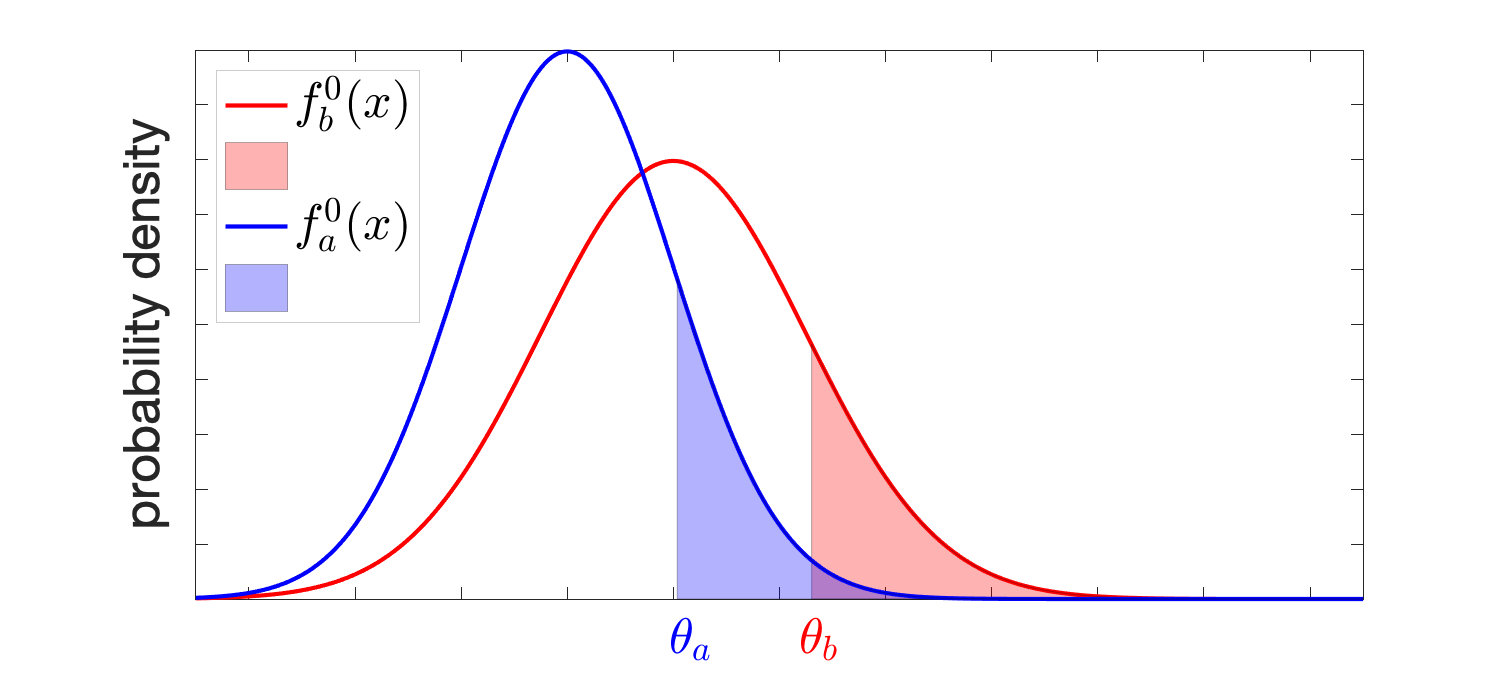}}
	\subfigure[Equalized Loss]{\label{fig1:d}\includegraphics[trim={1.95cm 0cm 2.3cm 0.8cm},clip=true,width=0.245\textwidth]{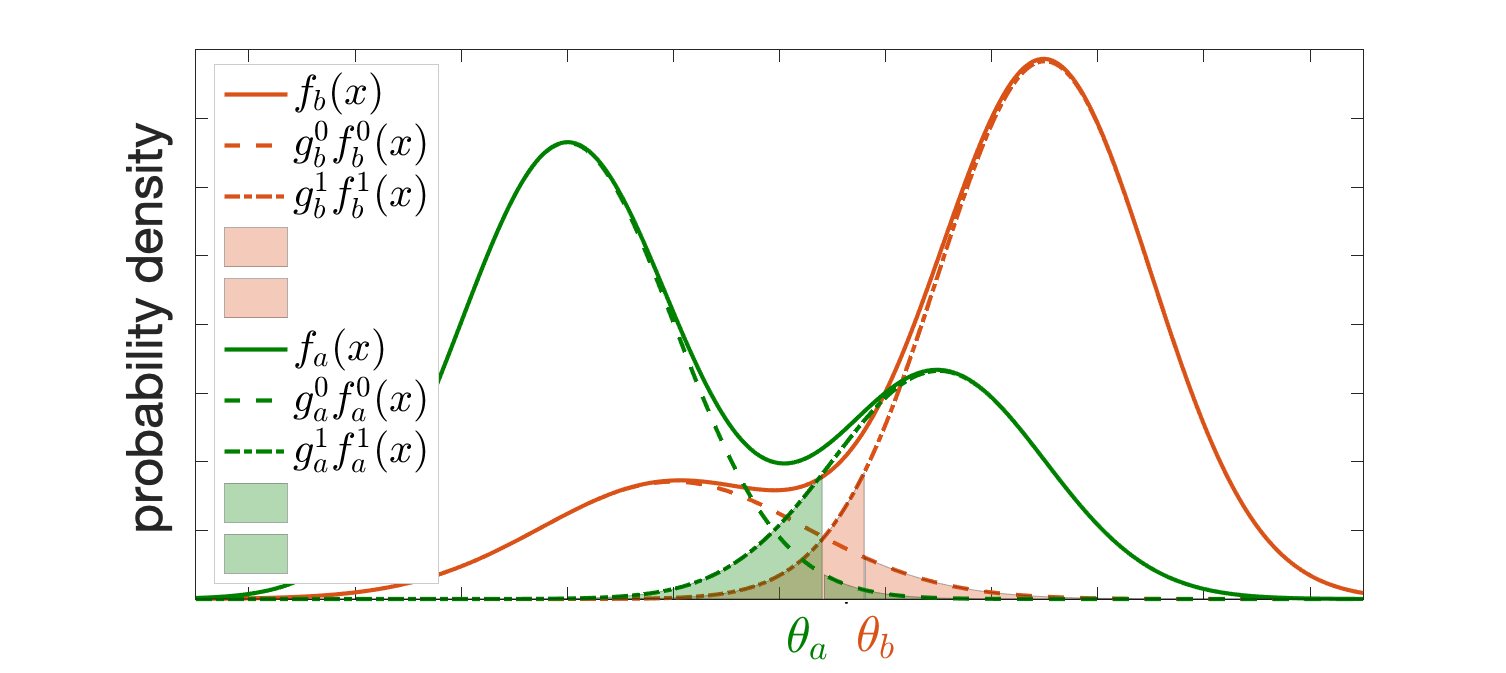}}
	\captionof{figure}{For $G_a$, $G_b$ with group proportions $\alpha_a^1 = 0.55, \alpha_a^0 = 0.15, \alpha_b^1 = 0.1, \alpha_b^0 = 0.2$, a pair of $(\theta_a,\theta_b)$ is fair under each criterion stated in Fig. \ref{fig1:b}-\ref{fig1:d} requires the corresponding colored areas be equal.} 
	\label{fig1}
\end{figure}
\begin{assumption}\label{assumption2}
	Let 
	$\mathcal{T}_{a,t} = [\underline{a}_t^{1},\overline{a}_t^{0}]$ (resp. $\mathcal{T}_{b,t} = [\underline{b}_t^{1},\overline{b}_t^{0}]$) be the overlapping interval  between $f_{a,t}^0(x)$ and $f_{a,t}^1(x)$ (resp. $f_{b,t}^0(x)$ and $f_{b,t}^1(x)$).  Distribution  $f_{k,t}^1(x)$ is strictly increasing and $f_{k,t}^0(x)$ is strictly decreasing over $\mathcal{T}_{k,t}$, $\forall k\in\{a,b\}$.
\end{assumption} 

For bell-shaped feature distributions (e.g., Normal, Cauchy, etc.), Assumption \ref{assumption2} implies that $f_{k,t}^1(x)$ and $f_{k,t}^0(x)$ are sufficiently separated. An example is shown in Fig. \ref{fig2}. \rev{As we show later, this assumption helps us establish the monotonic convergence of decisions $(\theta_a(t),\theta_b(t))$ but is not necessary for the convergence of group representation.} We next find the one-shot decision to this problem under \texttt{Simple}, \texttt{EqOpt}, and \texttt{StatPar} fairness criteria. 
\begin{lemma}\label{lemma_new1}
Under Assumption \ref{assumption2}, $\forall k\in\{a,b\}$, the optimal decision at time $t$ for $G_k$ without considering fairness is 
\begin{eqnarray*}
\theta_k^*(t) =  \argmin_{\theta_k}L_{k,t}(\theta_k) = \begin{cases}
\underline{k}_{t}^1,\text{ if }g_{k,t}^1f_{k,t}^1(\underline{k}_{t}^1)\geq g_{k,t}^0f_{k,t}^0(\underline{k}_{t}^1)\\
\delta_{k,t},\text{ if }g_{k,t}^1f_{k,t}^1(\underline{k}_{t}^1)<g_{k,t}^0f_{k,t}^0(\underline{k}_{t}^1) ~\& ~g_{k,t}^1f_{k,t}^1(\overline{k}_{t}^0)>g_{k,t}^0f_{k,t}^0(\overline{k}_{t}^0) \\
\overline{k}_{t}^0,\text{ if }g_{k,t}^1f_{k,t}^1(\overline{k}_{t}^0)\leq g_{k,t}^0f_{k,t}^0(\overline{k}_{t}^0) 
\end{cases}
\end{eqnarray*}
 where $\delta_{k,t}\in\mathcal{T}_{k,t}$ is defined such that $ g_{k,t}^1f_{k,t}^1(\delta_{k,t})=g_{k,t}^0f_{k,t}^0(\delta_{k,t})$. Moreover, $L_{k,t}(\theta_k)$ is decreasing in $\theta_k$ over $[\underline{k}_{t}^0,\theta_k^*(t)]$ and increasing over $[\theta_k^*(t),\overline{k}_{t}^1]$.  
\end{lemma}
Below we will focus on the case when $\theta_a^*(t) = \delta_{a,t}$ and $\theta_b^*(t) = \delta_{b,t}$, while analysis for the other cases are essentially the same. For \texttt{Simple}, \texttt{StatPar} and \texttt{EqOpt} fairness, $\exists$ a strictly increasing function $\phi_{\mathcal{C},t}$, such that $\Gamma_{\mathcal{C},t}(\phi_{\mathcal{C},t}(\theta_b),\theta_b) = 0$. Denote by $\phi_{\mathcal{C},t}^{-1}$ the inverse of $\phi_{\mathcal{C},t}$. Without loss of generality, we will assign group labels $a$ and $b$ such that $ \phi_{\mathcal{C},t}(\delta_{b,t})< \delta_{a,t}$ and $ \phi_{\mathcal{C},t}^{-1}(\delta_{a,t}) > \delta_{b,t}$, $\forall t$. \footnote{If the change of $f_{a,t}(x)$ and $f_{b,t}(x)$ w.r.t. the decisions follows the same rule (e.g., examples given in Section \ref{subsec:reshape}), \rev{then this relationship holds} $\forall t$.}

\begin{lemma}\label{lemma6} 
	Under \texttt{Simple}, \texttt{EqOpt}, \texttt{StatPar} fairness criteria, one-shot fair decision at time $t$ satisfies $(\theta_a^*(t),\theta_b^*(t)) = \argmin_{\theta_a, \theta_b} \overline{\alpha}_a(t) L_{a,t}(\theta_a) + \overline{\alpha}_b(t) L_{b,t}(\theta_b) \in \{(\theta_a,\theta_b)|\theta_a\in [\phi_{\mathcal{C},t}(\delta_{b,t}),\delta_{a,t}], \theta_b \in [\delta_{b,t},\phi_{\mathcal{C},t}^{-1}(\delta_{a,t})], \Gamma_{\mathcal{C},t}(\theta_a,\theta_b) = 0\}\neq \emptyset$ regardless of group proportions $\overline{\alpha}_a(t),\overline{\alpha}_b(t)$. 
\end{lemma}

Lemma \ref{lemma6} shows that \mrev{given feature distributions $f_{a,t}(x)$, $f_{b,t}(x)$, }
 although one-shot fair decisions can be different \mrev{under} different group proportions $\overline{\alpha}_a(t),\overline{\alpha}_b(t)$, these solutions are all bounded by the same compact intervals (Fig. \ref{fig3_1}). Theorem \ref{theorem5} below describes the more specific relationship between group representation $\frac{\overline{\alpha}_a(t)}{\overline{\alpha}_b(t)}$ and the corresponding one-shot decision $(\theta_a(t),\theta_b(t))$.

\begin{theorem}\label{theorem5}
	[Impact of group representation disparity on the one-shot decision] Consider the one-shot problem with group proportions $\overline{\alpha}_a(t), \overline{\alpha}_b(t)$ at time step $t$, let $(\theta_a(t),\theta_b(t))$ be the corresponding one-shot decision under either \texttt{Simple}, \texttt{EqOpt} or \texttt{StatPar} criterion. Under Assumption \ref{assumption2}, $(\theta_a(t),\theta_b(t))$ is unique and satisfies the following: 
	\begin{eqnarray}\label{eq:relation}
	 \Psi_{\mathcal{C},t}(\theta_a(t),\theta_b(t)) = \frac{\overline{\alpha}_a(t)}{\overline{\alpha}_b(t)}, 
	\end{eqnarray}  
	where $\Psi_{\mathcal{C},t}$ is some function increasing in $\theta_a(t)$ and $\theta_b(t)$, with details illustrated in Table \ref{table1}.

\small{
\begin{center}
	\begin{tabular}{ | @{\hskip4pt}c@{\hskip4pt} |@{\hskip4pt}c@{\hskip4pt} |@{\hskip4pt}c@{\hskip4pt} | @{\hskip4pt}c@{\hskip4pt}|}
		\hline
		\rule{0pt}{1.2ex} & $\theta_a \in [\underline{a}_t^0,\underline{a}_t^1]$, $\theta_b \in \mathcal{T}_{b,t}$ &  $\theta_a \in \mathcal{T}_{a,t}$, $\theta_b \in\mathcal{T}_{b,t}$ &  $\theta_a \in \mathcal{T}_{a,t}$, $\theta_b \in [\overline{b}_t^0,\overline{b}_t^1]$\\[1pt] \hline
		\texttt{EqOpt}& $ \Big(\frac{g_{b,t}^1}{g_{b,t}^0}\frac{f_{b,t}^1(\theta_b)}{f_{b,t}^0(\theta_b)} - 1\Big)\frac{g_{b,t}^0}{g_{a,t}^0} $ &\rule{0pt}{6ex} 
		$\frac{ \frac{g_{b,t}^1}{g_{b,t}^0}\frac{f_{b,t}^1(\theta_b)}{f_{b,t}^0(\theta_b)}-1}{1-\frac{g_{a,t}^1}{g_{a,t}^0}\frac{f_{a,t}^1(\theta_a)}{f_{a,t}^0(\theta_a)}}\frac{g_{b,t}^0}{g_{a,t}^0} $
		&  \diagbox[width=93pt,height=42pt]{}{} \\[14pt] \hline
		\texttt{StatPar} & $ 1-\frac{2}{\frac{g_{b,t}^1}{g_{b,t}^0}\frac{f_{b,t}^1(\theta_b)}{f_{b,t}^0(\theta_b)}+1}$  &\rule{0pt}{3ex} $ \Big(1-\frac{2}{\frac{g_{b,t}^1}{g_{b,t}^0}\frac{f_{b,t}^1(\theta_b)}{f_{b,t}^0(\theta_b)}+1}\Big)\Big(\frac{2}{1-\frac{g_{a,t}^1f_{a,t}^1(\theta_a)}{g_{a,t}^0f_{a,t}^0(\theta_a)}}-1 \Big)$& $ \frac{2}{1-\frac{g_{a,t}^1}{g_{a,t}^0}\frac{f_{a,t}^1(\theta_a)}{f_{a,t}^0(\theta_a)}}-1 $ \\[13pt] \hline
		\texttt{Simple}&  \diagbox[width=94pt,height=27pt]{}{}  &\rule{0pt}{4ex} $ \frac{g_{b,t}^1f_{b,t}^1(\theta_b)-g_{b,t}^0f_{b,t}^0(\theta_b)}{g_{a,t}^0f_{a,t}^0(\theta_a)-g_{a,t}^1f_{a,t}^1(\theta_a)}$ &\diagbox[width=93pt,height=27pt]{}{}  \\[6pt]  \hline
	\end{tabular}
\captionof{table}[]{The form of $\Psi_{\mathcal{C},t}(\theta_a,\theta_b)$ for $\mathcal{C} = \texttt{EqOpt}, \texttt{StatPar}, \texttt{Simple}$.\footnotemark}\label{table1}
\footnotetext{\mrev{The cases represented by blank cells cannot happen. When $\mathcal{C}=\texttt{Simple}$, the table only illustrates the result when $\delta_{a,t}, \delta_{b,t} \in \mathcal{T}_{a,t}\cap \mathcal{T}_{b,t}\neq\emptyset$.}} 
\end{center}}
\end{theorem}
\mrev{Note that under Assumption \ref{assumption2}, both $\frac{g_{k,t}^1f_{k,t}^1(\theta_k)}{g_{k,t}^0f_{k,t}^0(\theta_k)}$ and $g_{k,t}^1f_{k,t}^1(\theta_k)-g_{k,t}^0f_{k,t}^0(\theta_k)$ are strictly increasing in $\theta_k\in\mathcal{T}_{k,t}$, $k\in\{a,b\}$, and $\theta_a(t) = \phi_{\mathcal{C},t}(\theta_b(t))$ for some strictly increasing function. According to $\Psi_{\mathcal{C},t}(\theta_a,\theta_b)$ given in Table \ref{table1}, the larger $\frac{\overline{\alpha}_a(t)}{\overline{\alpha}_b(t)}$ results in the larger $\frac{g_{k,t}^1f_{k,t}^1(\theta_k)}{g_{k,t}^0f_{k,t}^0(\theta_k)}$ and  $g_{k,t}^1f_{k,t}^1(\theta_k)-g_{k,t}^0f_{k,t}^0(\theta_k)$, thus the larger $\theta_a(t)$ and $\theta_b(t)$.}
The above theorem characterizes the impact of the underlying population on the one-shot decisions. Next we investigate how the one-shot decision impacts the underlying population. 

\subsection{Participation dynamics}\label{subsec:dynamic}
How a user reacts to the decision is captured by the retention dynamics \eqref{eq:dynamic} which is fully characterized by the retention rate. Below we introduce two types of (perceived) mistreatment as examples when the monotonicity condition is satisfied.

\textbf{(1) User departure driven by model accuracy:}
Examples include discontinuing the use of products viewed as error-prone, e.g., speech recognition software, or medical diagnostic tools. 
In these cases, the determining factor is the classification error, i.e., users who experience low accuracy have a higher probability of leaving the system. The retention rate at time $t$ can be modeled as $\pi_{k,t}(\theta_k) = \nu(L_{k,t}(\theta_k))$ for some strictly \textit{decreasing} function $\nu(\cdot): [0,1]\rightarrow[0,1]$. 

\textbf{(2) User departure driven by intra-group disparity:}
Participation can also be affected by intra-group disparity, that between users from the same demographic group but with different labels, i.e., $G_k^j$ for $j\in\{0,1\}$. 
An example is in making financial assistance decisions where one expects to see more awards given to those qualified than to those unqualified. 
Denote by \mrev{$D_{k,t}(\theta_k) = \text{Pr}(Y = 1,h_{\theta_k}(X) = 1|K=k)-\text{Pr}(Y =  0,h_{\theta_k}(X) = 1|K=k) =
	 \int_{\theta_k}^{\infty}\big(g_k^1f_{k,t}^1(x) - g_k^0f_{k,t}^0(x)\big)dx$} as intra-group disparity of $G_k$ at time $t$, then the retention rate can be modeled as $\pi_{k,t}(\theta_k) = w(D_{k,t}(\theta_k))$ for some strictly \textit{increasing} function $w(\cdot)$ mapping to $[0,1]$.

\begin{theorem}\label{proposition1}
 Consider the one-shot problem \eqref{eq:opt} defined in Sec. \ref{subsec:prob} under either \texttt{Simple}, \texttt{EqOpt} or \texttt{StatPar} criterion, and assume distributions $f_{k,t}(x) = f_{k}(x)$ are fixed over time. Then the one-shot problems in any two consecutive time steps, i.e., $\pmb{O}_{t}, \pmb{O}_{t+1}$, satisfy the monotonicity condition under dynamics \eqref{eq:dynamic} with $\pi_{k}(\cdot)$ being either $\nu(L_{k}(\cdot))$ or $w(D_{k}(\cdot))$.\footnote{When $f_{k,t}(x)=f_k(x)$, $\forall t$, subscript $t$ is omitted in some notations ($\phi_{\mathcal{C},t}, \delta_{k,t}, \pi_{k,t}$, etc.) for simplicity. }  This implies that Theorem \ref{thm1} holds and $(\theta_a(t),\theta_b(t))$ converges monotonically to a constant decision $(\theta_a^{\infty},\theta_b^{\infty})$. Furthermore, $\underset{t\rightarrow \infty}{\lim}\frac{\overline{\alpha}_a(t) }{\overline{\alpha}_b(t) }=\frac{\beta_a}{\beta_b}\frac{1-\pi_b(\theta_b^{\infty})}{1-\pi_a(\theta_a^{\infty})}$.
\end{theorem}


When distributions are fixed, the discrepancy between $\pi_a(\theta_a(t))$ and $\pi_b(\theta_b(t))$ increases over time as $(\theta_a(t),\theta_b(t))$ changes. The process is illustrated in Fig. \ref{fig3_1}, where
 $\theta_a(t)\in [\phi_{\mathcal{C}}(\delta_b),\delta_a], \theta_b(t) \in [\delta_b,\phi_{\mathcal{C}}^{-1}(\delta_a)]$ are constrained by the same interval $\forall t$. Left and right plots illustrate cases when $\pi_{k}(\theta_k) = \nu(L_{k}(\theta_k))$ and $\pi_{k}(\theta_k) = w(D_{k}(\theta_k))$ respectively. 
 
Note that the case considered in Theorem \ref{proposition1} is a special case of Theorem \ref{thm:MC}, with distributions $f_{k,t}(x) = f_k(x)$ fixed, $O_k(\theta_k) = L_k(\theta_k)$ and both dynamics $\pi_k(\cdot) = \nu(L_k(\cdot))$ and $\pi_k(\cdot) = w(D_k(\cdot))$ some decreasing functions of $L_k(\cdot)$.\footnote{By Fig. \ref{fig3_1}, we have $D_k(\theta) = g_k^1 - L_k(\theta)$.} \rev{In this special case we obtain the additional result of monotonic convergence of decisions, which holds due to Assumption \ref{assumption2}.} 

\begin{minipage}{0.42\textwidth}
Once $\frac{\overline{\alpha}_a(t)}{\overline{\alpha}_b(t)}$ starts to increase, the corresponding one-shot solution $(\theta_{a}(t),\theta_b(t))$ also increases (Theorem \ref{theorem5}), meaning that $\theta_a(t)$ moves closer to $\theta_a^* = \delta_a$ and $\theta_b(t)$ moves further away from $\theta_b^* = \delta_b$ (solid arrows in Fig. \ref{fig3_1}).  Consequently, $L_a(\theta_a(t))$ and $D_{b}(\theta_b(t))$ decrease while $L_b(\theta_b(t))$ and $D_{a}(\theta_a(t))$ increase. Under both dynamics, $\pi_{a}(\theta_a(t))$ increases and $\pi_{b}(\theta_b(t))$ decreases, resulting in the increase of $\frac{\overline{\alpha}_a(t+1)}{\overline{\alpha}_b(t+1)}$; the feedback loop becomes self-reinforcing and representation disparity worsens.


\end{minipage}
\hspace{0.2cm}
\begin{minipage}{0.58\textwidth}
\includegraphics[trim={0cm 0cm 0cm 0cm},clip=true,width=0.495\textwidth]{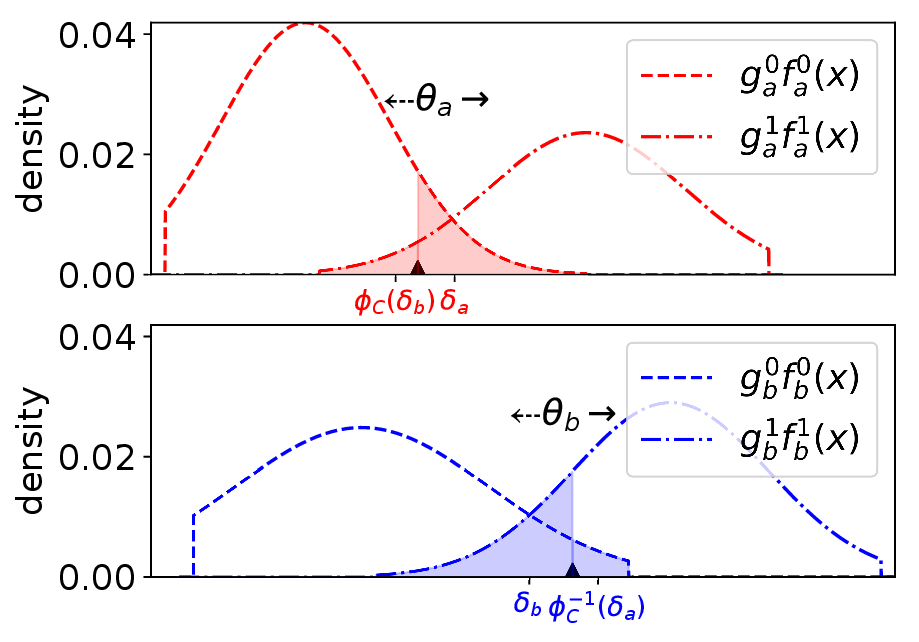}
\includegraphics[trim={0cm 0cm 0cm 0cm},clip=true, width=0.495\textwidth]{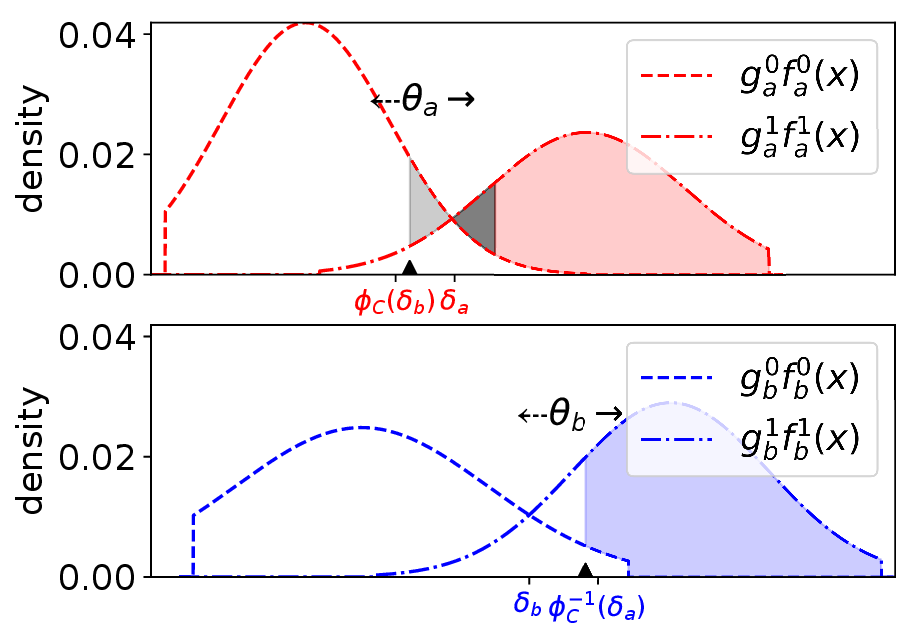}
\captionof{figure}{ Illustration of $L_{k}(\theta_k)$ and $D_{k}(\theta_k)$ w.r.t. $\theta_k$: Each black triangle represents the one-shot decision $\theta_k$; size of the colored area represents the value of $L_{k}(\theta_k)$ (left) or $D_{k}(\theta_k)$ (right). Note that for the right plot, there are two gray regions and the darker one is for compensating the lighter one thus they are of the same size; the smaller gray regions result in the larger $D_a(\theta_a)$.	\label{fig3_1}}
\end{minipage}

\subsection{Impact of decisions on reshaping feature distributions}\label{subsec:reshape}

\begin{minipage}{0.66\textwidth}
Our results so far show the potential adverse impact on group representation when imposing certain fairness criterion, while their underlying feature distributions are assumed fixed.  Below we examine what happens when decisions also affect feature distributions over time, i.e., $f_{k,t}(x)= g_{k,t}^1f_{k,t}^1(x) + g_{k,t}^0f_{k,t}^0(x)$, which is not captured by Theorem \ref{thm:MC}. We will focus on the dynamics $\pi_{k,t}(\theta_k) = \nu(L_{k,t}(\theta_k))$. Since $G_k^0$, $G_k^1$ may react differently to the same $\theta_k$, we consider two scenarios as illustrated in \rev{Fig. \ref{fig3_2}}, which shows the change in distribution from $t$ to $t+1$ when $G^1_k$ (resp. $G^0_k$) experiences the higher (resp. lower) loss at $t$ than $t-1$ (see Appendix \ref{app_reshape} for more detail): $\forall j\in\{0,1\}$, 
\end{minipage}
\hspace{0.2cm}
\begin{minipage}{0.34\textwidth}
	\includegraphics[trim={1cm 0cm 1cm 1cm},clip=true, width=\textwidth]{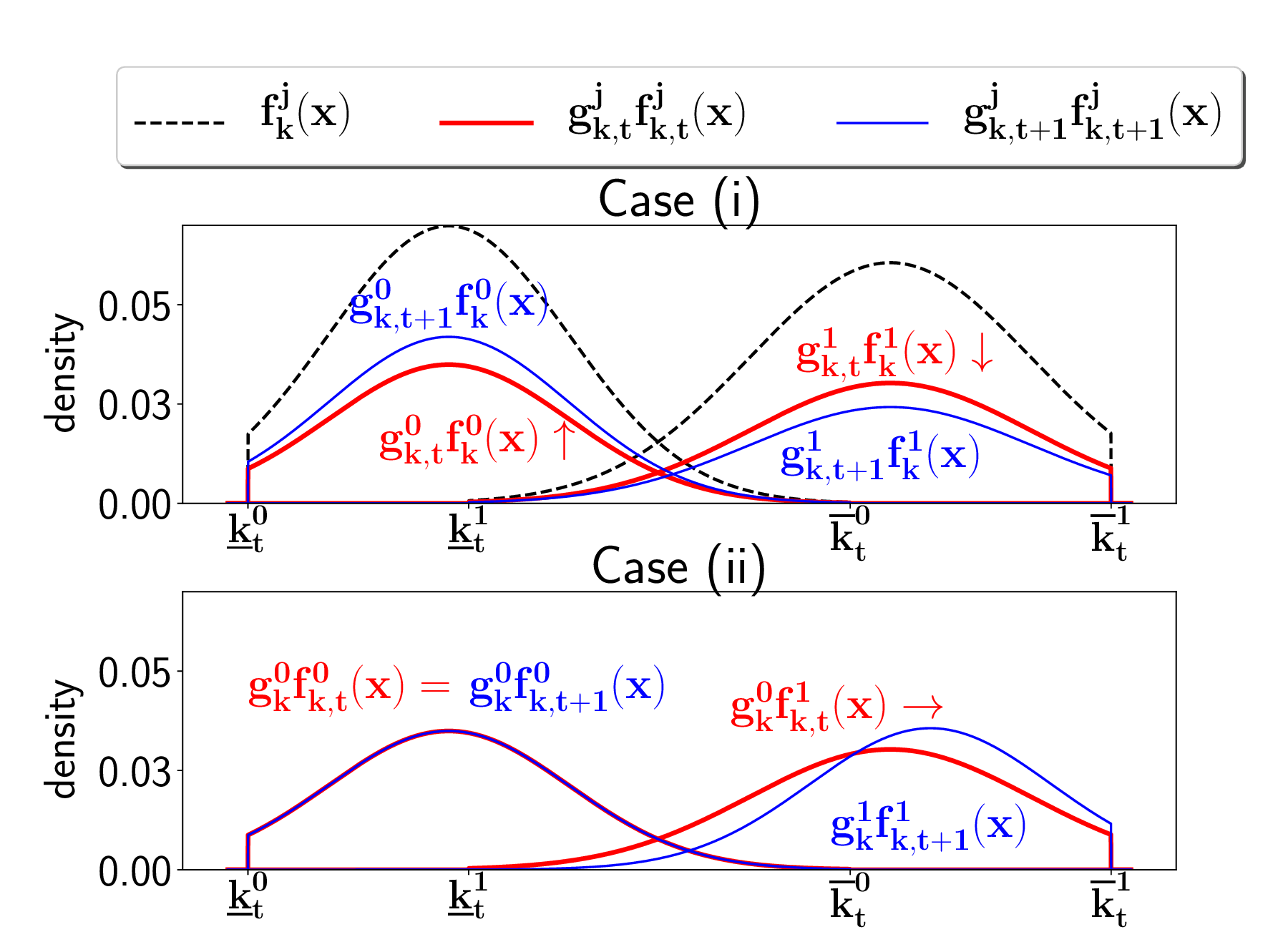}
	\captionof{figure}{ \mrev{Visualization of decisions shaping feature distributions.  \label{fig3_2}}}
\end{minipage}

\textbf{Case (i):}  $f_{k,t}^j(x) = f_{k}^j(x)$ remain fixed but $g_{k,t}^j$ changes over time given $G_k^j$'s retention determined by its perceived loss $L_{k,t}^j$,\footnote{Here $L_{k,t}^1(\theta_k) =
 	\int_{-\infty}^{\theta_k}f_{k,t}^1(x)dx$ and $
 L_{k,t}^0(\theta_k) =	\int_{\theta_k}^{\infty}f_{k,t}^0(x)dx$.
} In other words, for $i\in\{0,1\}$ and $t\geq 2$ such that $L_{k,t}^i(\theta_k(t)) < L_{k,t-1}^i(\theta_k(t-1))$, \rev{we have} $g_{k,t+1}^i>g_{k,t}^i$ and $g_{k,t+1}^{ -i}<g_{k,t}^{ -i}$, where $-i :=\{0,1\}\setminus \{i\}$.

\textbf{Case (ii):} $g_{k,t}^j = g_k^j$ but for subgroup $G_k^{i}$ that is less favored by the decision over time, its members make extra effort such that $f_{k,t}^{i}(x)$ skews toward the direction of lowering their losses.\footnote{Suppose Assumption \ref{assumption2}  holds for all $f_{k,t}^{j}(x)$ and their support does not change, then $f_{k,t}^{1}(x)$ and $f_{k,t}^{0}(x)$ overlap over $\mathcal{T}_k=[\underline{k}^1,\overline{k}^0]$, $\forall t$.}  In other words, for $i\in\{0,1\}$ and $t\geq 2$ such that $L_{k,t}^{i}(\theta_k(t)) > L_{k,t-1}^{i}(\theta_k(t-1))$, \rev{we have} $f_{k,t+1}^{i}(x) <  f_{k,t}^{i}(x), ~\forall x \in \mathcal{T}_k$, while  $f_{k,t+1}^{- i}(x) = f_{k,t}^{ -i}(x),~\forall x$, where $-i :=\{0,1\}\setminus \{i\}$.

In both cases, under the condition that $f_{k,t}(x)$ is relatively insensitive to the change in one-shot decisions, representation disparity can worsen and deterioration accelerates. The precise conditions are formally given in Conditions \ref{exacerbation} and \ref{accelerate} in Appendix \ref{app_reshape}, which describes the case 
where the change from $f_{k,t}(x)$ to $f_{k,t+1}(x)$ is sufficiently small while the change from $\frac{\overline{\alpha}_a(t)}{\overline{\alpha}_b(t)}$ to  $\frac{\overline{\alpha}_a(t+1)}{\overline{\alpha}_b(t+1)}$ and the resulting decisions from $\theta_k(t)$ to $\theta_k(t+1)$ are sufficiently large. These conditions hold in scenarios when the change in feature distributions induced by the one-shot decisions is a slow process.

\begin{theorem}\label{thm2}
[Exacerbation in representation disparity can accelerate] Consider the one-shot problem defined in \eqref{eq:opt} under either \texttt{Simple}, \texttt{EqOpt} or \texttt{StatPar} fairness criterion. 
Let the one-shot decision, representation disparity and retention rate at time $t$ be given by $\theta_k^f(t)$, $\frac{\overline{\alpha}_a^f(t)}{\overline{\alpha}_b^f(t)}$, and $\pi_{k,t}^f(\theta_k^f(t))$ \rev{when distribution $f_k(x)$ is fixed $\forall t$.  Let the same be denoted by $\theta_k^r(t)$, $\frac{\overline{\alpha}_a^r(t)}{\overline{\alpha}_b^r(t)}$, and $\pi_{k,t}^r(\theta_k^r(t))$ when $f_{k,t}(x)$ changes according to either case (i) or (ii) defined above.} Assume we start from the same distribution $f_{k,1}(x)=f_k(x)$. Under Conditions \ref{exacerbation} and \ref{accelerate} in Appendix \ref{app_reshape}, if $\pi_{a,1}^f(\theta_a^f(1))=\pi_{a,1}^r(\theta_a^r(1))\diamond \pi_{b,1}^f(\theta_b^f(1))=\pi_{b,1}^r(\theta_b^r(1))$, then $\frac{\overline{\alpha}_a^r(t+1)}{\overline{\alpha}_b^r(t+1)}\diamond\frac{\overline{\alpha}_a^r(t)}{\overline{\alpha}_b^r(t)}$ (disparity worsens) and $\frac{\overline{\alpha}_a^r(t+1)}{\overline{\alpha}_b^r(t+1)}\diamond\frac{\overline{\alpha}_a^f(t+1)}{\overline{\alpha}_b^f(t+1)}$ (accelerates), $\forall t$, where $\diamond$ represents either $``<" $or $``>"$. 
\end{theorem}

\subsection{Potential mitigation \& finding the proper fairness criterion from participation dynamics}\label{subsec:method}
The above results show that when the objective is to minimize the average loss over the entire population, applying commonly used and seemingly fair decisions at each time can exacerbate representation disparity over time under reasonable participation dynamics.
It highlights the fact that fairness has to be defined with a good understanding of how users are affected by the algorithm, and how they may react to it. For instance, consider the dynamics with $\pi_{k,t}(\theta_k) = \nu(L_{k,t}(\theta_k))$, then imposing \texttt{EqLos} fairness (Fig. \ref{fig1:d}) at each time step would sustain group representations, i.e., $\underset{t\rightarrow \infty}{\lim}\frac{\overline{\alpha}_a(t)}{\overline{\alpha}_b(t)} = \frac{\beta_a}{\beta_b}$, as we are essentially equalizing departure when equalizing loss. In contrast, under other fairness criteria the factors that are equalized do not match what drives departure, and different losses incurred to different groups cause significant change in group representation over time. 



In reality the true dynamics is likely a function of a mixture of factors given the application context, and a proper fairness constraint $\mathcal{C}$ should be adopted accordingly. Below we illustrate a method for finding the proper criterion from a general dynamics model defined below when $f_{k,t}(x) = f_k(x), \forall t$: 
\begin{eqnarray}\label{eq:dynamic_g}
N_k(t+1) = \Lambda (N_k(t),\{\pi_{k}^m(\theta_k(t))\}_{m=1}^M,\beta_k), ~\forall k\in \{a,b\},
\end{eqnarray}
where user retention in $G_k$ is driven by $M$ different factors $\{\pi_{k}^m(\theta_k(t))\}_{m=1}^M$ (e.g. accuracy, true positives, etc.) and each of them depends on decision $\theta_k(t)$. Constant $\beta_k$ is the intrinsic growth rate while the actual arrivals may depend on $\pi_{k}^m(\theta_k(t))$. The expected number of users at time $t+1$ depends on users at $t$ and new users; both may be effected by $\pi_{k}^m(\theta_k(t))$. This relationship is characterized by a general function $\Lambda$. Let $\Theta$ be the set of all  possible decisions.
\begin{assumption}\label{stable}
	$\exists (\theta_a,\theta_b)\in \Theta\times \Theta$ such that $\forall k\in \{a,b\}$, $\hat{N}_k = \Lambda (\hat{N}_k,\{\pi_{k}^m(\theta_k)\}_{m=1}^M,\beta_k)$ and $|\Lambda '(\hat{N}_k,\{\pi_{k}^m(\theta_k)\}_{m=1}^M,\beta_k) |<1$ hold for some $\hat{N}_k$, i.e., dynamics \eqref{eq:dynamic_g} under some decision pairs $(\theta_a,\theta_b)$ have stable fixed points, where $\Lambda '$ denotes the derivative of $\Lambda$ with respect to $N_k$.
\end{assumption}

To find the proper fairness constraint, let $\mathcal{C}$ be the set of decisions $(\theta_a,\theta_b)$ that can sustain group representation. It can be found via the following optimization problem; the set of feasible solutions is guaranteed to be non-empty under Assumption \ref{stable}. 
\begin{eqnarray*}
	\mathcal{C} = \argmin_{(\theta_a,\theta_b)} \Big|\frac{\tilde{N}_a}{\tilde{N}_b} - \frac{\beta_a}{\beta_b}\Big| \text{ s.t. }  \tilde{N}_k = \Lambda (\tilde{N}_k,\{\pi_k^m(\theta_k)\}_{m=1}^M,\beta_k) \in \mathbb{R}_+, \theta_k \in \Theta, \forall k\in\{a,b\}.
\end{eqnarray*}
\begin{minipage}{0.48\textwidth}
The idea is to first select decision pairs whose corresponding dynamics can lead to stable fixed points $(\tilde{N}_a, \tilde{N}_b)$; then among them select those that are best in sustaining group representation, which may or may not be unique.  Sometimes guaranteeing the perfect fairness can be unrealistic and a relaxed version is preferred, in which case all pairs $(\theta_a,\theta_b)$ satisfying $|\frac{\tilde{N}_a}{\tilde{N}_b} - \frac{\beta_a}{\beta_b}|\leq \min\{|\frac{\tilde{N}_a}{\tilde{N}_b} - \frac{\beta_a}{\beta_b}| \}+ \Delta$ constitute the $\Delta$-fair set. An example under dynamics $N_k(t+1) =N_k(t) \pi_{k}^2(\theta_k(t))+\beta_k\pi_{k}^1(\theta_k(t))$ is illustrated in Fig. \ref{fig:example}, where all curves with $\epsilon\leq \Delta\frac{\beta_b}{\beta_a}$ constitute $\Delta$-fair set (perfect fairness set is given by the deepest red curve with $\epsilon=0$). See Appendix \ref{app_example} for more details.
\end{minipage}
\hspace{0.2cm}
\begin{minipage}{0.5\textwidth}
	\centering   
	\includegraphics[width=0.48\textwidth]{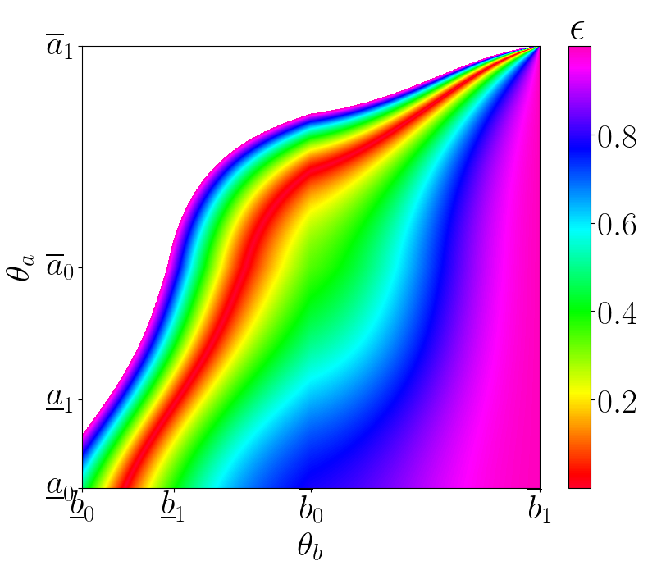}
\includegraphics[width=0.48\textwidth]{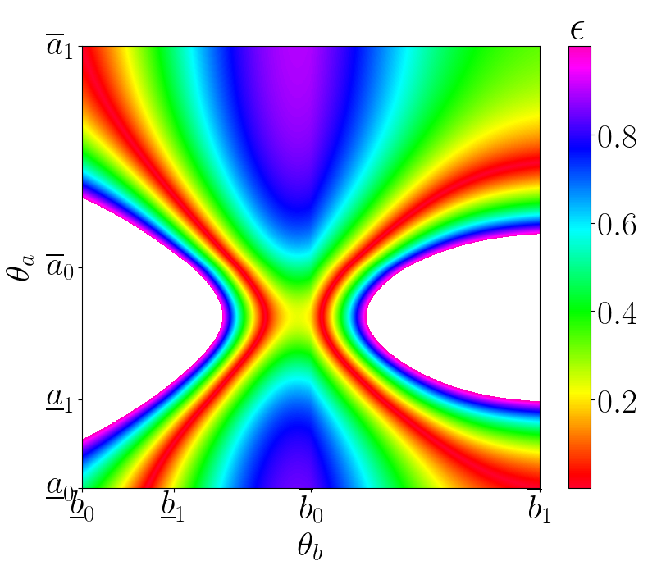}
	\captionof{figure}{Left plot: $\pi_{k}^2(\theta_k) =\nu( \int_{\theta_k}^{\infty}f_k(x)dx )$,
		$\pi_{k}^1(\theta_k) = \nu(L_k(\theta_k))$; right plot: $\pi_{k}^2(\theta_k) =\nu(L_k(\theta_k))$, $\pi_{k}^1(\theta_k) =1 $, and $\nu(x) = 1-x$. Value of each pair $(\theta_a,\theta_b)$ corresponds to $|\frac{\tilde{N}_a}{\tilde{N}_b}-\frac{\beta_a}{\beta_b}|$ measuring how well it can sustain the group representation. All points $(\theta_a,\theta_b)$ with the same value of $|\frac{\tilde{N}_a}{\tilde{N}_b} - \frac{\beta_a}{\beta_b}| = \frac{\beta_a}{\beta_b}\epsilon$ form a curve of the same color with $\epsilon \in[0,1]$ shown in the color bar.	\label{fig:example}}
\end{minipage}

\section{Experiments}\label{sec:result}
We first performed a set of experiments on synthetic data where every $G_k^j$, $k\in\{a,b\},j\in\{0,1\}$ follows
the truncated normal (Fig. \ref{fig2}) distributions. 
A sequence of one-shot fair decisions are used and group representation changes over time according to dynamics \eqref{eq:dynamic} with $\pi_k(\theta_k) = \nu(L_k(\theta_k))$.  
\rev{Parameter settings and more experimental results (e.g., sample paths, results under other dynamics and when feature distributions are learned from data) are presented in Appendix \ref{supp:experiment}.}

\begin{figure}[h]
	\centering   
	\subfigure[\texttt{Simple} fair]{\label{fig3:a}\includegraphics[trim={0cm 0cm 0cm 0cm},clip=true,width=0.245\textwidth]{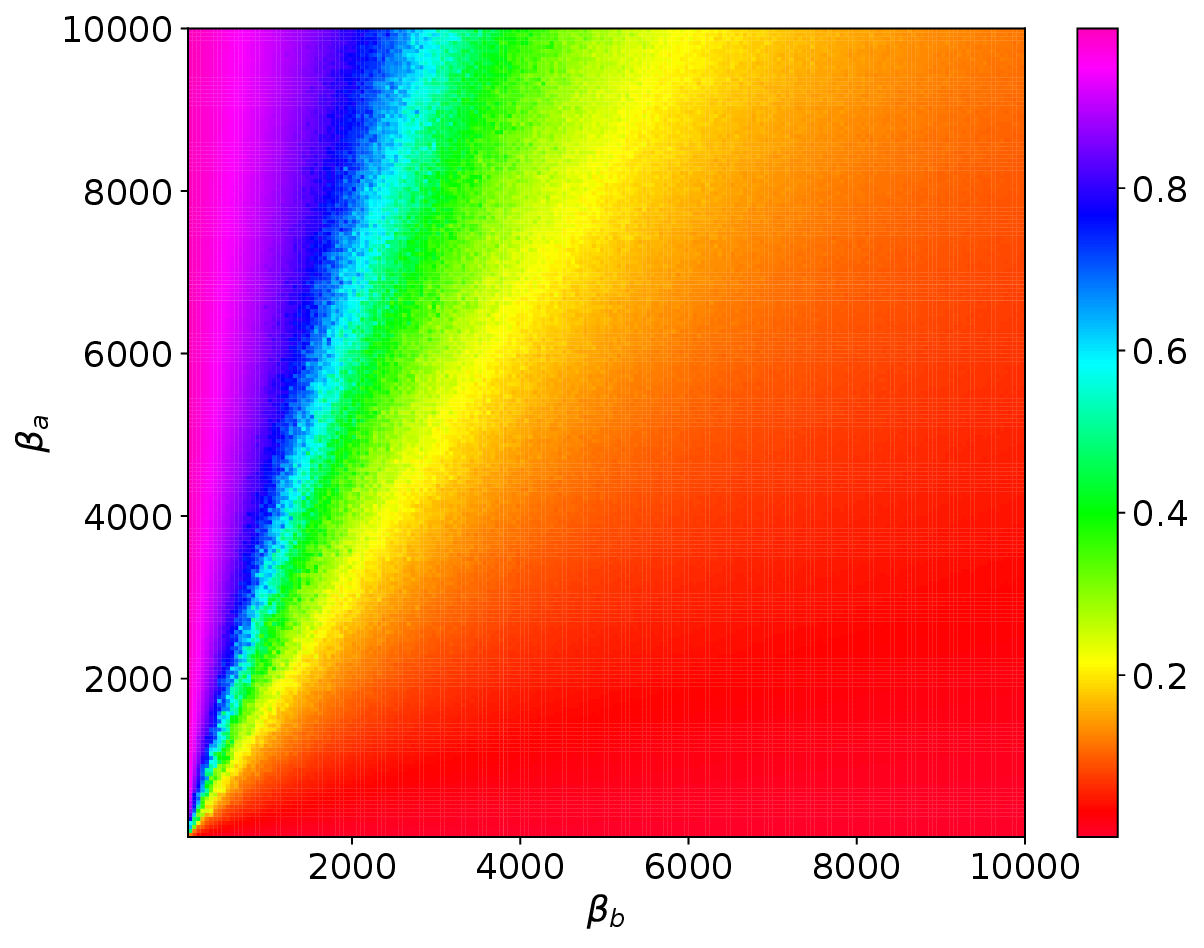}}
	\subfigure[\texttt{StatPar} fair]{\label{fig3:b}\includegraphics[trim={0cm 0cm 0cm 0cm},clip=true,width=0.245\textwidth]{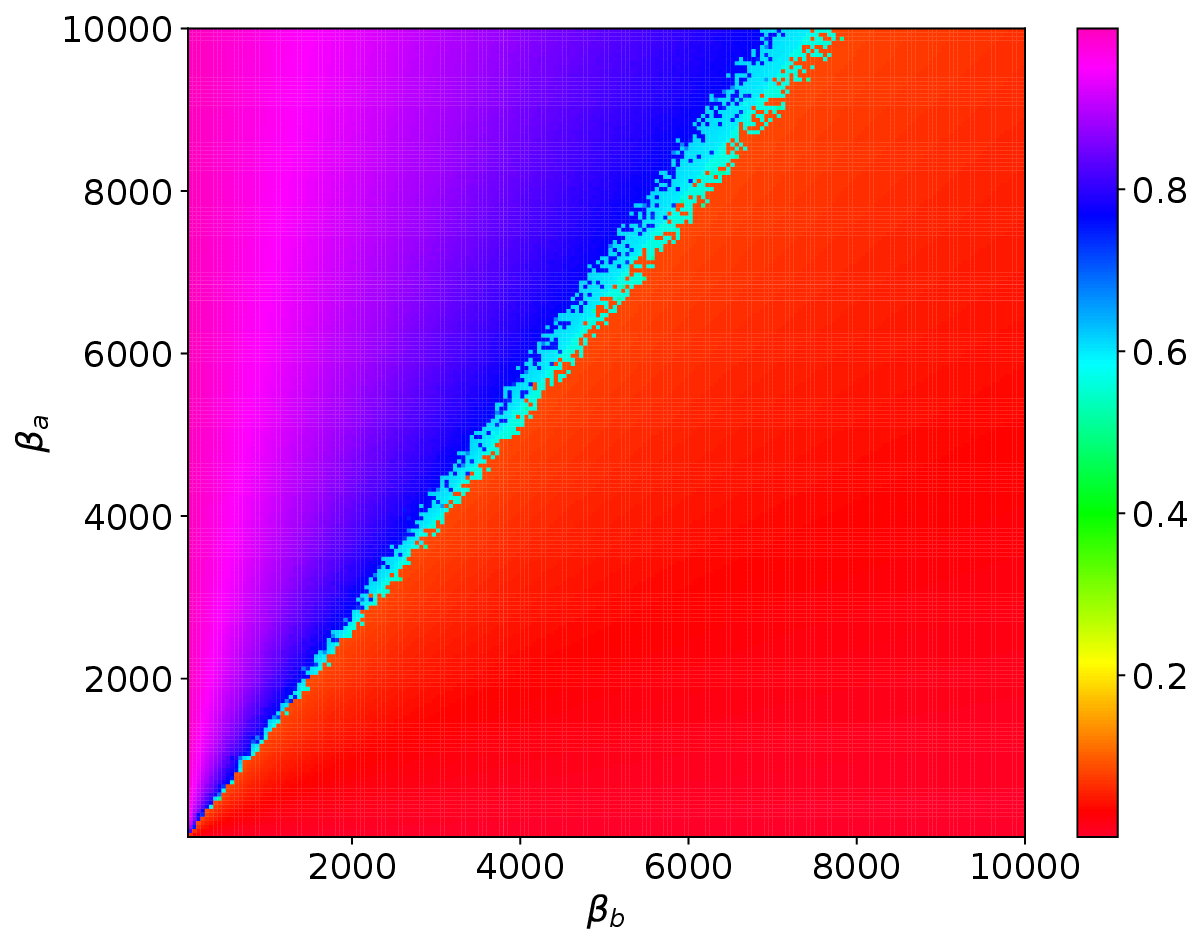}}
	\subfigure[\texttt{EqOpt} fair ]{\label{fig3:c}\includegraphics[trim={0cm 0cm 0cm 0cm},clip=true, width=0.245\textwidth]{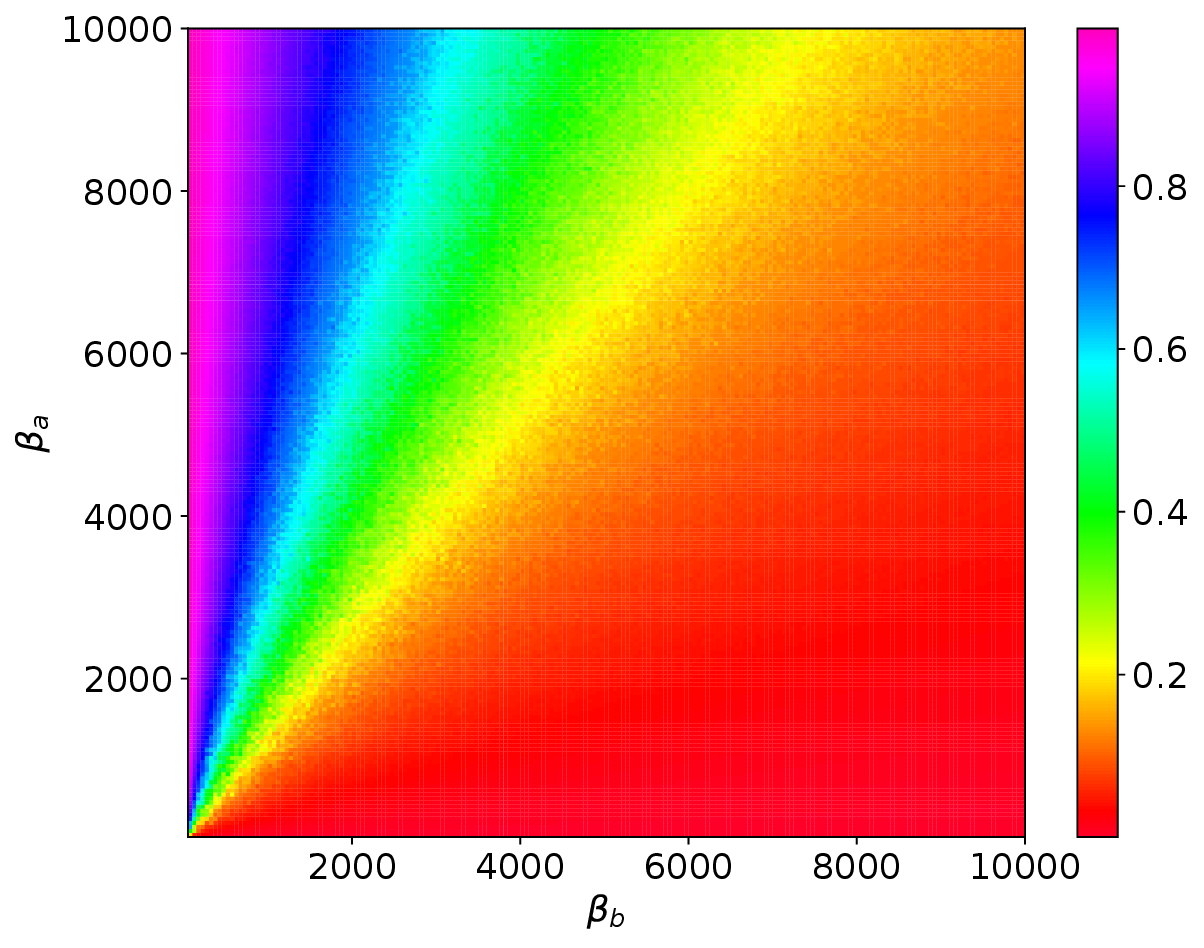}}
	\subfigure[\texttt{EqLos} fair ]{\label{fig3:d}\includegraphics[trim={0cm 0cm 0cm 0cm},clip=true,width=0.245\textwidth]{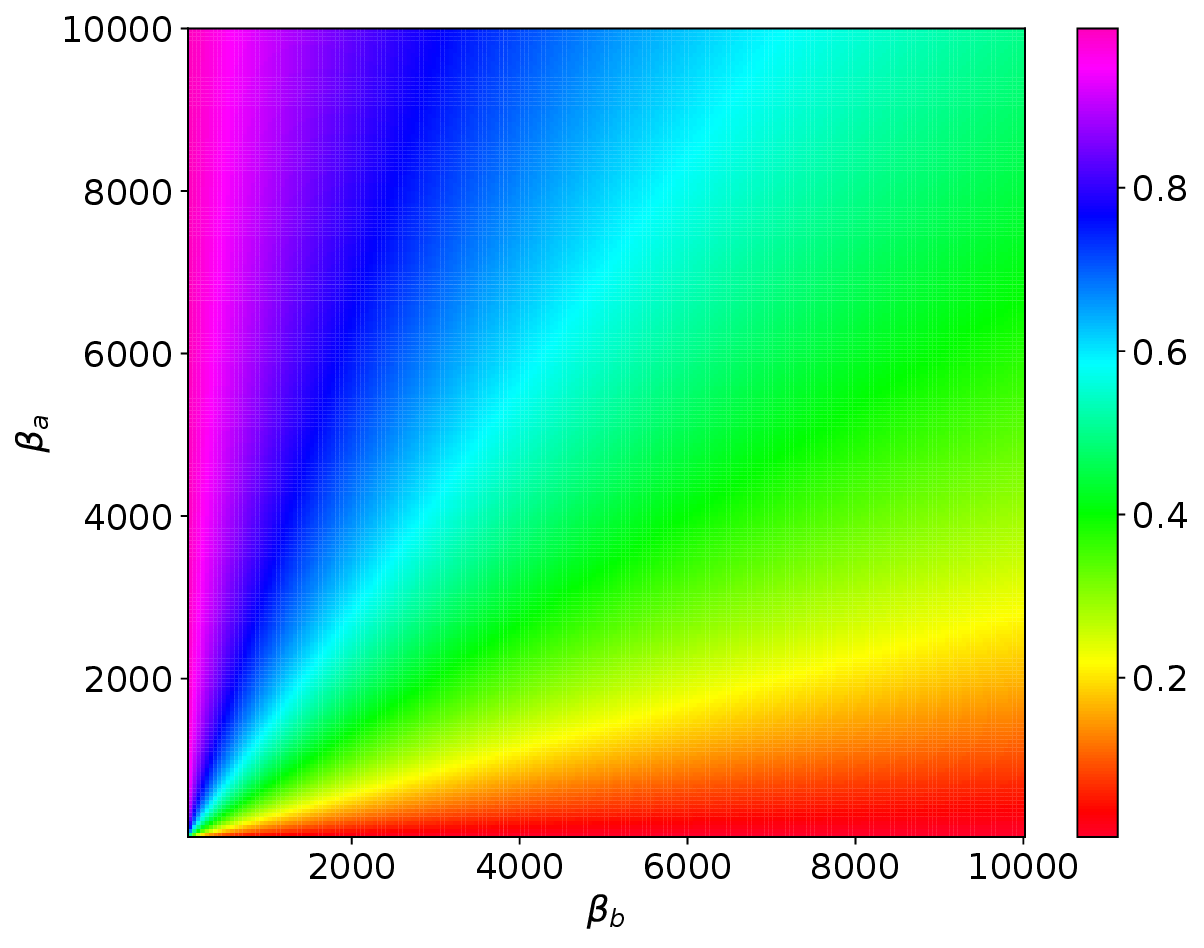}}
	\caption{Each dot in Fig. \ref{fig3:a}-\ref{fig3:d} represents the final group proportion $\lim_{t\rightarrow \infty}\overline{\alpha}_a(t)$ of one sample path under a pair of arriving rates $(\beta_a,\beta_b)$. If the group representation is sustained, then $\lim_{t\rightarrow \infty}\overline{\alpha}_a(t) = \frac{1}{1+\beta_b/\beta_a}$ for each pair of $(\beta_a,\beta_b)$, as shown in Fig. \ref{fig3:d} under \texttt{EqLos} fairness. However, under \texttt{Simple}, \texttt{StatPar} and \texttt{EqOpt} fairness, $\lim_{t\rightarrow \infty}\overline{\alpha}_a(t) = 1/( 1+\frac{\beta_b(1-\nu(L_a(\theta_a^{\infty})))}{\beta_a(1-\nu(L_b(\theta_b^{\infty})))})$.}
	\label{fig3}
\end{figure}

	Fig. \ref{fig3} illustrates the final group proportion (the converged state) $ \lim_{t\rightarrow \infty}\overline{\alpha}_a(t) $ as a function of the exogenous arrival sizes $\beta_a$ and $\beta_b$ under different fairness criteria.  With the exception of \texttt{EqLos} 
	
	\begin{minipage}{0.65\textwidth}
 fairness, group representation is severely skewed in the long run, with the system consisting mostly of $G_b$, even for scenarios when $G_a$ has larger arrival, i.e.,  $\beta_a>\beta_b$. Moreover, decisions under an inappropriate fairness criterion (\texttt{Simple}, \texttt{EqOpt} or \texttt{StatPar}) can result in poor robustness, where a minor change in $\beta_a$ and $\beta_b$ can result in very different representation in the long run (Fig. \ref{fig3:b}). 
	
	\vspace{0.7em}
	We also consider the dynamics presented in Fig. \ref{fig:example} and show the effect of $\Delta = \epsilon \frac{\beta_a}{\beta_b}$-fair decision found with method in Sec. \ref{subsec:method} on $\overline{\alpha}_a(t)$.  Each curve in Fig. \ref{fig5} represents a sample path under different $\epsilon$ where $(\theta_a(t),\theta_b(t))$ is from a small randomly selected subset of $\Delta$-fair set, $\forall t$ (to model the situation where perfect fairness is not feasible) and $\beta_a = \beta_b$. We observe that fairness is always violated at the beginning in lower plot even with small $\epsilon$. This is because the fairness set is found based on stable fixed points, which only concerns fairness in the long run. 
	
	\vspace{0.7em}
	We also trained binary classifiers over \textit{Adult} dataset \cite{Dua:2019} by
	minimizing empirical loss where features are individual data points such as sex, race, and nationality, and labels are their annual income ($\geq50k$ or $<50k$). Since the dataset does not reflect dynamics, we employ \eqref{eq:dynamic} with $\pi_k(\theta_k) = \nu(L_k(\theta_k))$ \mrev{and $\beta_a=\beta_b$}. We examine the monotonic convergence of representation disparity under \texttt{Simple}, \texttt{EqOpt} (equalized false positive/negative cost(FPC/FNC)) and \texttt{EqLos}, and consider cases where $G_a$, $G_b$ are distinguished by the three features mentioned above. \mrev{These results are shown in Fig. \ref{fig6}.} 
\end{minipage}
\hspace{0.1em}
\begin{minipage}{0.34\textwidth}
	\includegraphics[trim={0cm 0.3cm 0cm 0cm},clip=true, height=\textwidth]{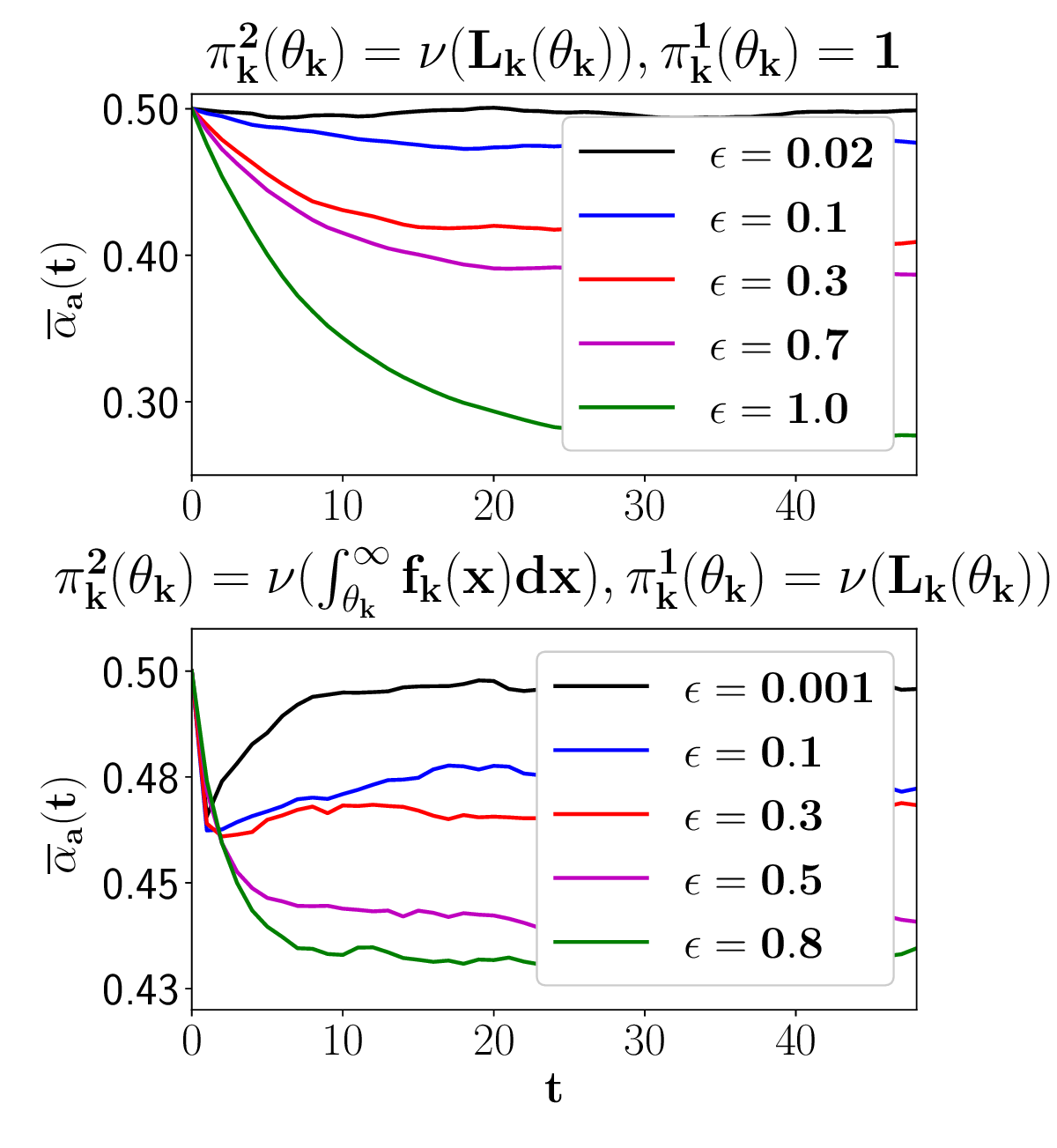}
	\captionof{figure}{ Effect of $\Delta$-fair decisions found with proposed method. \label{fig5}}
		\includegraphics[trim={0cm 0cm 2cm 0cm},clip=true,height=0.8\textwidth]{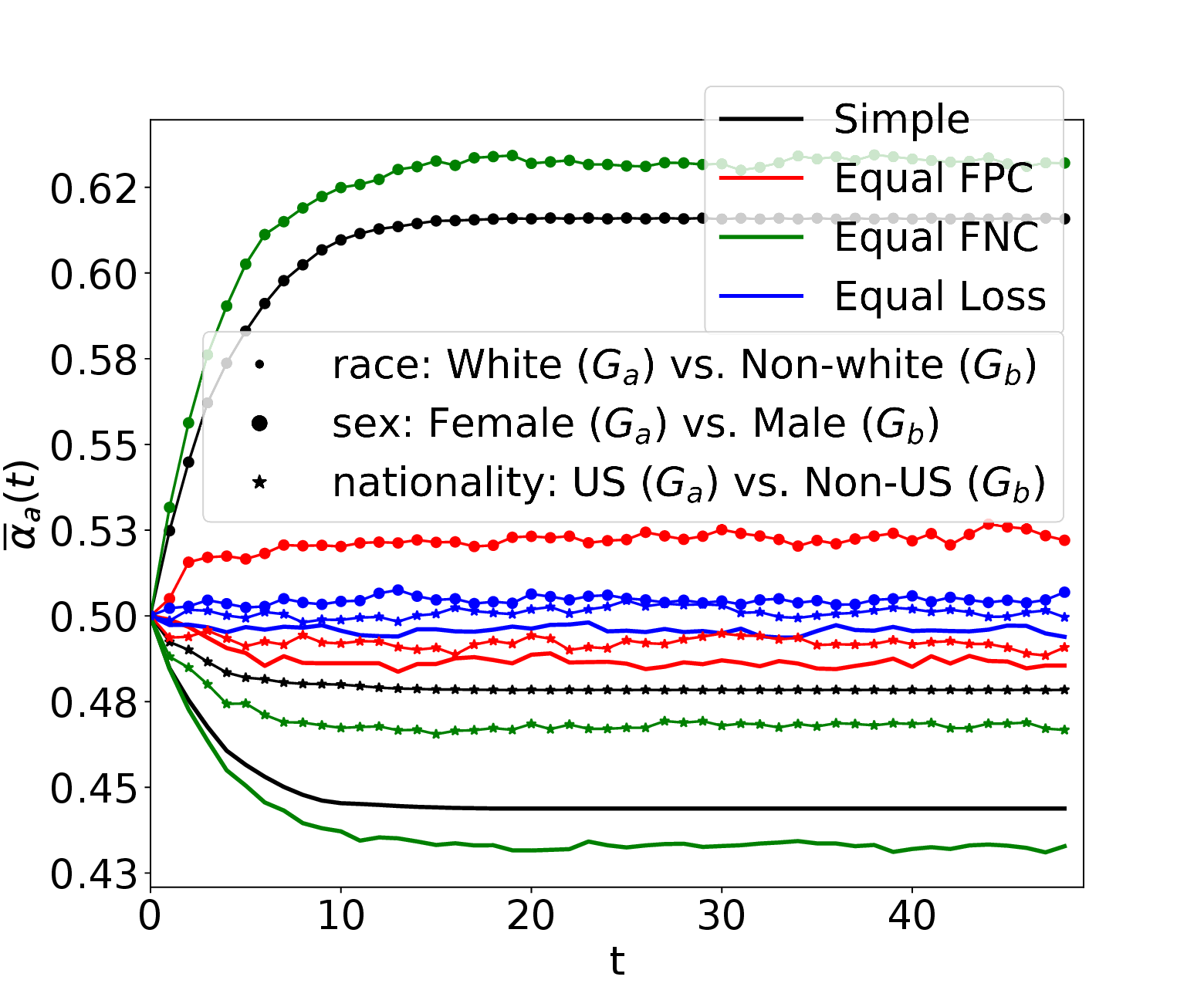}
	\captionof{figure}{Illustration of group representation disparity using \textit{Adult} dataset.    \label{fig6}}
\end{minipage}


\section{Conclusion}\label{sec:conclusion}
This paper characterizes the impact of fairness intervention on group representation in a sequential setting. We show that the representation disparity can easily get exacerbated over time under relatively mild conditions. Our results suggest that fairness has to be defined with a good understanding of participation dynamics. Toward this end, we develop a method of selecting a proper fairness criterion based on prior knowledge of participation dynamics. 
Note that we do not always have full knowledge of participation dynamics; modeling dynamics from real-world measurements and finding a proper fairness criterion based on the obtained model is a potential direction for future work.

\subsubsection*{Acknowledgments}
This work is supported by the NSF under grants CNS-1616575, CNS-1646019, CNS-1739517. 
The work of Cem Tekin was supported by BAGEP 2019 Award of the Science Academy. 

\bibliography{references}
\newpage
{\centering 
	
		\textbf{ \Large{Appendix}}}
\appendix
\section{Notation table}\label{app_table}
\begin{center}
\begin{table}[h]
	\centering
	\begin{tabular}{ p{3.3cm}| p{10cm}}
		\cline{1-2}
		  Notation& Description   \\ \cline{1-2}
		  $G_k$, $k\in\{a,b\}$& two demographic groups   \\
		  $G_k^j$, $j\in\{0,1\}$& subgroup with label $j$ in $G_k$  \\ 
		  $\alpha_k^j(t)$ & size of $G_k^j$ as a fraction of entire population at time $t$\\
		  $\overline{\alpha}_k(t)$ & size of $G_k$ as a fraction of entire population at time $t$, i.e., $\alpha_a^0(t)+\alpha_k^1(t)$\\
		  $g_{k,t}^j$& fraction of subgroup with label $j$ in $G_k$ at time $t$, i.e., $\text{Pr}(Y=j|K=k)=\alpha_k^j(t)/\overline{\alpha}_k(t)$ \\
		  $f_{k,t}^j(x)$& feature distribution of $G_k^j$at time $t$, i.e., $\text{Pr}(X=x|K=k, Y=j)$\\
		  $f_{k,t}(x)$ & feature distribution of $G_k$ at time $t$, i.e., $\text{Pr}(X=x|K=k)$ and $f_{k,t}(x) = g_{k,t}^1f_{k,t}^1(x)+g_{k,t}^0f_{k,t}^0(x)$\\
		  $h_{\theta}(x)$& decision rule parameterized by $\theta$\\
		  $\theta_k(t)$ & decision parameter for $G_k$ at time $t$\\
		  $\pmb{O}_t(\theta_a,\theta_b;\overline{\alpha}_a(t),\overline{\alpha}_b(t))$& objective of one-shot problem at time $t$ with group proportions $\overline{\alpha}_a(t),\overline{\alpha}_b(t)$ \\
		  $O_{k,t}(\theta_k)$ & sub-objective of $G_k$ at time $t$\\
		  $\Gamma_{\mathcal{C},t}(\theta_a,\theta_b)$ & a fairness constraint imposed on $\theta_a$ and $\theta_b$ for two groups at time $t$\\
		  $N_k(t)$& expected number of users from $G_k$ at time $t$\\
		  $\pi_{k,t}(\theta_k(t))$& retention rate of $G_k$ at time $t$ when imposing decision $\theta_k(t)$\\
		  $\beta_k$& number of exogenous arrivals to $G_k$ at every time step\\
		  $L_{k,t}(\theta_k)$&expected loss incurred to $G_k$ by taking decision $\theta_k$ at time $t$\\
		  $L_{k,t}^j(\theta_k)$&expected loss incurred to $G_k^j$ by taking decision $\theta_k$ at time $t$\\
		  $[\underline{k}_t^j,\overline{k}_t^j]$& bounded support of distribution $f_{k,t}^j(x)$\\
		  $\mathcal{T}_{k,t}$ & overlapping interval between $f_{k,t}^0(x)$ and $f_{k,t}^1(x)$ at time $t$, i.e.,  $[\underline{k}_t^1,\overline{k}_t^0]$\\
		  $\delta_{k,t}$& optimal decision for $G_k$ at time $t$ such that $\delta_{k,t} = \argmin_{\theta}L_{k,t}(\theta)$ and satisfies $g_{k,t}^1f_{k,t}^1(\delta_k) = g_{k,t}^0f_{k,t}^0(\delta_k)$\\
		  $\phi_{\mathcal{C},t}(\cdot)$& a increasing function determined by constraint $\Gamma_{\mathcal{C},t}(\theta_a,\theta_b)$ mapping $\theta_b$ to $\theta_a$, i.e., $\Gamma_{\mathcal{C},t}(\phi_{\mathcal{C},t}(\theta_b),\theta_b)$\\
		  $D_{k,t}(\theta_k)$ & intra-group disparity of $G_k$ at time $t$\\
		  $\pi_k^m(\theta_k)$& $m$th factor that drives user retention\\
		  $\Lambda(\cdot)$ & dynamics model specifying the relationship between $N_k(t+1)$ and $N_k(t+1)$, $\beta_k$, $\pi_k^m(\theta_k(t))$ \\\cline{1-2}

	\end{tabular}
\end{table}
\end{center}

\section{Related Work}\label{app_related}
The impact of fairness interventions on both individuals and society, and the fairness in sequential decision making have been studied in the literature. \cite{pmlr-v80-liu18c} constructs a one-step feedback model over two consecutive time steps and characterizes the impact of fairness criteria (statistical parity and equal of opportunity) on changing each individual’s feature and reshaping the entire population. Similarly, \cite{heidari2019on} proposes an effort-based measure of
unfairness and constructs an individual-level model characterizing how an individual responds to the decisions based on it. The impact on the entire group is then derived from it and the impacts of fairness intervention are examined. While both highlight the importance of temporal modeling in evaluating the fairness, their main focus is on the adverse impact on feature distribution, rather than on group representation disparity. In contrast, our work focuses on the latter but also considers the impact of reshaping feature distributions. Moreover, we formulate the long-term impact over infinite horizon while \cite{heidari2019on,pmlr-v80-liu18c} only inspect the impact over two steps. 

\cite{pmlr-v80-hashimoto18a} also considers a sequential framework where the user departure is driven by model accuracy. It adopts the objective of minimizing the loss of the group with the highest loss (instead of overall or average loss), which can prevent the extinction of any group from the system. It requires multiple demographic groups use the same model and does not adopt any fairness criterion. In contrast, we are more interested in the impact of various fairness criteria on representation disparity and if it is possible to sustain the group representation by imposing any fairness criterion. Other differences include the fact we consider the case when feature distributions are reshaped by the decisions (Section \ref{subsec:reshape}) and \cite{pmlr-v80-hashimoto18a} does not.

 \cite{kannan2019downstream} also constructs a two-stage model in the context of college admission, it shows that increasing admission rate of a group can increase the overall qualification for this group overtime.  \cite{hu2018short} describes a model in the context of labor market. They show that imposing the demographic parity constraint can incentivize under-represented groups to invest in education, which leads to a better long-term equilibrium.

Extensive studies on fairness in sequential decision making or online learning has been done \cite{blum2018preserving,dimitrakakis2019bayesian,heidari2018preventing,jabbari2017fairness,valera2018enhancing,zhang2014fairness}. Most of them focus on proposing appropriate fairness notions to improve the fairness-accuracy trade-off. To the best of our knowledge, none of them considers the impact of fairness criteria on group representation disparity.

\section{Proof of Theorem \ref{thm1}}

Theorem \ref{thm1} is proved based on the following Lemma.
\begin{lemma}\label{lemma4}

Let $a, b, z_a,z_b$ be real constants, where $a,b\in \mathbb{R}_+$ and  $z_a,z_b\in[0,1]$. If $b\geq a>1$, $z_b-z_a>\frac{1}{a}-\frac{1}{b}$ and $b<\frac{1}{1-z_b}$ are satisfied, then the following holds:
\begin{eqnarray}\label{eq:lemma3}
\frac{1+z_a+az_a^2}{1+z_b+bz_b^2} \leq \frac{1+az_a}{1+bz_b} 
\end{eqnarray}		
\end{lemma}

\begin{proof}
Re-organizing \eqref{eq:lemma3} gives the following:
\begin{eqnarray*}
(1+z_a+az_a^2)(1+bz_b ) &\leq  &(1+z_b+bz_b^2)(1+az_a)\\
bz_b + bz_az_b + z_a + abz_a^2z_b + az_a^2 &\leq &az_a + z_b + az_az_b+bz_b^2+abz_b^2z_a
\end{eqnarray*}
Proving \eqref{eq:lemma3} is equivalent to showing the following:
\begin{eqnarray*}
0 \leq (a-1)\frac{1}{z_b} + (1-b)\frac{1}{z_a} + b\frac{z_b}{z_a}-a\frac{z_a}{z_b} + \underbrace{a-b+ab(z_b-z_a)}_{\textbf{term 1}}
\end{eqnarray*}
Since $z_b-z_a>\frac{1}{a}-\frac{1}{b}$, $\textbf{term 1}>a-b+b-a = 0$ holds. Therefore, proving \eqref{eq:lemma3} is equivalent to showing:
\begin{eqnarray}\label{eq:lemma3_1}
az_a^2 + (1-a)z_a \leq bz_b^2+(1-b)z_b
\end{eqnarray}
Since $b< \frac{1}{1-z_b}$ holds, implying $z_b> 1-\frac{1}{b}$.

Define a function $g(z) = cz^2 + (1-c)z$, $z\in[0,1]$ under any constant $c>1$. The following holds:
\begin{eqnarray*}
g(1-\frac{1}{c}) = 0; & g(1) = 1&\\
g'(z) = 2cz + 1-c; & g'(1-\frac{1}{c}) = c- 1; &
g''(z) = 2c  
\end{eqnarray*} 

Since $g''(z)$ is a positive constant over $z\in[0,1]$, $g'(z)$ is strictly increasing and $g'(z)>0$ when $z\in(1-\frac{1}{c},1]$, thus $g(z)$ is increasing over $z\in(1-\frac{1}{c},1]$ from $0$ to $1$. 

Now consider two functions $g_a(z) = az^2 + (1-a)z $ and $g_b(z) = bz^2+(1-b)z$ with $z\in[0,1]$.
From the above analysis, $g_a(z)$ is increasing over  $(1-\frac{1}{a},1]$ from $0$ to $1$ and $g_b(z)$ is increasing over  $(1-\frac{1}{b},1]$ from $0$ to $1$.  Moreover, $1-\frac{1}{b}\geq1-\frac{1}{a}$ and $g_b''(z) = 2b \geq 2a = g_a''(z)$, i.e., the speed that $g_b(z)$ increases over  $(1-\frac{1}{b},1]$ is NOT slower than the speed that $g_a(z)$ increases over  $(1-\frac{1}{a},1]$.  Since $z_b-z_a>\frac{1}{a}-\frac{1}{b} = (1-\frac{1}{b})-(1-\frac{1}{a})$ and $z_b> 1-\frac{1}{b}$,
$g_a(z_a)\leq g_b(z_b)$ must hold.

Therefore, \eqref{eq:lemma3_1} is satisfied. Inequality \eqref{eq:lemma3} is proved.

\end{proof}

To simplify the notation, denote $\pi_{k,t} := \pi_{k,t}(\theta_k(t))$. We will only present the case when $\diamond := ``<"$, cases when $\diamond := ``>"$ and $\diamond := ``="$ can be derived similarly and are omitted.

To prove Theorem \ref{thm1}, we prove the following statement using induction:  If $ \pi_{a,1} < \pi_{b,1}$, then $\forall t$,  $\frac{\overline{\alpha}_a(t+1)}{\overline{\alpha}_b(t+1)} < \frac{\overline{\alpha}_a(t)}{\overline{\alpha}_b(t)}$ and $\pi_{a,t+1}<\pi_{a,t}<\pi_{b,t}<\pi_{b,t+1}$ hold under monotonicity condition. Moreover, $N_b(t)< \frac{\beta_b}{1-\pi_{b,t}}, \forall t$.

\textbf{Base Case: }

Since $\frac{N_a(1)}{N_b(1)} = \frac{\beta_a}{\beta_b}$. If $ \pi_{a,1} < \pi_{b,1}$, then $\frac{\overline{\alpha}_a(2)}{\overline{\alpha}_b(2)} = \frac{N_a(1)\pi_{a,1} + \beta_a}{N_b(1)\pi_{b,1}+\beta_b} < \frac{N_a(1)}{N_b(1)} = \frac{\overline{\alpha}_a(1)}{\overline{\alpha}_b(1)}$. Under monotonicity condition, it results in $\pi_{a,2}<\pi_{a,1}<\pi_{b,1}<\pi_{b,2}$. Moreover, since $N_b(2) = N_b(1)\pi_{b,1} + \beta_b >N_b(1)$, implying $N_b(1)< \frac{\beta_b}{1-\pi_{b,1}}$.

\textbf{Induction Step:}

 Suppose    $\frac{\overline{\alpha}_a(t+1)}{\overline{\alpha}_b(t+1)} < \frac{\overline{\alpha}_a(t)}{\overline{\alpha}_b(t)}\leq \frac{\beta_a}{\beta_b}$, $\pi_{a,t+1}<\pi_{a,t}<\pi_{b,t}<\pi_{b,t+1}$ and $N_b(t)<\frac{\beta_b}{1-\pi_{b,t}}$ hold at time $t\geq 1$. Show that for time step $t+1$, $\frac{\overline{\alpha}_a(t+2)}{\overline{\alpha}_b(t+2)} < \frac{\overline{\alpha}_a(t+1)}{\overline{\alpha}_b(t+1)}\leq \frac{\beta_a}{\beta_b}$, $\pi_{a,t+2}<\pi_{a,t+1}<\pi_{b,t+1}<\pi_{b,t+2}$ and $N_b(t+1)< \frac{\beta_b}{1-\pi_{b,t+1}}$  also hold.
 
  Denote $N_a(t) = c_a\beta_a$ and $N_b(t) = c_b\beta_b$. Since $N_k(t) = N_k(t-1)\pi_{k,t-1}+\beta_k>\beta_k,\forall t$, it holds that $c_a,c_b>1$.
  
  By hypothesis, $\frac{\overline{\alpha}_a(t)}{\overline{\alpha}_b(t)}\leq \frac{\beta_a}{\beta_b}$ implies that $c_b\geq c_a > 1$, and $N_b(t)<\frac{\beta_b}{1-\pi_{b,t}}$ implies that  $c_b<\frac{1}{1-\pi_{b,t}}$. Since $\frac{N_a(t+1)}{N_b(t+1)}=\frac{N_a(t)\pi_{a,t}+\beta_a}{N_b(t)\pi_{b,t}+\beta_b} =\frac{\beta_a}{\beta_b}\frac{c_a\pi_{a,t}+1}{c_b\pi_{b,t}+1}  <\frac{N_a(t)}{N_b(t)} = \frac{\beta_a}{\beta_b}\frac{c_a}{c_b}$, re-organizing it gives $\pi_{b,t} - \pi_{a,t} > \frac{1}{c_a} - \frac{1}{c_b}$. 
  
  By Lemma \ref{lemma4}, the following holds:
  $$\frac{N_a(t)\pi_{a,t}^2 + \beta_a(1+\pi_{a,t})}{N_b(t)\pi_{b,t}^2 + \beta_b(1+\pi_{b,t})} = \frac{\beta_a}{\beta_b}\frac{1+\pi_{a,t} +c_a\pi_{a,t} ^2}{1+\pi_{b,t}+c_b\pi_{b,t}^2} \leq  \frac{\beta_a}{\beta_b}\frac{1+c_a\pi_{a,t} }{1+c_b\pi_{b,t}} = \frac{N_a(t+1)}{N_b(t+1)}=\frac{\overline{\alpha}_a(t+1)}{\overline{\alpha}_b(t+1)}$$
  
  Since we suppose $\pi_{a,t+1}<\pi_{a,t}<\pi_{b,t}<\pi_{b,t+1}$, we have:
  $$ \frac{N_a(t)\pi_{a,t}^2 + \beta_a(1+\pi_{a,t})}{N_b(t)\pi_{b,t}^2 + \beta_b(1+\pi_{b,t})} > \frac{(N_a(t)\pi_{a,t}+\beta_a)\pi_{a,t+1}+\beta_a}{(N_b(t)\pi_{b,t}+\beta_b)\pi_{b,t+1}+\beta_b}= \frac{\overline{\alpha}_a(t+2)}{\overline{\alpha}_b(t+2)} $$

 It implies that $ \frac{\overline{\alpha}_a(t+2)}{\overline{\alpha}_b(t+2)}< \frac{\overline{\alpha}_a(t+1)}{\overline{\alpha}_b(t+1)} $.
 
By motonoticity condition, it results in $\pi_{a,t+2}<\pi_{a,t+1}<\pi_{b,t+1}<\pi_{b,t+2}$.

Moreover, $N_b(t+1) = N_b(t)\pi_{b,t} + \beta_b<\frac{\beta_b\pi_{b,t}}{1-\pi_{b,t}}+\beta_b = \frac{\beta_b}{1-\pi_{b,t}}<\frac{\beta_b}{1-\pi_{b,t+1}}$.

The statement holds for time $t+1$. This completes the proof.

\section{Proof of Theorem \ref{thm:MC}}

Without loss of generality, let $\frac{\widehat{ \overline{\alpha}}_a}{\widehat{ \overline{\alpha}}_b}< \frac{\widetilde{ \overline{\alpha}}_a}{\widetilde{ \overline{\alpha}}_b}$. Since $\pi_k(\theta_k) = h_k(O_k(\theta_k))$ with $h_k(\cdot)$ being a decreasing function, showing that $\widetilde{\mathbf{O}}$ and $\widehat{\mathbf{O}}$ satisfy Monotonicity condition is equivalent to showing that $O_a(\widehat{\theta}_a) > O_a(\widetilde{\theta}_a)$, $O_b(\widehat{\theta}_b) < O_b(\widetilde{\theta}_b)$. Under the condition that $O_k(\widehat{\theta}_k) \neq O_k(\widetilde{\theta}_k)$ for any possible $\widehat{\overline{\alpha}}_a \neq\widetilde{\overline{\alpha}}_a$ , prove by contradiction: suppose $O_a(\widehat{\theta}_a) < O_a(\widetilde{\theta}_a)$ holds, then $O_b(\widehat{\theta}_b) > O_b(\widetilde{\theta}_b)$ must also hold otherwise $(\widehat{\theta}_a, \widehat{\theta}_b)$ will be the solution to $\widetilde{\mathbf{O}}$. 

Because $(\widehat{\theta}_a, \widehat{\theta}_b)$ is the optimal solution to $\widehat{\mathbf{O}}$ and $(\widetilde{\theta}_a, \widetilde{\theta}_b)$ is the optimal solution to $\widetilde{\mathbf{O}}$, and $O_b(\widehat{\theta}_b) > O_b(\widetilde{\theta}_b)$, the following holds: 
\begin{eqnarray*}
	\widehat{\overline{\alpha}}_a O_a(\widehat{\theta}_a)+\widehat{\overline{\alpha}}_b O_b(\widehat{\theta}_b) \leq \widehat{\overline{\alpha}}_a O_a(\widetilde{\theta}_a)+\widehat{\overline{\alpha}}_b O_b(\widetilde{\theta}_b) \rightarrow \frac{O_a(\widehat{\theta}_a) - O_a(\widetilde{\theta}_a)}{O_b(\widetilde{\theta}_b) - O_b(\widehat{\theta}_b)} \geq \frac{\widehat{\overline{\alpha}}_b}{\widehat{\overline{\alpha}}_a}\\
	\widetilde{\overline{\alpha}}_a O_a(\widetilde{\theta}_a)+\widetilde{\overline{\alpha}}_b O_b(\widetilde{\theta}_b) \leq \widetilde{\overline{\alpha}}_a O_a(\widehat{\theta}_a)+\widetilde{\overline{\alpha}}_b O_b(\widehat{\theta}_b) \rightarrow \frac{O_a(\widehat{\theta}_a) - O_a(\widetilde{\theta}_a)}{O_b(\widetilde{\theta}_b) - O_b(\widehat{\theta}_b)} \leq \frac{\widetilde{\overline{\alpha}}_b}{\widetilde{\overline{\alpha}}_a}
\end{eqnarray*}

It implies that $\frac{\widehat{ \overline{\alpha}}_a}{\widehat{ \overline{\alpha}}_b}\geq \frac{\widetilde{ \overline{\alpha}}_a}{\widetilde{ \overline{\alpha}}_b}$, which is a contradiction. 

\section{Proof of Lemma \ref{lemma_new1}}\label{app_optimal}
Starting from Appendix \ref{app_optimal} until Appendix \ref{app_dynamic}, we simplify the notations by removing $t$ from subscript, i.e., $L_{k,t}(\theta_k) :=L_k(\theta_k) $, $g_{k,t}^j:=g_k^j$, $f_{k,t}(x) :=f_k(x)$, $f_{k,t}^j(x) :=f_k^j(x)$, $\underline{k}_t^j:=\underline{k}^j$, $\overline{k}_t^j:=\overline{k}^j$, $\phi_{\mathcal{C},t} :=\phi_{\mathcal{C}} $, $\Gamma_{\mathcal{C},t}:=\Gamma_{\mathcal{C}}$, $\delta_{k,t}:=\delta_k$, $\mathcal{T}_{k,t}:=\mathcal{T}_k$.

The loss for group $k$ can be written as $$L_k(\theta_k) = \int_{-\infty}^{\theta_k} g_k^1f_k^1(x)dx+ \int_{\theta_k}^{\infty} g_k^0f_k^0(x)dx=\begin{cases}
\int_{\theta_k}^{\overline{k}^0} g_k^0f_k^0(x)dx , \text{ if }\theta_k \in [\underline{k}^0,\underline{k}^1]\\
\int_{\theta_k}^{\overline{k}^0} g_k^0f_k^0(x)dx +\int_{\underline{k}^1}^{\theta_k} g_k^1f_k^1(x)dx, \text{ if }\theta_k \in [\underline{k}^1,\overline{k}^0]\\
\int_{\underline{k}^1}^{\theta_k} g_k^1f_k^1(x)dx , \text{ if }\theta_k \in [\overline{k}^0,\overline{k}^1]
\end{cases}$$ 

which is decreasing in $\theta_k$ over $[\underline{k}^0,\underline{k}^1]$ and increasing over $[\overline{k}^0,\overline{k}^1]$, the optimal solution $\theta_k^*\in[\underline{k}^1,\overline{k}^0] $. Taking derivative of $L_k(\theta_k)$ w.r.t. $\theta_k$ gives $\frac{d L_k(\theta_k)}{d \theta_k}= g_k^1f_k^1(\theta_k)-g_k^0f_k^0(\theta_k)$, which is strictly increasing over $[\underline{k}^1,\overline{k}^0] $ under Assumption \ref{assumption2}.

The optimal solution $\theta_k^* =  \argmin_{\theta_k}L_k(\theta_k)\in\{\underline{k}^1, \delta_k,\overline{k}^0\}$ can be thus found easily. Moreover, $L_k(\theta_k)$ is decreasing in $\theta_k$ over $[\underline{k}^0,\theta_k^* ]$ and increasing over $[\theta_k^* ,\overline{k}^1]$.

\section{Proof of Lemma \ref{lemma6}}

	Some notations are simplified by removing subscript $t$ as mentioned in Appendix \ref{app_optimal}.
	
	We proof this Lemma by contradiction.
	
	Let $\mathcal{V} = \{(\theta_a,\theta_b)|\theta_a\in [\phi_{\mathcal{C}}(\delta_b),\delta_a], \theta_b \in [\delta_b,\phi_{\mathcal{C}}^{-1}(\delta_a)], \Gamma_{\mathcal{C}}(\theta_a,\theta_b) = 0   \}$.
	
	Note that for \texttt{Simple}, \texttt{EqOpt}, \texttt{StatPar} fairness, for any $(\theta_a,\theta_b)$ and $(\theta_a',\theta_b')$ that satisfy constraints $\Gamma_{\mathcal{C}}(\theta_a,\theta_b)= 0$ and $\Gamma_{\mathcal{C}}(\theta_a',\theta_b')= 0$, $\theta_a \geq \theta_a'$ if and only if $\theta_b \geq \theta_b'$. 
	Suppose that $(\check{\theta}_a,\check{\theta}_b)$ satisfies $ \Gamma_{\mathcal{C}}(\check{\theta}_a,\check{\theta}_b) = 0$ and $(\check{\theta}_a,\check{\theta}_b) = \argmin_{\theta_a, \theta_b}\overline{\alpha}_aL_a(\theta_a)+\overline{\alpha}_bL_b(\theta_b) \notin \mathcal{V}$, then one of the following must hold: (1) $\check{\theta}_a < \phi_{\mathcal{C}}(\delta_b)$, $\check{\theta}_b < \delta_b$; (2)  $\check{\theta}_a > \delta_a$, $\check{\theta}_b > \phi_{\mathcal{C}}^{-1}(\delta_a)$. Consider two cases separately.
	
	(1) $\check{\theta}_a < \phi_{\mathcal{C}}(\delta_b)$, $\check{\theta}_b < \delta_b$
	
	Since $L_b(\check{\theta}_b)>L_b(\delta_b )$, $\forall \overline{\alpha}_a,\overline{\alpha}_b$, to satisfy $\overline{\alpha}_aL_a(\check{\theta}_a)+\overline{\alpha}_bL_b(\check{\theta}_b) < \overline{\alpha}_aL_a(\phi(\delta_b)) +  \overline{\alpha}_bL_b(\delta_b)$, $L_a(\check{\theta}_a)<L_a(\phi_{\mathcal{C}}(\delta_b))$ must hold. However, by Lemma \ref{lemma_new1}, $L_a(\theta_a)$ is strictly decreasing on $[\underline{a}^0,\delta_a]$ and strictly increasing on $[\delta_a,\overline{a}^1]$. Since $\check{\theta}_a < \phi_{\mathcal{C}}(\delta_b)<\delta_a$, this implies $L_a(\check{\theta}_a)>L_a(\phi_{\mathcal{C}}(\delta_b))$. Therefore, $(\check{\theta}_a,\check{\theta}_b)$ cannot be the optimal pair.
	
	(2) $\check{\theta}_a > \delta_a$, $\check{\theta}_b > \phi_{\mathcal{C}}^{-1}(\delta_a)$
	
	Since $L_a(\check{\theta}_a)>L_a(\delta_a )$,  $\forall \overline{\alpha}_a,\overline{\alpha}_b$, to satisfy $\overline{\alpha}_aL_a(\check{\theta}_a)+\overline{\alpha}_bL_b(\check{\theta}_b) < \overline{\alpha}_aL_a( \delta_a)+\overline{\alpha}_bL_b(\phi_{\mathcal{C}}^{-1}(\delta_a))$, $L_b(\check{\theta}_b)<L_b(\phi_{\mathcal{C}}^{-1}(\delta_a))$ must hold. However, by Lemma \ref{lemma_new1}, $L_b(\theta_b)$ is strictly decreasing on $[\underline{b}^0,\delta_b]$ and strictly increasing on $[\delta_b,\overline{b}^1]$. Since $\check{\theta}_b > \phi_{\mathcal{C}}^{-1}(\delta_a)>\delta_b$, this implies $L_b(\check{\theta}_b)>L_b(\phi_{\mathcal{C}}^{-1}(\delta_a))$. Therefore, $(\check{\theta}_a,\check{\theta}_b)$ cannot be the optimal pair.

\section{Proof of Theorem \ref{theorem5}}\label{proof:theorem5}
	Some notations are simplified by removing subscript $t$ as mentioned in Appendix \ref{app_optimal}.

Proof of Theorem \ref{theorem5} is based on the following Lemma.

\begin{lemma}\label{lemma8}
	Consider the one-shot problem \eqref{eq:opt} at some time step $t$, with group proportions given by 
	$\overline{\alpha}_a(t), \overline{\alpha}_b(t)$.  Under Assumption \ref{assumption2} the one-shot decision $(\theta_a(t),\theta_b(t))$ for this time step is unique and satisfies the following:   
	
	(1) Under \texttt{EqOpt} fairness: 
	\begin{itemize}[noitemsep,topsep=0pt]
		\item If $\theta_a(t) \in [\underline{a}^0,\underline{a}^1]$, $\theta_b(t) \in [\underline{b}^1,\overline{b}^0]$, then $\frac{\overline{\alpha}_a(t)}{\overline{\alpha}_b(t)} = (\frac{g_b^1}{g_b^0}\frac{f_b^1(\theta_b(t))}{f_b^0(\theta_b(t))} - 1)\frac{g_b^0}{g_a^0} $.
		\item If $\theta_a(t) \in [\underline{a}^1,\overline{a}^0]$, $\theta_b(t) \in [\underline{b}^1,\overline{b}^0]$, then 
		$\frac{\overline{\alpha}_a(t)}{\overline{\alpha}_b(t)} = \frac{ \frac{g_b^1}{g_b^0}\frac{f_b^1(\theta_b(t))}{f_b^0(\theta_b(t))}-1}{1-\frac{g_a^1}{g_a^0}\frac{f_a^1(\theta_a(t))}{f_a^0(\theta_a(t))}}\frac{g_b^0}{g_a^0} $.
	\end{itemize}
	
	(2) Under \texttt{StatPar} fairness: 
	\begin{itemize}[noitemsep,topsep=0pt]
		\item 	If $\theta_a(t) \in [\underline{a}^0,\underline{a}^1]$, $\theta_b(t) \in [\underline{b}^1,\overline{b}^0]$, then $\frac{\overline{\alpha}_a(t)}{\overline{\alpha}_b(t)} = 1-\frac{2}{\frac{g_b^1}{g_b^0}\frac{f_b^1(\theta_b(t))}{f_b^0(\theta_b(t))}+1}$ .
		\item If $\theta_a(t) \in [\underline{a}^1,\overline{a}^0]$, $\theta_b(t) \in [\underline{b}^1,\overline{b}^0]$, then $\frac{\overline{\alpha}_a(t)}{\overline{\alpha}_b(t)} = (1-\frac{2}{\frac{g_b^1}{g_b^0}\frac{f_b^1(\theta_b(t))}{f_b^0(\theta_b(t))}+1})(\frac{2}{1-\frac{g_a^1f_a^1(\theta_a(t))}{g_a^0f_a^0(\theta_a(t))}}-1 )$ .
		\item If $\theta_a(t) \in [\underline{a}^1,\overline{a}^0]$, $\theta_b(t) \in [\overline{b}^0,\overline{b}^1]$
		, then $\frac{\overline{\alpha}_a(t)}{\overline{\alpha}_b(t)} = \frac{2}{1-\frac{g_a^1f_a^1(\theta_a(t))}{g_a^0f_a^0(\theta_a(t))}}-1 $.
	\end{itemize}
	
	(3) Under \texttt{Simple} fairness:
	\begin{itemize}[noitemsep,topsep=0pt]
		\item If we further assume $\delta_a, \delta_b \in \mathcal{T}_a\cap \mathcal{T}_b$, then $\theta_a(t) = \theta_b(t) \in [ \underline{a}^1,\overline{b}^0]$ and
		$\frac{\overline{\alpha}_a(t)}{\overline{\alpha}_b(t)} = \frac{g_b^1f_b^1(\theta_b(t))-g_b^0f_b^0(\theta_b(t))}{g_a^0f_a^0(\theta_a(t))-g_a^1f_a^1(\theta_a(t))}$. 
	\end{itemize}
\end{lemma}

\begin{proof}

We focus on the case when $g_a^1f_a^1(\underline{a}^1)<g_a^0f_a^0(\underline{a}^1) ~\& ~g_a^1f_a^1(\overline{a}^0)>g_a^0f_a^0(\overline{a}^0)$ and $g_b^1f_b^1(\underline{b}^1)<g_b^0f_b^0(\underline{b}^1)~\& ~g_b^1f_b^1(\overline{b}^0)>g_b^0f_b^0(\overline{b}^0)$. That is, $\theta_k^* =\text{argmin}_{\theta}L_k(\theta)= \delta_k$ holds for $k\in\{a,b\}$.

Constraint $\Gamma_{\mathcal{C}}(\theta_a,\theta_b) = 0$ can be rewritten as $\theta_a = \phi_{\mathcal{C}}(\theta_b)$ for some strictly increasing function $\phi_{\mathcal{C}}$. The following holds:
\begin{eqnarray*}
\frac{d \phi_{\mathcal{C}}(\theta_b)}{d \theta_b} =- \frac{\frac{\partial \Gamma_\mathcal{C}(\theta_a,\theta_b)}{\partial \theta_b}}{\frac{\partial \Gamma_\mathcal{C}(\theta_a,\theta_b)}{\partial \theta_a}}\Big|_{\theta_a = \phi_{\mathcal{C}}(\theta_b)} = \begin{cases}
 \frac{f_b^0(\theta_b)}{f_a^0(\phi_{\mathcal{C}}(\theta_b))}, ~\mathcal{C} := \texttt{EqOpt }\\
 \frac{g_b^0f_b^0(\theta_b) + g_b^1f_b^1(\theta_b)}{g_a^0f_a^0(\phi_{\mathcal{C}}(\theta_b))+g_a^1f_a^1(\phi_{\mathcal{C}}(\theta_b))},~\mathcal{C} :=\texttt{StaPar } \\
 1, ~\mathcal{C} :=\texttt{Simple }
\end{cases}
\end{eqnarray*}

The one-shot problem can be expressed with only one variable, either $\theta_a$ or $\theta_b$. Here we express it in terms of $\theta_b$. At each round, decision
 maker finds $\theta_b(t) = \argmin_{\theta_b}L^t(\theta_b) = \overline{\alpha}_a(t) L_a(\phi_\mathcal{C} (\theta_b)) + \overline{\alpha}_b(t) L_b(\theta_b)$ and $\theta_a(t) = \phi_\mathcal{C} (\theta_b(t))$. Since $\phi_\mathcal{C}(\delta_b) < \delta_a$ ( $\phi_\mathcal{C}^{-1}(\delta_a) > \delta_b$), when $\mathcal{C}:=\texttt{StatPar}$, solution $(\theta_a(t),\theta_b(t))$ can be in one of the following three forms: (1) $\theta_a(t) \in [\underline{a}^0,\underline{a}^1]$, $\theta_b(t) \in [\underline{b}^1,\overline{b}^0]$; (2) $\theta_a(t) \in [\underline{a}^1,\overline{a}^0]$, $\theta_b(t) \in [\underline{b}^1,\overline{b}^0]$; (3) $\theta_a(t) \in [\underline{a}^1,\overline{a}^0]$, $\theta_b(t) \in [\overline{b}^0,\overline{b}^1]$. When $\mathcal{C}:=\texttt{EqOpt}$, solution $(\theta_a(t),\theta_b(t))$ can be either (1) or (2) listed above. In the following analysis, we simplify the notation $\phi_{\mathcal{C}}$ as $\phi$ when fairness criterion $\mathcal{C}$ is explicitly stated.
 For \texttt{EqOpt} and \texttt{StatPar} criteria, we consider each case separately. 

\textbf{Case 1: } $\theta_a(t) \in [\underline{a}^0,\underline{a}^1]$, $\theta_b(t) \in [\underline{b}^1,\overline{b}^0]$

Let $\theta_b^{\max} = \min\{\overline{b}^0,\phi_{\mathcal{C}}^{-1}(\underline{a}^1)\}$ be the maximum value $ \theta_b$ can take. $L^t(\theta_b) = \overline{\alpha}_b(t)\int_{\underline{b}^1}^{\theta_b}g_b^1f_b^1(x)-g_b^0f_b^0(x)dx - \overline{\alpha}_a(t)\int_{\underline{a}^0}^{\phi_{\mathcal{C}}(\theta_b)}g_a^0f_a^0(x)dx + \overline{\alpha}_a(t)g_a^0 + \overline{\alpha}_b(t)\int_{\underline{b}^1}^{\overline{b}^0}g_b^0f_b^0(x)dx$

Taking derivative w.r.t. $\theta_b$ gives $$\frac{d L^t(\theta_b)}{d \theta_b} = \overline{\alpha}_b(t)(g_b^1f_b^1(\theta_b)-g_b^0f_b^0(\theta_b)) - \overline{\alpha}_a(t) g_a^0f_a^0(\phi_{\mathcal{C}}(\theta_b))\frac{d \phi_{\mathcal{C}}(\theta_b)}{d \theta_b}.$$

1. ${\mathcal{C}} := \texttt{EqOpt} $

$\frac{d L^t(\theta_b)}{d \theta_b} = \overline{\alpha}_b(t)(g_b^1f_b^1(\theta_b)-g_b^0f_b^0(\theta_b)) - \overline{\alpha}_a(t) g_a^0f_b^0(\theta_b)$, since $g_b^1f_b^1(\theta_b)-g_b^0f_b^0(\theta_b)$ is increasing from negative to positive and $f_b^0(\theta_b)$ is decreasing over $[\underline{b}^1,\overline{b}^0]$, implying $\frac{d L^t(\theta_b)}{d \theta_b} $ is increasing over $[\underline{b}^1,\overline{b}^0]$. Based on the value of $\frac{\overline{\alpha}_a(t)}{\overline{\alpha}_b(t)}$,

$\bullet$ If $\frac{d L^t(\theta_b)}{d \theta_b}|_{\theta_b = {\theta}_b^{\max}} \geq 0$, then one-shot decision $\theta_b(t)$ satisfies $\frac{\overline{\alpha}_a(t)}{\overline{\alpha}_b(t)} = (\frac{g_b^1}{g_b^0}\frac{f_b^1(\theta_b(t))}{f_b^0(\theta_b(t))} - 1)\frac{g_b^0}{g_a^0} $ and is unique.

$\bullet$ If $\frac{d L^t(\theta_b)}{d \theta_b} < 0, \forall \theta_b \in [\underline{b}^1,{\theta}_b^{\max}]$, then $\theta_b(t) > {\theta}_b^{\max}$ and $(\theta_a(t),\theta_b(t))$ does not satisfy Case 1.

2. ${\mathcal{C}} := \texttt{StatPar} $

$\frac{d L^t(\theta_b)}{d \theta_b} = \overline{\alpha}_b(t)(g_b^1f_b^1(\theta_b)-g_b^0f_b^0(\theta_b)) - \overline{\alpha}_a(t)\frac{g_b^1f_b^1(\theta_b)+g_b^0f_b^0(\theta_b)}{1 + \frac{g_a^1f_a^1(\phi(\theta_b))}{g_a^0f_a^0(\phi(\theta_b))}} = 
(\overline{\alpha}_b(t)- \overline{\alpha}_a(t))g_b^1f_b^1(\theta_b) - (\overline{\alpha}_b(t)+ \overline{\alpha}_a(t))g_b^0f_b^0(\theta_b)$, where the last equality holds since $f_a^1(\phi(\theta_b)) = 0$ over $[\underline{a}^0,\underline{a}^1]$. Since $\frac{d L^t(\theta_b)}{d \theta_b}|_{\theta_b=\underline{b}^1}<0$, based on the value of $\frac{\overline{\alpha}_a(t)}{\overline{\alpha}_b(t)}$,

$\bullet$ If $\exists \theta_b'$ such that $\frac{d L^t(\theta_b)}{d \theta_b}|_{\theta_b = \theta_b'} \geq 0$, then one-shot decision $\theta_b(t)$ satisfies $\frac{\overline{\alpha}_a(t)}{\overline{\alpha}_b(t)} = 1-\frac{2}{\frac{g_b^1}{g_b^0}\frac{f_b^1(\theta_b(t))}{f_b^0(\theta_b(t))}+1}$ and is unique.

$\bullet$ If $\frac{d L^t(\theta_b)}{d \theta_b} < 0, \forall \theta_b \in [\underline{b}^1,{\theta}_b^{\max}]$, then $ \theta_b(t) > {\theta}_b^{\max}$ and $(\theta_a(t),\theta_b(t))$ does not satisfy Case 1.

\textbf{Case 2:} $\theta_a(t) \in [\underline{a}^1,\overline{a}^0]$, $\theta_b(t) \in [\underline{b}^1,\overline{b}^0]$

Let $\theta_b^{\max} = \min\{\overline{b}^0,\phi_\mathcal{C}^{-1}(\overline{a}^0) \}$ and $\theta_b^{\min} = \max\{\underline{b}^1,\phi_\mathcal{C}^{-1}(\underline{a}^1)\}$ be the maximum and minimum value that $\theta_b$ can take respectively. $L^t(\theta_b) = \overline{\alpha}_b(t)\int_{\underline{b}^1}^{\theta_b}g_b^1f_b^1(x)-g_b^0f_b^0(x)dx + \overline{\alpha}_a(t)\int_{\underline{a}^1}^{\phi_\mathcal{C}(\theta_b)}g_a^1f_a^1(x)-g_a^0f_a^0(x)dx + \overline{\alpha}_b(t)\int_{\underline{b}^1}^{\overline{b}^0}g_b^0f_b^0(x)dx  + \overline{\alpha}_a(t)\int_{\underline{a}^1}^{\overline{a}^0}g_a^0f_a^0(x)dx$

Taking derivative w.r.t. $\theta_b$ gives 
\begin{eqnarray*}
\frac{d L^t(\theta_b)}{d \theta_b} = \overline{\alpha}_b(t)(g_b^1f_b^1(\theta_b)-g_b^0f_b^0(\theta_b)) + \overline{\alpha}_a(t)(g_a^1f_a^1(\phi_\mathcal{C}(\theta_b))-g_a^0f_a^0(\phi_\mathcal{C}(\theta_b)))\frac{d \phi_\mathcal{C}(\theta_b)}{d \theta_b}.
\end{eqnarray*}

1. ${\mathcal{C}} := \texttt{EqOpt} $

$\frac{d L^t(\theta_b)}{d \theta_b} = ((g_a^1\frac{f_a^1(\phi(\theta_b))}{f_a^0(\phi(\theta_b))}-g_a^0)\overline{\alpha}_a(t)-g_b^0\overline{\alpha}_b(t))f_b^0(\theta_b) + g_b^1f_b^1(\theta_b)\overline{\alpha}_b(t)$. Since $\frac{d L^t(\theta_b)}{d \theta_b}|_{\theta_b = \theta_b^{\max}}>0$, based on $\frac{\overline{\alpha}_a(t)}{\overline{\alpha}_b(t)}$,

$\bullet$ If $\exists \theta_b'$ such that $\frac{d L^t(\theta_b)}{d \theta_b}|_{\theta_b = \theta_b'} \leq 0$, then one-shot decision $\theta_b(t)$ satisfies $\frac{\overline{\alpha}_a(t)}{\overline{\alpha}_b(t)} = \frac{1 - \frac{g_b^1}{g_b^0}\frac{f_b^1(\theta_b(t))}{f_b^0(\theta_b(t))}}{\frac{g_a^1}{g_a^0}\frac{f_a^1(\phi(\theta_b(t)))}{f_a^0(\phi(\theta_b(t)))}-1}\frac{g_b^0}{g_a^0} $ and is unique.

$\bullet$ If $\frac{d L^t(\theta_b)}{d \theta_b} > 0, \forall \theta_b \in  [\theta_b^{\min},\theta_b^{\max}]$, then $ \theta_b(t)<{\theta}_b^{\min}$ and $(\theta_a(t),\theta_b(t))$ does not satisfy Case 2.

2. ${\mathcal{C}} := \texttt{StatPar} $ 

$\frac{d L^t(\theta_b)}{d \theta_b} = \overline{\alpha}_b(t)(g_b^1f_b^1(\theta_b)-g_b^0f_b^0(\theta_b)) + \overline{\alpha}_a(t)(g_b^0f_b^0(\theta_b) + g_b^1f_b^1(\theta_b))\frac{g_a^1f_a^1(\phi(\theta_b))-g_a^0f_a^0(\phi(\theta_b)) }{g_a^1f_a^1(\phi(\theta_b))+g_a^0f_a^0(\phi(\theta_b))}$. 

$\bullet$ If $\exists \theta_b(t)$ such that $\frac{d L^t(\theta_b)}{d \theta_b}|_{\theta_b = \theta_b(t)} = 0$, then it satisfies
$\frac{\overline{\alpha}_a(t)}{\overline{\alpha}_b(t)} = (1-\frac{2}{\frac{g_b^1}{g_b^0}\frac{f_b^1(\theta_b(t))}{f_b^0(\theta_b(t))}+1})(\frac{2}{1-\frac{g_a^1f_a^1(\phi(\theta_b(t)))}{g_a^0f_a^0(\phi(\theta_b(t)))} }-1 )$ 
and is unique. 

$\bullet$ If $\frac{d L^t(\theta_b)}{d \theta_b} > 0, \forall \theta_b \in  [\theta_b^{\min},\theta_b^{\max}]$, then $\theta_b(t) < {\theta}_b^{\min}$ and $(\theta_a(t),\theta_b(t))$ does not satisfy Case 2.

$\bullet$ If $\frac{d L^t(\theta_b)}{d \theta_b} < 0, \forall \theta_b \in  [\theta_b^{\min},\theta_b^{\max}]$, then $ \theta_b(t) > {\theta}_b^{\max}$ and $(\theta_a(t),\theta_b(t))$ does not satisfy Case 2.

\textbf{Case 3:} $\theta_a(t) \in [\underline{a}^1,\overline{a}^0]$, $\theta_b(t) \in [\overline{b}^0,\overline{b}^1]$

Express $L^t(\theta_a,\theta_b)$ as function of $\theta_a$, the analysis will be similar to Case 1. 

Let $\theta_a^{\min} = \max\{\underline{a}^1,\phi_{\mathcal{C}}(\overline{b}^0)\}$ be the minimum value $ \theta_a$ can take.

$L^t(\theta_a) = \overline{\alpha}_a(t)\int_{\underline{a}^1}^{\theta_a}g_a^1f_a^1(x)-g_a^0f_a^0(x)dx + \overline{\alpha}_b(t)\int_{\underline{b}^1}^{\phi_{\mathcal{C}}^{-1}(\theta_a)}g_b^1f_b^1(x)dx + \overline{\alpha}_a(t)\int_{\underline{a}^1}^{\overline{a}^0}g_a^0f_a^0(x)dx$

Taking derivative w.r.t. $\theta_a$ gives $$\frac{d L^t(\theta_a)}{d \theta_a} = \overline{\alpha}_a(t)(g_a^1f_a^1(\theta_a)-g_a^0f_a^0(\theta_a)) + \overline{\alpha}_b(t) g_b^1f_b^1(\phi_{\mathcal{C}}^{-1}(\theta_a))\frac{d \phi_{\mathcal{C}}^{-1}(\theta_a)}{d \theta_a},$$where ${\mathcal{C}} := \texttt{StatPar} $.

	$\frac{d L^t(\theta_a)}{d \theta_a} = \overline{\alpha}_a(t)(g_a^1f_a^1(\theta_a)-g_a^0f_a^0(\theta_a)) + \overline{\alpha}_b(t)\frac{g_a^1f_a^1(\theta_a)+g_a^0f_a^0(\theta_a)}{1 + \frac{g_b^0f_b^0(\phi^{-1}(\theta_a))}{g_b^0f_b^0(\phi^{-1}(\theta_a))}} =  \overline{\alpha}_a(t)(g_a^1f_a^1(\theta_a)-g_a^0f_a^0(\theta_a)) + \overline{\alpha}_b(t)(g_a^1f_a^1(\theta_a)+g_a^0f_a^0(\theta_a) )$, where the last equality holds since $f_b^0(\phi^{-1}(\theta_a)) = 0$ over $[\overline{b}^0,\overline{b}^1]$. Since $\frac{d L^t(\theta_b)}{d \theta_b}|_{\theta_b=\overline{a}^0}>0$, based on the value of $\frac{\overline{\alpha}_a(t)}{\overline{\alpha}_b(t)}$,

$\bullet$ If $\exists \theta_a'$ such that $\frac{d L^t(\theta_a)}{d \theta_a}|_{\theta_a = \theta_a'} \leq 0$, then one-shot decision $\theta_a(t)$ satisfies
$\frac{\overline{\alpha}_a(t)}{\overline{\alpha}_b(t)} = \frac{2}{1-\frac{g_a^1f_a^1(\theta_a(t))}{g_a^0f_a^0(\theta_a(t))}}-1 $ and is unique. 

$\bullet$  If $\frac{d L^t(\theta_a)}{d \theta_a} > 0, \forall \theta_a \in [\underline{b}^1,{\theta}_a^{\min}]$, then $ \theta_a(t) < {\theta}_a^{\min}$ and $(\theta_a(t),\theta_b(t))$ does not satisfy Case 3.

Now consider the case when ${\mathcal{C}} := \texttt{Simple} $, where $\theta_a(t) = \theta_b(t) = \theta(t)$. Since $\delta_a>\delta_b$, suppose that both $\delta_a, \delta_b \in \mathcal{T}_a \cap \mathcal{T}_b$ and according to Lemma \ref{lemma6}, there could be only one case: $\theta(t) \in [ \underline{a}^1,\overline{b}^0]$.  

Taking derivative w.r.t. $\theta$ gives $$\frac{d L^t(\theta)}{d \theta} = \overline{\alpha}_b(t)(g_b^1f_b^1(\theta)-g_b^0f_b^0(\theta)) + \overline{\alpha}_a(t)(g_a^1f_a^1(\theta)-g_a^0f_a^0(\theta)).$$ 

$\frac{d L^t(\theta)}{d \theta}$ is increasing from negative to positive over $[\delta_b,\delta_a]$, $\exists \theta(t)$ such that  $\frac{d L^t(\theta)}{d \theta}|_{\theta = \theta(t)} = 0$, and it satisfies $\frac{\overline{\alpha}_a(t)}{\overline{\alpha}_b(t)} = \frac{g_b^1f_b^1(\theta(t))-g_b^0f_b^0(\theta(t))}{g_a^0f_a^0(\theta(t))-g_a^1f_a^1(\theta(t))}$.

\end{proof}

By Lemma \ref{lemma6}, $\theta_a(t)\in [\phi_{\mathcal{C}}(\delta_b),\delta_a], \theta_b(t) \in [\delta_b,\phi_{\mathcal{C}}^{-1}(\delta_a)]$ hold. Under Assumption \ref{assumption2}, $f_b^1f_b^1(\theta_b)\geq f_b^0f_b^0(\theta_b)$ for $\theta_b\in[\delta_b,\overline{b}^0]$, $f_a^1f_a^1(\theta_a)\leq f_a^0f_a^0(\theta_a)$ for $\theta_a\in[\underline{a}^1,\delta_a]$. Moreover, $f_k^1(x)$ is increasing and $f_k^0(x)$ is decreasing over $\mathcal{T}_k$. According to Lemma \ref{lemma8}, for each case, function $ \Psi_{\mathcal{C}}(\theta_a(t),\theta_b(t))$ is  increasing in $\theta_a(t)$ and $\theta_b(t)$.

\section{Proof of Theorem \ref{proposition1}}\label{app_dynamic}

	Some notations are simplified by removing subscript $t$ as mentioned in Appendix \ref{app_optimal}.
	
Note that $f_{k,t}(x) = f_k(x)$ is fixed. Consider two one-shot problems under the same distributions at two consecutive time steps with group representation disparity $\frac{\widetilde{\overline{\alpha}}_a}{\widetilde{\overline{\alpha}}_b}$ and $\frac{\widehat{\overline{\alpha}}_a}{\widehat{\overline{\alpha}}_b}$ respectively. Let $(\widetilde{\theta}_a,\widetilde{\theta}_b)$ and $(\widehat{\theta}_a,\widehat{\theta}_b)$ be the corresponding solutions.

According to Lemma \ref{lemma6}, $\widetilde{\theta}_a, \widehat{\theta}_a\in[\phi_{\mathcal{C}}(\delta_b),\delta_a]$, $\widetilde{\theta}_b,\widehat{\theta}_b\in [\delta_b,\phi_{\mathcal{C}}^{-1}(\delta_a)]$ hold. Suppose $\frac{\widetilde{\overline{\alpha}}_a(t)}{\widetilde{\overline{\alpha}}_b(t)}>\frac{\widehat{\overline{\alpha}}_a}{\widehat{\overline{\alpha}}_b}$. By Theorem \ref{theorem5}, it implies that $\widetilde{\theta}_a>\widehat{\theta}_a$, $\widetilde{\theta}_b>\widehat{\theta}_b$. 

Consider the dynamics with $\pi_k(\theta_k) = \nu(L_k(\theta_k))$, since $L_k(\theta_k)$ is decreasing over $[\underline{k}^0,\delta_k]$ and increasing over $[\delta_k,\overline{k}^1]$, the larger one-shot decisions $\theta_a$, $\theta_b$ would result in the larger retention rate $\pi_a(\theta_a)$ and the smaller $\pi_b(\theta_b)$ as $\nu(\cdot)$ is strictly decreasing. Therefore, $\pi_a(\widetilde{\theta}_a)>\pi_a(\widehat{\theta}_a)$ and $\pi_b(\widetilde{\theta}_b)<\pi_b(\widehat{\theta}_b)$. Hence, Monotonicity condition is satisfied. 

Consider the dynamics with $\pi_k(\theta_k) = w(D_k(\theta_k))$ where $D_k(\theta_k) = \int_{\theta_k}^{\infty}g_k^1f_k^1(x) - g_k^0f_k^0(x)dx$. The following holds for $G_a$ and $G_b$:
\begin{eqnarray*}
D_a(\theta_a) = \int_{\delta_a}^{\infty}g_a^1f_a^1(x) - g_a^0f_a^0(x)dx +  \int_{\theta_a}^{\delta_a}g_a^1f_a^1(x) - g_a^0f_a^0(x)dx\\
D_b(\theta_b) = \int_{\delta_b}^{\infty}g_b^1f_b^1(x) - g_b^0f_b^0(x)dx - \int_{\delta_b}^{\theta_b}g_b^1f_b^1(x) - g_b^0f_b^0(x)dx
\end{eqnarray*}
Since $g_a^1f_a^1(x) \leq g_a^0f_a^0(x)$ for $x\leq \delta_a$ and $g_b^1f_b^1(x) \geq g_b^0f_b^0(x)$ for $x\geq \delta_b$, the larger $\theta_a$, $\theta_b$ will thus result in the larger $\pi_a(\theta_a)$ and smaller $\pi_b(\theta_b)$ as $w(\cdot)$ is strictly increasing. Therefore, $\pi_a(\widetilde{\theta}_a)>\pi_a(\widehat{\theta}_a)$ and $\pi_b(\widetilde{\theta}_b)<\pi_b(\widehat{\theta}_b)$. Hence, Monotonicity condition is satisfied.

Combine with Theorem \ref{thm1}, $\frac{\overline{\alpha}_a(t)}{\overline{\alpha}_a(t)} $ changes monotonically. By Theorem \ref{theorem5}, the corresponding one-shot fair decision $(\theta_a(t),\theta_b(t))$ also converges monotonically.

\section{Proof of Theorem \ref{thm2}}\label{app_reshape}
\subsection{Lemmas}
To begin, we first introduce some lemmas for two cases. Lemma \ref{lemma_reshape} and \ref{lemma_reshape1} show that under the same group representation $\overline{\alpha}_a$, $\overline{\alpha}_b$, the impact of reshaping distributions on the resulting one-shot decisions. Lemma \ref{lemma_reshapeC} and \ref{lemma_reshape1C} demonstrate a sufficient condition on feature distributions and one-shot decisions of two problems such that their expected losses satisfy certain conditions. The proof of these lemmas are presented in Appendix \ref{app_lemmas}.

\textbf{Case (i): }$f_{k,t}(x) = g_{k,t}^1f_{k}^1(x) + g_{k,t}^0f_{k}^0(x) $:

Fraction of subgroup $G_k^j$ over $G_k$ changes according to change of their own perceived loss $L^j_k$, i.e., for
$i\in\{0,1\}$ such that $L_{k,t}^i(\theta_k(t))<L_{k,t-1}^i(\theta_k(t-1))$, $g_{k,t}^i>g_{k,t-1}^i$ and $g_{k,t}^{- i}<g_{k,t-1}^{-i}$.

\begin{lemma}\label{lemma_reshape}
	Let $(\widehat{\theta}_a,\widehat{\theta}_b)$, $(\widetilde{\theta}_a,\widetilde{\theta}_b)$ be two pairs of decisions under any of $\texttt{EqOpt}, \texttt{StatPar}, \texttt{Simple}$ fairness criteria such that $\widehat{\Psi}_{\mathcal{C}}(\widehat{\theta}_a,\widehat{\theta}_b) = \widetilde{\Psi}_{\mathcal{C}}(\widetilde{\theta}_a,\widetilde{\theta}_b)$, where functions $\widehat{\Psi}_{\mathcal{C}}$, $\widetilde{\Psi}_{\mathcal{C}}$ have the form given in Table \ref{table1} and are defined under feature distributions $\widehat{f}_{k}(x) = \widehat{g}_{k}^1f_k^1(x)+\widehat{g}_{k}^0f_k^0(x)$, $\widetilde{f}_{k}(x) = \widetilde{g}_{k}^1f_k^1(x)+\widetilde{g}_{k}^0f_k^0(x)$ respectively $\forall k\in\{a,b\}$. If $\widehat{g}_{k}^1<\widetilde{g}_{k}^1$ and $\widehat{g}_{k}^0>\widetilde{g}_{k}^0$, then $\widehat{\theta}_k> \widetilde{\theta}_k$ will hold $\forall k\in\{a,b\}$.
\end{lemma}

\begin{lemma}\label{lemma_reshapeC}
	Consider two one-shot problems defined in \eqref{eq:opt} with objectives $\widetilde{\pmb{O}}(\theta_a,\theta_b;\widetilde{\overline{\alpha}}_{a}, \widetilde{\overline{\alpha}}_{b})$ and $\widehat{\pmb{O}}(\theta_a,\theta_b;\widehat{\overline{\alpha}}_{a}, \widehat{\overline{\alpha}}_{b})$, where $\widetilde{\pmb{O}}$ is defined over distributions $\widetilde{f}_{k}(x) = \widetilde{g}_{k}^0f_{k}^0(x)+\widetilde{g}_{k}^1f_{k}^1(x)$ and $\widehat{\pmb{O}}$ is defined over distributions $\widehat{f}_{k}(x) = \widehat{g}_{k}^0f_{k}^0(x)+\widehat{g}_{k}^1f_{k}^1(x)$, $k\in\{a,b\}$. Let $(\widetilde{\theta}_{a},\widetilde{\theta}_{b})$, $(\widehat{\theta}_{a},\widehat{\theta}_{b})$ be the corresponding one-shot decisions under any of \texttt{Simple}, \texttt{EqOpt} or \texttt{StatPar} fairness criteria. For any $\widehat{g}_{k}^0 + \widehat{g}_{k}^1 = 1$ and  $\widetilde{g}_{k}^0 + \widetilde{g}_{k}^1 = 1$ such that $\widehat{g}_{k}^0 > \widetilde{g}_{k}^0$, $\widehat{g}_{k}^1 < \widetilde{g}_{k}^1$, $\forall k\in\{a,b\}$, if $\widehat{\theta}_{a}>\widetilde{\theta}_{a}$ and $\widehat{\theta}_{b}>\widetilde{\theta}_{b}$, then $\widehat{L}_{a}(\widehat{\theta}_{a})<\widetilde{L}_{a}(\widetilde{\theta}_{a})$ and $\widehat{L}_b(\widehat{\theta}_{b})>\widetilde{L}_{b}(\widetilde{\theta}_{b})$ can be satisfied under the following condition: 
	\begin{eqnarray}\label{condition1}
	|\Delta g_k (\widetilde{L}_{k}^0(\widetilde{\theta}_{k})-\widetilde{L}_{k}^1(\widetilde{\theta}_{k}))| < |\int_{\widetilde{\theta}_{k}}^{\widehat{\theta}_{k}} \widehat{g}_{k}^0f_k^0(x) - \widehat{g}_{k}^1f_k^1(x )   dx|, ~\forall k\in\{a,b\}
	\end{eqnarray}
	where $\Delta g_k = |\widehat{g}_{k}^0 -\widetilde{g}_{k}^0 | = |\widehat{g}_{k}^1 -\widetilde{g}_{k}^1 |$.
\end{lemma}

Note that Condition \eqref{condition1} can be satisfied when: (1) $\Delta g_k$ is sufficiently small; and (2) the difference in the decision $\widehat{\theta}_{k} - \widetilde{\theta}_{k} $ is sufficiently large, which can be achieved if $\widehat{\overline{\alpha}}_{k}$ and $\widetilde{\overline{\alpha}}_{k}$ are quite different.

\textbf{Case (ii): }$f_{k,t}(x) = g_{k}^1f_{k,t}^1(x) + g_{k}^0f_{k,t}^0(x) $

Suppose $L_{k,t}^1(\theta_k(t))>L_{k,t-1}^1(\theta_k(t-1))$, i.e., $G_k^1$ is less and less favored by the decision over time, then users from $G_k^1$ will make additional effort to improve their features so that $f_{k,t}^1(x)$ will skew toward the direction of higher feature value, i.e., $f^1_{k,t+1}(x)<f^1_{k,t}(x)$ for $x$ with smaller value ($x\in \mathcal{T}_k$) while $G_k^0$ is assumed to be unaffected, i.e.,  $f^0_{k,t+1}(x)=f^0_{k,t}(x)$. Similar statements hold when $\theta_k(t)<\theta_k(t-1)$ and $G_k^0$ is less and less favored. Moreover, assume that Assumption \ref{assumption2} holds for any reshaped distributions and the support of $f_{k,t}^1(x)$ and $f_{k,t}^0(x)$ do not change over time.   


$\forall t$, let $f_{k,t}^0(x)$ and $f_{k,t}^1(x)$ overlap over $\mathcal{T}_k:=[\underline{k}^1,\overline{k}^0]$.

\begin{lemma}\label{lemma_reshape1}
	Let $(\widehat{\theta}_a,\widehat{\theta}_b)$, $(\widetilde{\theta}_a,\widetilde{\theta}_b)$ be two pairs of decisions under any of $\texttt{EqOpt}, \texttt{StatPar}, \texttt{Simple}$ fairness criteria such that $\widehat{\Psi}_{\mathcal{C}}(\widehat{\theta}_a,\widehat{\theta}_b) = \widetilde{\Psi}_{\mathcal{C}}(\widetilde{\theta}_a,\widetilde{\theta}_b)$, where functions $\widehat{\Psi}_{\mathcal{C}}$, $\widetilde{\Psi}_{\mathcal{C}}$ have the form given in Table \ref{table1} and are defined under feature distributions $\widehat{f}_{k}(x) = g_{k}^1\widehat{f}_{k}^1(x)+g_{k}^0\widehat{f}_{k}^0(x)$, $\widetilde{f}_{k}(x) = g_{k}^1\widetilde{f}_{k}^1(x)+g_{k}^0\widetilde{f}_{k}^0(x)$ respectively $\forall k\in\{a,b\}$. If $\widehat{f}_{k}^0(x) =\widetilde{f}_{k}^0(x)$ and $\widehat{f}_{k}^1(x) < \widetilde{f}_{k}^1(x)$, $\forall x \in\mathcal{T}_k$, then $\widehat{\theta}_k> \widetilde{\theta}_k $ will hold $\forall k\in\{a,b\}$.
\end{lemma}

\begin{lemma}\label{lemma_reshape1C}
	Consider two one-shot problems defined in \eqref{eq:opt} with objectives $\widetilde{\pmb{O}}(\theta_a,\theta_b;\widetilde{\overline{\alpha}}_{a},\widetilde{ \overline{\alpha}}_{b})$ and $\widehat{\pmb{O}}(\theta_a,\theta_b;\widehat{\overline{\alpha}}_{a},\widehat{ \overline{\alpha}}_{b})$, where $\widetilde{\pmb{O}}$ is defined over distributions $\widetilde{f}_{k}(x) = g_{k}^0\widetilde{f}_{k}^0(x)+g_{k}^1\widetilde{f}_{k}^1(x)$ and $\widehat{\pmb{O}}$ is defined over distributions $\widehat{f}_{k}(x) = g_{k}^0\widehat{f}_{k}^0(x)+g_{k}^1\widehat{f}_{k}^1(x)$, $k\in\{a,b\}$. Let $(\widetilde{\theta}_{a},\widetilde{\theta}_{b})$, $(\widehat{\theta}_{a},\widehat{\theta}_{b})$ be the corresponding one-shot decisions under any of \texttt{Simple}, \texttt{EqOpt} or \texttt{StatPar} fairness criteria. For any distributions $\widetilde{f}_{k}^1$, $\widehat{f}_{k}^1$ increasing over $\mathcal{T}_k$ and $\widetilde{f}_{k}^0$, $\widehat{f}_{k}^0$ decreasing over $\mathcal{T}_k$ such that $\widehat{f}_{k}^1(x) <\widetilde{f}_{k}^1(x)$ over $\mathcal{T}_k$ and $\widehat{f}_{k}^0(x) =\widetilde{f}_{k}^0(x)={f}_{k}^0(x), \forall x$, $\forall k\in\{a,b\}$. if $\widehat{\theta}_{a}>\widetilde{\theta}_{a}$ and $\widehat{\theta}_{b}>\widetilde{\theta}_{b}$, then $\widehat{L}_{a}(\widehat{\theta}_{a})<\widetilde{L}_{a}(\widetilde{\theta}_{a})$ holds. Moreover, $\widehat{L}_{b}(\widehat{\theta}_{b})>\widetilde{L}_{b}(\widetilde{\theta}_{b})$ can be satisfied under the following condition: 
	\begin{eqnarray}\label{condition2}
	\Delta f_b^1 g_b^1(\max\{\widetilde{\theta}_{b},\widehat{\delta}_{b}\}-\underline{b}^1) < \int_{\max\{\widetilde{\theta}_{b},\widehat{\delta}_{b}\}}^{\widehat{\theta}_{b}} g_b^1\widehat{f}_{b}^1(x)-g_b^0\widehat{f}_{b}^0(x)dx 
	\end{eqnarray}
	where 
	$\Delta f_b^1 = \max_{x\in[\underline{b}^1,\max\{\widetilde{\theta}_{b},\widehat{\delta}_{b}\}]}|\widehat{f}_{b}^1(x)-\widetilde{f}_{b}^1(x)|$ and $\widehat{\delta}_{b}$ is defined such that $g_{b}^0\widehat{f}_{b}^0(\widehat{\delta}_{b} ) = g_{b}^1\widehat{f}_{b}^1(\widehat{\delta}_{b} ) $.
\end{lemma}

Note that Condition \eqref{condition2} can be satisfied when: (1) $\Delta f_b^1$ is sufficiently small, which makes $\widehat{\delta}_{b}$ close to $\widetilde{\delta}_{b}$ and $\widetilde{\theta}_{b} = \max\{\widetilde{\theta}_{b},\widehat{\delta}_{b}\}$ is more likely to hold; and (2) the difference in the decision $\widehat{\theta}_{b} - \widetilde{\theta}_{b} $ is sufficiently large, which can be achieved if $\widehat{\overline{\alpha}}_{k}$ and $\widetilde{\overline{\alpha}}_{k}$ are quite different.

\subsection{Sufficient conditions}

Below we formally state the sufficient condition under which Theorem \ref{thm2} can hold. 

\begin{con}\label{exacerbation}
	[Sufficient condition for exacerbation]  Condition \ref{exacerbation} is satisfied if the following holds:
	\begin{itemize}
		\item under \textbf{Case (i)}: Condition \eqref{condition1} is satisfied for objectives $\pmb{O}_{t}$ and $\pmb{O}_{t+1}$, $\forall t\geq 2$, i.e., 
		$$|\Delta g_{k,t+1} (L_{k,t}^0(\theta_{k}^r(t))-L_{k,t}^1(\theta_{k}^r(t)))| < |\int_{\theta_{k}^r(t)}^{\theta_{k}^r(t+1)} g_{k,t+1}^0f_k^0(x) - g_{k,t+1}^1f_k^1(x )  dx|, k\in\{a,b\}$$
		with $\Delta g_{k,t+1} = |g_{k,t+1}^j - g_{k,t}^j| , j\in\{0,1\}$.
		\item under \textbf{Case (ii)}: Condition \eqref{condition2} is satisfied for objectives $\pmb{O}_{t}$ and $\pmb{O}_{t+1}$, $\forall t\geq 2$, i.e., 
		$$\Delta f_{b,t+1}^1 g_b^1(\max\{\theta_b^r(t),\delta_{b,t+1}\}-\underline{b}^1) < \int_{\max\{\theta_b^r(t),\delta_{b,t+1}\}}^{\theta_b^r(t+1)} g_b^1f_{b,t+1}^1(x)-g_b^0f_{b,t+1}^0(x)dx $$
		with $\Delta f_{b,t+1}^1= \max_{x\in[\underline{b}^1,\max\{\theta_b^r(t),\delta_{b,t+1}\}]}|f_{b,t+1}^1(x)-f_{b,t}^1(x)|$.
	\end{itemize}
\end{con}

\begin{con}\label{accelerate}
	[Sufficient condition for acceleration of exacerbation]  
	
	Let $\pmb{O}^f_t := \pmb{O}^f_t(\theta_a,\theta_b;\overline{\alpha}_a^f(t),\overline{\alpha}_b^f(t)) $ be the objective of the one-shot problem at time $t$ for the case when distributions are fixed over time. Condition \ref{accelerate} is satisfied if the following holds:
	\begin{itemize}
		\item under \textbf{Case (i)}: Condition \eqref{condition1} is satisfied for objectives $\pmb{O}_{t}$ and $\pmb{O}_{t}^f$, $\forall t\geq 2$, i.e., 
		$$|\Delta g_{k,t} (L_{k,t}^0(\theta_{k}^f(t))-L_{k,t}^1(\theta_{k}^f(t)))| < |\int_{\theta_{k}^f(t)}^{\theta_{k}^r(t)} g_{k,t}^0f_k^0(x) - g_{k,t}^1f_k^1(x )  dx|, k\in\{a,b\}$$
		with $\Delta g_{k,t} = g_{k,t}^j - g_{k,1}^j , j\in\{0,1\}$.
		\item under \textbf{Case (ii)}: Condition \eqref{condition2} is satisfied for objectives $\pmb{O}_{t}$ and $\pmb{O}_{t}^f$, $\forall t\geq 2$, i.e., 
		$$\Delta f_{b,t}^1 g_b^1(\max\{\theta_b^f(t),\delta_{b,t}\}-\underline{b}^1) < \int_{\max\{\theta_b^f(t),\delta_{b,t}\}}^{\theta_b^r(t)} g_b^1f_{b,t}^1(x)-g_b^0f_{b,t}^0(x)dx $$
		with $\Delta f_{b,t}^1= \max_{x\in[\underline{b}^1,\max\{\theta_b^f(t),\delta_{b,t}\}]}|f_{b,t}^1(x)-f_{b,1}^1(x)|$.
	\end{itemize}
\end{con}

Note that Condition \ref{exacerbation} is likely to be satisfied when changing the decision from $\theta_k(t)$ to $\theta_k(t+1)$ results in: (i) a minor change of $f_{k,t+1}(x)$ from $f_{k,t}(x)$; or/and (ii) a significant change of representation disparity $\frac{\overline{\alpha}_a(t+1)}{\overline{\alpha}_b(t+1)}$ from $\frac{\overline{\alpha}_a(t)}{\overline{\alpha}_b(t)}$ so that $|\theta_k^r(t+1)-\theta_k^r(t)|$ is sufficiently large.

Condition \ref{accelerate} is likely to be satisfied if for any time step, (i) the change of $f_{k,t}(x)$ is minor as compared to the fixed distribution, i.e., $f_{k,1}(x)$ at time $t=1$; or/and (ii) the resulting decisions at same time under two schemes are quite different, i.e., $|\theta_k^f(t)-\theta_k^r(t)|$ is sufficiently large.

In other words, both requires that $f_{k,t}(x)$ is relatively insensitive to the change of one-shot decisions, and this applies to scenarios where the impact of reshaping distributions is considered as a slow process, e.g., change of credit score takes time and is a slow process.

\subsection{Proof of main theorem}


If $f_{k,t}(x) = f_k(x)$ is fixed $\forall t$, then the relationship between $\frac{\overline{\alpha}_a^f(t)}{\overline{\alpha}_b^f(t)}$ and one-shot solutions $(\theta_a^f(t),\theta_b^f(t))$ follows $\frac{\overline{\alpha}_a^f(t)}{\overline{\alpha}_b^f(t)} = \Psi_{\mathcal{C},1}(\theta_a^f(t),\theta_b^f(t)) ,\forall t$. If $f_{k,t}(x)$ varies over time
, then $\frac{\overline{\alpha}_a^r(t)}{\overline{\alpha}_b^r(t)} = \Psi_{\mathcal{C},t}(\theta_a^r(t),\theta_b^r(t)) ,\forall t$. We consider that distributions start to change after individuals feel the change of perceived decisions, i.e., $f_{k,t}(x)$ begins to change at time $t=3$. In the following $\forall k\in\{a,b\}$, $\theta_k^f(t)= \theta_k^r(t)= \theta_k(t)$, $\pi_{k,t}^f(\theta_k^f(t)) =\pi_{k,t}^r(\theta_k^r(t)) =\pi_{k,t}(\theta_k(t)) $ for $t=1,2$ and $\frac{\overline{\alpha}_a^f(t)}{\overline{\alpha}_b^f(t)}=\frac{\overline{\alpha}_a^r(t)}{\overline{\alpha}_b^r(t)}=\frac{\overline{\alpha}_a(t)}{\overline{\alpha}_b(t)}$ for $t=1,2,3$.

Start from $t=1$, if $(\theta_a(1),\theta_b(1))$ satisfies $\pi_{a,1}(\theta_a(1))>\pi_{b,1}(\theta_b(1))$, then $\frac{\overline{\alpha}_a(2)}{\overline{\alpha}_b(2)}>\frac{\overline{\alpha}_a(1)}{\overline{\alpha}_b(1)}$ and $\theta_k(2)>\theta_k(1)$ holds $\forall k\in\{a,b\}$, implying $\pi_{a,2}(\theta_a(2))>\pi_{a,1}(\theta_a(1))>\pi_{b,1}(\theta_b(1))>\pi_{b,2}(\theta_b(2))$ (\textit{$\pmb{O}_1$ and $\pmb{O}_2$ satisfy monotonicity condition}) and $\frac{\overline{\alpha}_a(3)}{\overline{\alpha}_b(3)}>\frac{\overline{\alpha}_a(2)}{\overline{\alpha}_b(2)}$.  Moreover, the change of decisions begins to reshape the feature distributions in the next time step. 

Consider two ways of reshaping distributions: \textbf{Case (i)} and \textbf{Case (ii)}. For both cases, show that as long as the change of distribution from $f_{k,t-1}(x)$ to $f_{k,t}(x)$ is relatively small w.r.t. the change of decision from $\theta_k(t-2)$ to $\theta_k(t-1)$ (formally stated in Condition \ref{exacerbation} and Condition \ref{accelerate}), the following can hold for any time step $t\geq 3$: \textit{(i)} $\pmb{O}_t$ and $\pmb{O}_{t+1}$ satisfy monotonicity condition: $\pi_{a,t+1}^r(\theta_a^r(t+1)) > \pi_{a,t}^r(\theta_a^r(t))$, $\pi_{b,t}^r(\theta_b^r(t)) > \pi_{b,t+1}^r(\theta_b^r(t+1))$ hold when $\frac{\overline{\alpha}_a^r(t+1)}{\overline{\alpha}_b^r(t+1)}>\frac{\overline{\alpha}_a^r(t)}{\overline{\alpha}_b^r(t)}$; \textit{(ii)} group representation disparity changes faster than case when distributions are fixed, i.e., $\frac{\overline{\alpha}_a^r(t)}{\overline{\alpha}_b^r(t)}\geq\frac{\overline{\alpha}_a^f(t)}{\overline{\alpha}_b^f(t)}, \forall t$.

Since $\theta_k(2)>\theta_k(1)$, within the same group $G_k$, subgroup $G_k^1$ (resp. $G_k^0$) experiences the higher (resp. lower) loss at time $t=2$ than $t=1$. Consider two types of change $\forall k\in\{a,b\}$:

 $\bullet$  \textbf{Case (i)}: $g_{k,3}^1<g_{k,2}^1 = g_{k,1}^1$ and $g_{k,3}^0>g_{k,2}^0 = g_{k,1}^0$.

 $\bullet$ \textbf{Case (ii)}: $f_{k,{3}}^0(x) = f_{k,{2}}^0(x)= f_{k,{1}}^0(x), \forall x$ and $f_{k,{3}}^1(x) < f_{k,{2}}^1(x)=f_{k,{1}}^1(x)$$,\forall x \in\mathcal{T}_k$.

Prove the following by induction under Condition \ref{exacerbation} and \ref{accelerate} (on the sensitivity of $f_{k,t}(x)$ w.r.t. the change of decisions): 
For $t>3$, $\frac{\overline{\alpha}_a^r(t+1)}{\overline{\alpha}_b^r(t+1)}>\frac{\overline{\alpha}_a^f(t+1)}{\overline{\alpha}_b^f(t+1)}$ and $\frac{\overline{\alpha}_a^r(t+1)}{\overline{\alpha}_b^r(t+1)}>\frac{\overline{\alpha}_a^r(t)}{\overline{\alpha}_b^r(t)}$ hold, and $\forall k\in\{a,b\}$:

 $\bullet$ \textbf{Case (i)}: $g_{k,t+1}^1<g_{k,t}^1<g_{k,1}^1$ and $g_{k,t+1}^0>g_{k,t}^0>g_{k,1}^0$ are satisfied.

 $\bullet$ \textbf{Case (ii)}: $f_{k,{t+1}}^0(x) = f_{k,{t}}^0(x)=f_{k,{1}}^0(x), \forall x$ and $f_{k,{t+1}}^1(x) < f_{k,{t}}^1(x)<f_{k,{1}}^1(x)$$,\forall x \in\mathcal{T}_k$.

\textbf{Base case:}

$\Psi_{\mathcal{C},t}$ are defined under feature distributions $f_{k,t}(x) =  g_{k,t}^1f_{k,t}^1(x)+g_{k,t}^0f_{k,t}^0(x)$, $\forall k\in\{a,b\}$. Define a pair $(\tilde{\theta}_a,\tilde{\theta}_b) $ such that the following holds:

$\frac{\overline{\alpha}_a(3)}{\overline{\alpha}_b(3)} = \Psi_{\mathcal{C},1}(\theta_a^f(3),\theta_b^f(3)) = \Psi_{\mathcal{C},3}(\theta_a^r(3),\theta_b^r(3))= \Psi_{\mathcal{C},2}(\tilde{\theta}_a,\tilde{\theta}_b)> \Psi_{\mathcal{C},2}(\theta_a^r(2),\theta_b^r(2)) = \frac{\overline{\alpha}_a(2)}{\overline{\alpha}_b(2)} $.

Then, we have $\forall k\in\{a,b\}$:

 $\bullet$ \textbf{Case (i)}: As $g_{k,3}^1<g_{k,2}^1=g_{k,1}^1$ and $g_{k,3}^0>g_{k,2}^0=g_{k,1}^0$, by Lemma \ref{lemma_reshape}, $\theta_k^r(3) > \theta_k^f(3)=\tilde{\theta}_k$ holds. 

 $\bullet$ \textbf{Case (ii)}: As $f_{k,{3}}^0(x) = f_{k,{2}}^0(x)=f_{k,{1}}^0(x), \forall x$ and $f_{k,{3}}^1(x) < f_{k,{2}}^1(x)< f_{k,{1}}^1(x)$$,\forall x \in\mathcal{T}_k$, by Lemma \ref{lemma_reshape1}, $\theta_k^r(3) > \theta_k^f(3)=\tilde{\theta}_k$ holds.

By Theorem \ref{theorem5}, $\tilde{\theta}_k> \theta_k^r(2)$ holds. It implies that $\theta_k^r(3) > \theta_k^f(3)$ and $\theta_k^r(3) >\theta_k^r(2)$. 

Consider dynamics with $\pi_{k,t}(\theta_k(t)) = \nu(L_{k,t}(\theta_{k}(t)))$. The following statements hold:

(1) Under Condition \ref{exacerbation}, $L_{a,3}(\theta_{a}^r(3))<L_{a,2}(\theta_{a}^r(2))$ and $L_{b,3}(\theta_{b}^r(3))>L_{b,2}(\theta_{b}^r(2))$ hold, implying $\pi_{a,3}^r(\theta_a^r(3)) > \pi_{a,2}^r(\theta_a^r(2))> \pi_{b,2}^r(\theta_b^r(2)) > \pi_{b,3}^r(\theta_b^r(3))$ and $\frac{\overline{\alpha}_a^r(4)}{\overline{\alpha}_b^r(4)}>\frac{\overline{\alpha}_a^r(3)}{\overline{\alpha}_b^r(3)}$. 

(2) Under Condition \ref{accelerate}, $L_{a,3}(\theta_{a}^r(3))<L_{a,3}(\theta_{a}^f(3))$ and $L_{b,3}(\theta_{b}^r(3))>L_{b,3}(\theta_{b}^f(3))$ hold, implying $\pi_{a,3}^r(\theta_a^r(3)) > \pi_{a,3}^f(\theta_a^f(3))> \pi_{b,3}^f(\theta_b^f(3)) > \pi_{b,3}^r(\theta_b^r(3))$ and  $\frac{\overline{\alpha}_a^r(4)}{\overline{\alpha}_b^r(4)}>\frac{\overline{\alpha}_a^f(4)}{\overline{\alpha}_b^f(4)}$.

(3)  $G_k^1$ (resp. $G_k^0$) experiences the higher (resp. lower) loss at $t=3$ than $t=2$, i.e., $L_{k,3}^1(\theta_k^r(3))>L_{k,2}^1(\theta_k^r(2))$ and $L_{k,3}^0(\theta_k^r(3))<L_{k,2}^0(\theta_k^r(2))$,

 $\bullet$ \textbf{Case (i)}: $g_{k,4}^1<g_{k,3}^1<g_{k,1}^1$ and  $g_{k,4}^0>g_{k,3}^0>g_{k,1}^0$ hold. 

 $\bullet$ \textbf{Case (ii)}: $f_{k,{4}}^0(x) = f_{k,{3}}^0(x)= f_{k,{1}}^0(x), \forall x$ and $f_{k,{4}}^1(x) < f_{k,{3}}^1(x)< f_{k,{1}}^1(x)$$,\forall x \in\mathcal{T}_k$ hold.

\textbf{Induction step:}

Suppose at time $t>3$, $\frac{\overline{\alpha}_a^r(t+1)}{\overline{\alpha}_b^r(t+1)}>\frac{\overline{\alpha}_a^f(t+1)}{\overline{\alpha}_b^f(t+1)}$ and $\frac{\overline{\alpha}_a^r(t+1)}{\overline{\alpha}_b^r(t+1)}>\frac{\overline{\alpha}_a^r(t)}{\overline{\alpha}_b^r(t)}$ hold, and $\forall k\in\{a,b\}$:

 $\bullet$ \textbf{Case (i)}: $g_{k,t+1}^1<g_{k,t}^1<g_{k,1}^1$ and $g_{k,t+1}^0>g_{k,t}^0>g_{k,1}^0$ are satisfied.

 $\bullet$ \textbf{Case (ii)}: $f_{k,{t+1}}^0(x) = f_{k,{t}}^0(x)=f_{k,{1}}^0(x), \forall x$ and $f_{k,{t+1}}^1(x) < f_{k,{t}}^1(x)<f_{k,{1}}^1(x)$$,\forall x \in\mathcal{T}_k$.

Then consider time step $t+1$.

Define pairs $(\tilde{\theta}_a,\tilde{\theta}_b) $ and $(\hat{\theta}_a,\hat{\theta}_b) $ such that the following holds:
\begin{eqnarray*}
	\frac{\overline{\alpha}_a^r(t+1)}{\overline{\alpha}_b^r(t+1)} =  \Psi_{\mathcal{C},t+1}(\theta_a^r(t+1),\theta_b^r(t+1)) 
	>\begin{cases}
		\frac{\overline{\alpha}_a^f(t+1)}{\overline{\alpha}_b^f(t+1)} = \Psi_{\mathcal{C},1}(\theta_a^f(t+1),\theta_b^f(t+1))= \Psi_{\mathcal{C},t+1}(\tilde{\theta}_a,\tilde{\theta}_b)	\\ \frac{\overline{\alpha}_a^r(t)}{\overline{\alpha}_b^r(t)} =\Psi_{\mathcal{C},t}(\theta_a^r(t),\theta_b^r(t)) = \Psi_{\mathcal{C},t+1}(\hat{\theta}_a,\hat{\theta}_b)
	\end{cases}
\end{eqnarray*}
According to the hypothesis, Under \textbf{Case (i)}, $\tilde{\theta}_k> \theta_k^f(t+1)$ and $\hat{\theta}_k> \theta_k^r(t)$ hold by Lemma \ref{lemma_reshape}. Under \textbf{Case (ii)}, $\tilde{\theta}_k> \theta_k^f(t+1)$ and $\hat{\theta}_k> \theta_k^r(t)$ hold by Lemma \ref{lemma_reshape1}. 
By Theorem \ref{theorem5}, $\theta_k^r(t+1)>\tilde{\theta}_k$  and $\theta_k^r(t+1)>\hat{\theta}_k$ hold. It implies that $\theta_k^r(t+1)> \theta_k^f(t+1)$ and $\theta_k^r(t+1)> \theta_k^r(t)$.

(1) Under Condition \ref{exacerbation}, $L_{a,t+1}(\theta_{a}^r(t+1))<L_{a,t}(\theta_{a}^r(t))$ and $L_{b,t+1}(\theta_{b}^r(t+1))>L_{b,t}(\theta_{b}^r(t))$hold, implying $\pi_{a,t+1}^r(\theta_a^r(t+1)) > \pi_{a,t}^r(\theta_a^r(t))> \pi_{b,t}^r(\theta_b^r(t)) > \pi_{b,t+1}^r(\theta_b^r(t+1))$ and $\frac{\overline{\alpha}_a^r(t+1)}{\overline{\alpha}_b^r(t+1)}>\frac{\overline{\alpha}_a^r(t)}{\overline{\alpha}_b^r(t)}$: $\pmb{O}_t$ and $\pmb{O}_{t+1}$ satisfy monotonicity condition and representation disparity get exacerbated. 

(2) Under Condition \ref{accelerate}, $L_{a,t+1}(\theta_{a}^r(t+1))<L_{a,t+1}(\theta_{a}^f(t+1))$ and $L_{b,t+1}(\theta_{b}^r(t+1))>L_{b,t+1}(\theta_{b}^f(t+1))$ hold, implying $\pi_{a,t+1}^r(\theta_a^r(t+1)) > \pi_{a,t+1}^f(\theta_a^f(t+1))> \pi_{b,t+1}^f(\theta_b^f(t+1)) > \pi_{b,t+1}^r(\theta_b^r(t+1))$ and thus $\frac{\overline{\alpha}_a^r(t+1)}{\overline{\alpha}_b^r(t+1)}>\frac{\overline{\alpha}_a^f(t+1)}{\overline{\alpha}_b^f(t+1)}$: the discrepancy between retention rates of two demographic groups is larger at each time compared to the case when distributions are fixed, and if the disparity get exacerbated, this exacerbation is accelerated under the reshaping.

(3)  $G_k^1$ (resp. $G_k^0$) experiences the higher (resp. lower) loss at $t+1$ than $t$,  i.e., $L_{k,t+1}^1(\theta_k^r(t+1))>L_{k,t}^1(\theta_k^r(t))$ and $L_{k,t+1}^0(\theta_k^r(t+1))<L_{k,t}^0(\theta_k^r(t))$. Therefore, 

 $\bullet$ \textbf{Case (i)}: $g_{k,t+2}^1<g_{k,t+1}^1<g_{k,1}^1$ and  $g_{k,t+2}^0>g_{k,t+1}^0>g_{k,1}^0$ hold. 

 $\bullet$ \textbf{Case (ii)}: $f_{k,{t+2}}^0(x) = f_{k,{t+1}}^0(x)=f_{k,{1}}^0(x), \forall x$ and $f_{k,{t+2}}^1(x) < f_{k,{t+1}}^1(x)< f_{k,{1}}^1(x)$$,\forall x \in\mathcal{T}_k$ hold.

Proof is completed.

The case  if $\pi_{a,1}(\theta_a(1))<\pi_{b,1}(\theta_b(1))$ can be proved similarly and is omitted.

\section{Proof of Lemmas for Theorem \ref{thm2}}\label{app_lemmas}
\subsection{Proof of Lemma \ref{lemma_reshape}}
	
	$f_k^0(x)$ and $f_k^1(x)$ overlap over $\mathcal{T}_k:=[\underline{k}^1,\overline{k}^0]$.
	
	1.  $\mathcal{C} := \texttt{StatPar}$
	
	To satisfy $\widehat{\Psi}_{\texttt{StatPar}}(\widehat{\theta}_a,\widehat{\theta}_b) = \widetilde{\Psi}_{\texttt{StatPar}}(\widetilde{\theta}_a,\widetilde{\theta}_b)$, $\frac{\widehat{g}_{k}^1f_k^1(\widehat{\theta}_k)}{\widehat{g}_{k}^0f_k^0(\widehat{\theta}_k)} = \frac{\widetilde{g}_{k}^1f_k^1(\widetilde{\theta}_k)}{\widetilde{g}_{k}^0f_k^0(\widetilde{\theta}_k)} $ should hold. Under Assumption \ref{assumption2}, both $ \frac{\widehat{g}_{k}^1f_k^1(\cdot)}{\widehat{g}_{k}^0f_k^0(\cdot)}$ and $ \frac{\widetilde{g}_{k}^1f_k^1(\cdot)}{\widetilde{g}_{k}^0f_k^0(\cdot)}$ are strictly increasing over $\mathcal{T}_k$. Since $\forall k\in\{a,b\}$, there is $\frac{\widehat{g}_{k}^1f_k^1(\theta_k)}{\widehat{g}_{k}^0f_k^0(\theta_k)} < \frac{\widetilde{g}_{k}^1f_k^1(\theta_k)}{\widetilde{g}_{k}^0f_k^0(\theta_k)},~ \forall \theta_k\in\mathcal{T}_k$. For all three possibilities in Table \ref{table1}, $\widehat{\theta}_k > \widetilde{\theta}_k$ holds $\forall k\in\{a,b\}$. 
	
	2. $\mathcal{C} := \texttt{EqOpt}$
	
	Since $\widetilde{L}_{a}^0(\theta_a) = \widetilde{L}_{b}^0(\theta_b)$ and $\widehat{L}_{a}^0(\theta_a) = \widehat{L}_{b}^0(\theta_b)$ always hold for any $(\theta_a,\theta_b)$ satisfying \texttt{EqOpt} criterion, when change of $\widehat{g}_{k}^0$ (or $\widetilde{g}_{k}^0$) is determined by $\theta_k$ only via $\widehat{L}_{k}^0(\theta_k)$ (or $\widetilde{L}_{k}^0(\theta_k)$), both $\frac{\widehat{g}_{b}^0}{\widehat{g}_{a}^0} = 1$ and $\frac{\widetilde{g}_{b}^0}{\widetilde{g}_{a}^0} = 1$ are satisfied. To satisfy $\widehat{\Psi}_{\texttt{EqOpt}}(\widehat{\theta}_a,\widehat{\theta}_b) = \widetilde{\Psi}_{\texttt{EqOpt}}(\widetilde{\theta}_a,\widetilde{\theta}_b)$, $\frac{\widehat{g}_{k}^1f_k^1(\widehat{\theta}_k)}{\widehat{g}_{k}^0f_k^0(\widehat{\theta}_k)} = \frac{\widetilde{g}_{k}^1f_k^1(\widetilde{\theta}_k)}{\widetilde{g}_{k}^0f_k^0(\widetilde{\theta}_k)} $ should hold, which is same as the condition that should be satisfied in case when $\mathcal{C} := \texttt{StatPar}$. Rest of the proof is thus same as $\texttt{StatPar}$ case and is omitted.

	3. $\mathcal{C} := \texttt{Simple}$

	\texttt{Simple} fairness criterion requires that $\widehat{\theta}_a=\widehat{\theta}_b=\widehat{\theta}$ and $\widetilde{\theta}_a=\widetilde{\theta}_b=\widetilde{\theta}$. In order to satisfy $\widehat{\Psi}_{\texttt{Simple}}(\widehat{\theta}_a,\widehat{\theta}_b) = \widetilde{\Psi}_{\texttt{Simple}}(\widetilde{\theta}_a,\widetilde{\theta}_b)$, $\frac{\widehat{g}_{b}^1f_b^1(\widehat{\theta})-\widehat{g}_{b}^0f_b^0(\widehat{\theta})}{\widehat{g}_{a}^0f_a^0(\widehat{\theta})-\widehat{g}_{a}^1f_a^1(\widehat{\theta})} = \frac{\widetilde{g}_{b}^1f_b^1(\widetilde{\theta})-\widetilde{g}_{b}^0f_b^0(\widetilde{\theta})}{\widetilde{g}_{a}^0f_a^0(\widetilde{\theta})-\widetilde{g}_{a}^1f_a^1(\widetilde{\theta})}$ should hold. Under Assumption \ref{assumption2}, both $\frac{\widehat{g}_{b}^1f_b^1(\cdot)-\widehat{g}_{b}^0f_b^0(\cdot)}{\widehat{g}_{a}^0f_a^0(\cdot)-\widehat{g}_{a}^1f_a^1(\cdot)}$ and $ \frac{\widetilde{g}_{b}^1f_b^1(\cdot)-\widetilde{g}_{b}^0f_b^0(\cdot)}{\widetilde{g}_{a}^0f_a^0(\cdot)-\widetilde{g}_{a}^1f_a^1(\cdot)}$ are strictly increasing over $\mathcal{T}_k$. Since $\forall k\in\{a,b\}$, there is $ \frac{\widehat{g}_{b}^1f_b^1(\theta)-\widehat{g}_{b}^0f_b^0(\theta)}{\widehat{g}_{a}^0f_a^0(\theta)-\widehat{g}_{a}^1f_a^1(\theta)}< \frac{\widetilde{g}_{b}^1f_b^1(\theta)-\widetilde{g}_{b}^0f_b^0(\theta)}{\widetilde{g}_{a}^0f_a^0(\theta)-\widetilde{g}_{a}^1f_a^1(\theta)}$, $\forall \theta \in \mathcal{T}_a\cap \mathcal{T}_b$, implying that $\widehat{\theta}>\widetilde{\theta}$.

\subsection{Proof of Lemma \ref{lemma_reshapeC}}

	Define $\Delta L_k^j = |\widehat{L}_{k}^j(\widehat{\theta}_{k}) - \widetilde{L}_{k}^j(\widetilde{\theta}_{k}) |, j\in\{0,1\}$.
	Rewrite $\widehat{g}_{k}^0 = \widetilde{g}_{k}^0+ \Delta g_k $ and $\widehat{g}_{k}^1 = \widetilde{g}_{k}^1- \Delta g_k $.  For $k\in\{a,b\}$, $\widehat{\theta}_{k}>\widetilde{\theta}_{k}$ holds, which implies that $  \widehat{L}_{k}^1(\widehat{\theta}_{k}) = \widetilde{L}_{k}^1(\widetilde{\theta}_{k}) + \Delta L_k^1$ and $  \widehat{L}_{k}^0(\widehat{\theta}_{k}) = \widetilde{L}_{k}^0(\widetilde{\theta}_{k}) - \Delta L_k^0$. Therefore, 
	\begin{eqnarray*}
		\widehat{L}_{k}(\widehat{\theta}_{k}) - 	\widetilde{L}_{k}(\widetilde{\theta}_{k}) =  \Delta g_k (\widetilde{L}_{k}^0(\widetilde{\theta}_{k}) -\widetilde{L}_{k}^1(\widetilde{\theta}_{k}) ) - ( \widehat{g}_{k}^0\Delta L_k^0 -\widehat{g}_{k}^1\Delta L_k^1), k\in\{a,b\} 
	\end{eqnarray*}
	since 
	\begin{eqnarray*}
		\Delta L_k^1 = \int_{\widetilde{\theta}_{k}}^{\widehat{\theta}_{k}}f_k^1(x)dx ; ~\Delta L_k^0 = \int_{\widetilde{\theta}_{k}}^{\widehat{\theta}_{k}}f_k^0(x)dx 
	\end{eqnarray*}
	Define $\widehat{\delta}_{k}$ such that $\widehat{g}_{k}^0f_k^0(\widehat{\delta}_{k} ) = \widehat{g}_{k}^1f_k^1(\widehat{\delta}_{k} ) $, then $\widehat{g}_{a}^0f_a^0(x) > \widehat{g}_{a}^1f_a^1(x)$ when $x<\widehat{\delta}_{a}$ and $\widehat{g}_b^0f_b^0(x) < \widehat{g}_b^1f_b^1(x)$ when $x>\widehat{\delta}_{b}$. By Lemma \ref{lemma6}, $\widehat{\theta}_{a}< \widehat{\delta}_{a}$ and $\widehat{\theta}_{b}> \widehat{\delta}_{b}$ hold, implying
	\begin{eqnarray*}
		\widehat{g}_{k}^0 \Delta L_k^0 - \widehat{g}_{k}^1 \Delta L_k^1=\int_{\widetilde{\theta}_{k}}^{\widehat{\theta}_{k}} \widehat{g}_{k}^0f_k^0(x) - \widehat{g}_{k}^1f_k^1(x )   dx  \begin{cases}
			> 0, ~k = a\\
			<0,  ~k = b
		\end{cases}
	\end{eqnarray*}
	If $|\Delta g_k (\widetilde{L}_{k}^0(\widetilde{\theta}_{k})-\widetilde{L}_{k}^1(\widetilde{\theta}_{k}))| < |\int_{\widetilde{\theta}_{k}}^{\widehat{\theta}_{k}} \widehat{g}_{k}^0f_k^0(x) - \widehat{g}_{k}^1f_k^1(x )   dx|$ holds, then the sign of $	\widehat{L}_{k}(\widehat{\theta}_{k}) - 	\widetilde{L}_{k}(\widetilde{\theta}_{k})$ is determined by the sign of  $\widehat{g}_{k}^1\Delta L_k^1 -\widehat{g}_{k}^0\Delta L_k^0 $. We have $ \widehat{L}_{a}(\widehat{\theta}_{a}) <	\widetilde{L}_{a}(\widetilde{\theta}_{a})$ and $ \widehat{L}_{b}(\widehat{\theta}_{b}) >	\widetilde{L}_{b}(\widetilde{\theta}_{b})$.

\subsection{Proof of Lemma \ref{lemma_reshape1}}

	1.  $\mathcal{C} := \texttt{StatPar}$ or $\mathcal{C} := \texttt{EqOpt}$
	
	To satisfy $\widehat{\Psi}_{\texttt{StatPar}}(\widehat{\theta}_a,\widehat{\theta}_b) = \widetilde{\Psi}_{\texttt{StatPar}}(\widetilde{\theta}_a,\widetilde{\theta}_b)$ or $\widehat{\Psi}_{\texttt{EqOpt}}(\widehat{\theta}_a,\widehat{\theta}_b) = \widetilde{\Psi}_{\texttt{EqOpt}}(\widetilde{\theta}_a,\widetilde{\theta}_b)$, $  \frac{g_{k}^1\widetilde{f}_{k}^1(\widetilde{\theta}_k)}{g_{k}^0\widetilde{f}_{k}^0(\widetilde{\theta}_k)} =\frac{g_{k}^1\widehat{f}_{k}^1(\widehat{\theta}_k)}{g_{k}^0\widehat{f}_{k}^0(\widehat{\theta}_k)} < \frac{g_{k}^1\widetilde{f}_{k}^1(\widehat{\theta}_k)}{g_{k}^0\widetilde{f}_{k}^0(\widehat{\theta}_k)} $ should hold. Under Assumption \ref{assumption2}, $\frac{g_{k}^1\widetilde{f}_{k}^1(\cdot)}{g_{k}^0\widetilde{f}_{k}^0(\cdot)}$ is strictly increasing over $\mathcal{T}_k$. $ \widehat{\theta}_k> \widetilde{\theta}_k$ has to be satisfied.

	2. $\mathcal{C} := \texttt{Simple}$

	\texttt{Simple} fairness criterion requires that $\widehat{\theta}_a=\widehat{\theta}_b=\widehat{\theta}$ and $\widetilde{\theta}_a=\widetilde{\theta}_b=\widetilde{\theta}$. In order to satisfy $\widehat{\Psi}_{\texttt{Simple}}(\widehat{\theta}_a,\widehat{\theta}_b) = \widetilde{\Psi}_{\texttt{Simple}}(\widetilde{\theta}_a,\widetilde{\theta}_b)$, $ \frac{g_{b}^1\widetilde{f}_{b}^1(\widetilde{\theta})-g_{b}^0\widetilde{f}_{b}^0(\widetilde{\theta})}{g_{a}^0\widetilde{f}_{a}^0(\widetilde{\theta})-g_{a}^1\widetilde{f}_{a}^1(\widetilde{\theta})} =\frac{g_{b}^1\widehat{f}_{b}^1(\widehat{\theta})-g_{b}^0\widehat{f}_{b}^0(\widehat{\theta})}{g_{a}^0\widehat{f}_{a}^0(\widehat{\theta})-g_{a}^1\widehat{f}_{a}^1(\widehat{\theta})}< \frac{g_{b}^1\widetilde{f}_{b}^1(\widehat{\theta})-g_{b}^0\widetilde{f}_{b}^0(\widehat{\theta})}{g_{a}^0\widetilde{f}_{a}^0(\widehat{\theta})-g_{a}^1\widetilde{f}_{a}^1(\widehat{\theta})}$ should hold. Under Assumption \ref{assumption2}, 
	$\frac{{g}_{b}^1\widetilde{f}_b^1(\cdot)-{g}_{b}^0\widetilde{f}_b^0(\cdot)}{{g}_{a}^0\widetilde{f}_a^0(\cdot)-{g}_{a}^1\widetilde{f}_a^1(\cdot)}$ is strictly increasing over $\mathcal{T}_k$. For $\widehat{\theta},\widetilde{\theta}\in\mathcal{T}_a\cap \mathcal{T}_b$, $ \widehat{\theta} > \widetilde{\theta}$ has to be satisfied.

\subsection{Proof of Lemma \ref{lemma_reshape1C}}

	Define $\widehat{\delta}_{k}$ such that $g_{k}^0\widehat{f}_{k}^0(\widehat{\delta}_{k}) = g_{k}^1\widehat{f}_{k}^1(\widehat{\delta}_{k}) $. Then, $g_{a}^0\widehat{f}_{a}^0(x) > g_{a}^1\widehat{f}_{a}^1(x)$ when $x<\widehat{\delta}_{a} $ and $g_{b}^0\widehat{f}_{b}^0(x) < g_{b}^1\widehat{f}_{b}^1(x)$ when $x>\widehat{\delta}_{b} $.
	
	Since $\widehat{\theta}_{k}>\widetilde{\theta}_{k}$, we have
	\begin{eqnarray*}
		\widehat{L}_{k}^0(\widehat{\theta}_{k}) - \widetilde{L}_{k}^0(\widetilde{\theta}_{k}) &=& -\int_{\widetilde{\theta}_{k}}^{\widehat{\theta}_{k}} \widetilde{f}_{k}^0(x)dx=-\int_{\widetilde{\theta}_{k}}^{\widehat{\theta}_{k}} \widehat{f}_{k}^0(x)dx\\
		\widehat{L}_{k}^1(\widehat{\theta}_{k}) - \widetilde{L}_{k}^1(\widetilde{\theta}_{k}) &= & \int_{\widetilde{\theta}_{k}}^{\widehat{\theta}_{k}} \widehat{f}_{k}^1(x)dx  -\int_{\underline{k}^1}^{\widetilde{\theta}_{k}} (\widetilde{f}_{k}^1(x)-  \widehat{f}_{k}^1(x))dx
	\end{eqnarray*}
	Therefore, 
	\begin{eqnarray*}
		\widehat{L}_{k}(\widehat{\theta}_{k}) - \widetilde{L}_{k}(\widetilde{\theta}_{k}) = \int_{\widetilde{\theta}_{k}}^{\widehat{\theta}_{k}} g_k^1\widehat{f}_{k}^1(x)-g_k^0\widehat{f}_{k}^0(x)dx - g_k^1\int_{\underline{k}^1}^{\widetilde{\theta}_{k}} (\widetilde{f}_{k}^1(x)-  \widehat{f}_{k}^1(x))dx
	\end{eqnarray*}
	
	since $\widetilde{\theta}_{a}<\widehat{\theta}_{a}<\widehat{\delta}_{a}$,  $\int_{\widetilde{\theta}_{a}}^{\widehat{\theta}_{a}} g_a^1\widehat{f}_{a}^1(x)-g_a^0\widehat{f}_{a}^0(x)dx < 0$ holds. Since $\widetilde{f}_{a}^1(x)>  \widehat{f}_{a}^1(x)$ for $x\in \mathcal{T}_a$, we have $g_a^1\int_{\underline{a}^1}^{\widetilde{\theta}_{a}} (\widetilde{f}_{a}^1(x)-  \widehat{f}_{a}^1(x))dx > 0$. Therefore, $\widehat{L}_{a}(\widehat{\theta}_{a}) < \widetilde{L}_{a}(\widetilde{\theta}_{a}) $.
	
	When $k = b$, there are two possibilities: (i) $\widetilde{\theta}_{b} < \widehat{\delta}_{b}<\widehat{\theta}_{b}$; (ii) $\widehat{\delta}_{b}<\widetilde{\theta}_{b} <\widehat{\theta}_{b}$.
	
	For case (i), 
	\begin{eqnarray*}
		\widehat{L}_{b}(\widehat{\theta}_{b}) - \widetilde{L}_{b}(\widetilde{\theta}_{b}) &=& \underbrace{ \int_{\widehat{\delta}_{b}}^{\widehat{\theta}_{b}} g_b^1\widehat{f}_{b}^1(x)-g_b^0\widehat{f}_{b}^0(x)dx }_{\textbf{term 1}}+ \underbrace{\int_{\widetilde{\theta}_{b}}^{\widehat{\delta}_{b}} g_b^1\widehat{f}_{b}^1(x)-g_b^0\widehat{f}_{b}^0(x)dx }_{\textbf{term 2}}\\&+& \underbrace{g_b^1\int_{\underline{b}^1}^{\widetilde{\theta}_{b}} ( \widehat{f}_{b}^1(x)-\widetilde{f}_{b}^1(x))dx}_{\textbf{term 3}}
	\end{eqnarray*}
	Since $ \widetilde{\delta}_{b}<\widetilde{\theta}_{b} < \widehat{\delta}_{b}$ and $\widehat{f}_{b}^0(x)=\widetilde{f}_{b}^0(x)$, for $x\in[ \widetilde{\theta}_{b} ,\widehat{\delta}_{b}]$, $g_b^1\widehat{f}_{b}^1(x)-g_b^1\widetilde{f}_{b}^1(x)<g_b^1\widehat{f}_{b}^1(x)-g_b^0\widehat{f}_{b}^0(x) <0$, we have $0>\textbf{term 2} +  \textbf{term 3} >  g_b^1\int_{\underline{b}^1}^{\widehat{\delta}_{b}} ( \widehat{f}_{b}^1(x)-\widetilde{f}_{b}^1(x))dx$. 
	
	Define $\Delta_1= \max_{x\in[\underline{b}^1,\widehat{\delta}_{b}]}|\widehat{f}_{b}^1(x)-\widetilde{f}_{b}^1(x)|$. Since $\textbf{term 1} > 0$, $\widehat{L}_{b}(\widehat{\theta}_{b}) > \widetilde{L}_{b}(\widetilde{\theta}_{b})$ holds only if the following condition is satisfied:
	\begin{eqnarray*}
		\Delta_1 g_b^1(\widehat{\delta}_{b}-\underline{b}^1) < \int_{\widehat{\delta}_{b}}^{\widehat{\theta}_{b}} g_b^1\widehat{f}_{b}^1(x)-g_b^0\widehat{f}_{b}^0(x)dx 
	\end{eqnarray*}
	
	For case (ii), 
	\begin{eqnarray*}
		\widehat{L}_{b}(\widehat{\theta}_{b}) - \widetilde{L}_{b}(\widetilde{\theta}_{b}) &=& \underbrace{ \int_{\widetilde{\theta}_{b}}^{\widehat{\theta}_{b}} g_b^1\widehat{f}_{b}^1(x)-g_b^0\widehat{f}_{b}^0(x)dx }_{\textbf{term 1}}+ \underbrace{g_b^1\int_{\underline{b}^1}^{\widetilde{\theta}_{b}} ( \widehat{f}_{b}^1(x)-\widetilde{f}_{b}^1(x))dx}_{\textbf{term 2}}
	\end{eqnarray*}
	Define $\Delta_2= \max_{x\in[\underline{b}^1,\widetilde{\theta}_{b}]}|\widehat{f}_{b}^1(x)-\widetilde{f}_{b}^1(x)|$. Similar to case (i), $\widehat{L}_{b}(\widehat{\theta}_{b}) > \widetilde{L}_{b}(\widetilde{\theta}_{b})$ holds only if the following condition is satisfied:
	\begin{eqnarray*}
		\Delta_2 g_b^1(\widetilde{\theta}_{b}-\underline{b}^1) < \int_{\widetilde{\theta}_{b}}^{\widehat{\theta}_{b}} g_b^1\widehat{f}_{b}^1(x)-g_b^0\widehat{f}_{b}^0(x)dx 
	\end{eqnarray*}
	Combine two cases, let $\Delta f_b^1 = \max_{x\in[\underline{b}^1,\max\{\widetilde{\theta}_{b},\widehat{\delta}_{b}\}]}|\widehat{f}_{b}^1(x)-\widetilde{f}_{b}^1(x)|$, $\widehat{L}_{b}(\widehat{\theta}_{b}) > \widetilde{L}_{b}(\widetilde{\theta}_{b})$ holds only if the following condition is satisfied:
	\begin{eqnarray*}
		\Delta f_b^1 g_b^1(\max\{\widetilde{\theta}_{b},\widehat{\delta}_{b}\}-\underline{b}^1) < \int_{\max\{\widetilde{\theta}_{b},\widehat{\delta}_{b}\}}^{\widehat{\theta}_{b}} g_b^1\widehat{f}_{b}^1(x)-g_b^0\widehat{f}_{b}^0(x)dx 
	\end{eqnarray*}

\section{More on examples of finding proper fairness constraints from dynamics}\label{app_example}

\begin{example}
	\text{[Linear first order model]} is given by $N_k(t+1) =N_k(t) \pi_{k}^2(\theta_k(t))+\beta_k\pi_{k}^1(\theta_k(t))$. This is a general form of dynamics \eqref{eq:dynamic} where the arrivals can also depend on the decision. When $\pi_{k}^1(\theta_k(t))=1$, then dynamics model will be reduced to \eqref{eq:dynamic}.  $\tilde{N}_k  = \frac{\beta_k\pi_{k}^1(\theta_k)}{1 - \pi_{k}^2(\theta_k)}$ is the stable fixed point if $\pi_{k}^2(\theta_k)<1$ holds. Since $|\frac{\tilde{N}_a}{\tilde{N}_b}-\frac{\beta_a}{\beta_b}| = \frac{\beta_a}{\beta_b}|\frac{\pi_{a}^1(\theta_a)}{\pi_{b}^1(\theta_b)}\frac{1 -\pi_{b}^2(\theta_b)}{1 -\pi_{a}^2(\theta_a)}-1 |$, solution  pair $(\theta_a^*,\theta_b^*)$ should satisfy $\frac{\pi_{a}^1(\theta_a^*)}{1 -\pi_{a}^2(\theta_a^*)}=\frac{\pi_{b}^1(\theta_b^*)}{1 -\pi_{b}^2(\theta_b^*)}$. The constraint set that can sustain the group representation is given by: $$\mathcal{C} = \{(\theta_a,\theta_b)|(\theta_a,\theta_b)\in\Theta\times \Theta, \frac{\pi_{a}^1(\theta_a)}{1 -\pi_{a}^2(\theta_a)}=\frac{\pi_{b}^1(\theta_b)}{1 -\pi_{b}^2(\theta_b)}, \pi_{a}^2(\theta_a)<1, \pi_{b}^2(\theta_b)<1\}.$$
\end{example}
	Consider the case where departure is driven by positive rate $\pi_{k}^2(\theta_k) =\nu( \int_{\theta_k}^{\infty}f_k(x)dx )$ and arrival is driven by error rate $\pi_{k}^1(\theta_k) = \nu(g_k^0\int_{\theta_k}^{\infty}f_k^0(x)dx + g_k^1\int_{-\infty}^{\theta_k}f_k^1(x)dx) = \nu(L_k(\theta_k))$ where $\nu(\cdot)$ is a strictly decreasing function. This can be applied in lending scenario, where an applicant will stay as long as he/she gets the loan (positive rate) regardless of his/her qualification. Since an unqualified applicant who is issued the loan cannot repay, his/her credit score will be decreased which lowers the chance to get a loan in the future \cite{pmlr-v80-liu18c}. Therefore, users may decide whether to apply for a loan based on the error rate. 
	
	In Fig. \ref{fig:example}, $\Delta$-fair set is illustrated for the case when $f_k^j(x), ~k\in\{a,b\}, j\in\{0,1\}$ is truncated normal distributed with parameters $[\sigma_a^0,\sigma_a^1,\sigma_b^0,\sigma_b^1] = [5,6,6,5]$, $[\underline{k}^0,\underline{k}^1,\overline{k}^0,\overline{k}^1] = [5,11,20,35]$, $[\mu_k^0,\mu_k^1] = [10,25]$ for $k\in\{a,b\}$ and $\nu(x)=1-x$. The left heat map illustrates the $\Delta$-fair set for the dynamics model mentioned above. On the other hand, the right heat map illustrates the dynamics model introduced in Section \ref{subsec:dynamic} where the departure is driven by model accuracy, i.e., $\pi_{k}^2(\theta_k) = \nu(L_k(\theta_k))$ and $\pi_{k}^1(\theta_k) =1$. Here, $x$-axis and $y$-axis represent $\theta_b$ and $\theta_a$ respectively. Each pair $(\theta_a,\theta_b)$ has a corresponding value of  $|\frac{\tilde{N}_a}{\tilde{N}_b}-\frac{\beta_a}{\beta_b}|$ measuring how well it can sustain the group representation. The colored area illustrates all the pairs such that $|\frac{\tilde{N}_a}{\tilde{N}_b}-\frac{\beta_a}{\beta_b}|\leq \frac{\beta_a}{\beta_b}$.  All $(\theta_a,\theta_b)$ pairs that have the same value of $|\frac{\tilde{N}_a}{\tilde{N}_b}-\frac{\beta_a}{\beta_b}| =  \frac{\beta_a}{\beta_b}\epsilon$ form a curve of the same color, where the corresponding value of $\epsilon\in[0,1]$ is shown in the color bar. $\Delta$-fair set is the union of all curves with $\epsilon \leq \Delta \frac{\beta_b}{\beta_a}$.

\begin{example}
\text{[Quadratic first order model]} is given by $N_k(t+1) =(N_k(t))^2 \pi_k^1(\theta_k(t))+\beta_k$. $\tilde{N}_k = \frac{1}{2\pi_k^1(\theta_k)} - \sqrt{\frac{1}{4(\pi_k^1(\theta_k))^2} - \frac{\beta_k}{\pi_k^1(\theta_k)}}$ is the stable fixed point if $\pi_k^1(\theta_a)<\frac{1}{4\beta_k}$ holds. Since $|\frac{\tilde{N}_a}{\tilde{N}_b}-\frac{\beta_a}{\beta_b}| = \frac{\beta_a}{\beta_b}|\frac{\beta_b\pi_b^1(\theta_b)}{\beta_a\pi_a^1(\theta_a)}\frac{1-\sqrt{1-4\beta_a\pi_a^1(\theta_a)}}{1-\sqrt{1-4\beta_b\pi_b^1(\theta_b)}}-1 |$, then $\beta_a\pi_a^1(\theta_a^{*}) = \beta_b\pi_b^1(\theta_b^{*})$ should be satisfied.  The constraint set that can sustain the group representation is given by$$\mathcal{C} = \{(\theta_a,\theta_b)|(\theta_a,\theta_b)\in\Theta\times \Theta, \beta_a\pi_a^1(\theta_a)=\beta_b\pi_b^1(\theta_b), \pi_a^1(\theta_a)<\frac{1}{4\beta_a}, \pi_b^1(\theta_b)<\frac{1}{4\beta_b}\}.$$
\end{example}

\section{Supplementary Material for the Experiments}\label{supp:experiment}
\subsection{Parameter settings}\label{supp:param}
$f_k^j(x)$ follows the truncated normal distribution, the supports of $f_k^j(x), k\in\{a,b\},j\in\{0,1\}$ are $[\underline{a}^0,\underline{a}^1,\overline{a}^0,\overline{a}^1] = [-8,5,19,35]$, $[\underline{b}^0,\underline{b}^1,\overline{b}^0,\overline{b}^1] = [-6,25,9,43]$, with the means $[\mu_a^0,\mu_a^1,\mu_b^0,\mu_b^1] = [4,20,8,27]$ and standard deviations $[\sigma_a^0,\sigma_a^1,\sigma_b^0,\sigma_b^1] = [5,6,3,6]$. The label proportions are $g_a^0 = 0.4$, $g_b^0 = 0.6$. The dynamics \eqref{eq:dynamic} uses $\nu(x) = 1-x$.

\subsection{Illustration of convergence of sample paths}

\begin{figure}[h]
	\centering  
	\subfigure[Group proportion]{\label{fig4:b}\includegraphics[trim={0cm 0cm 0cm 0cm},clip=true,width=0.45\textwidth]{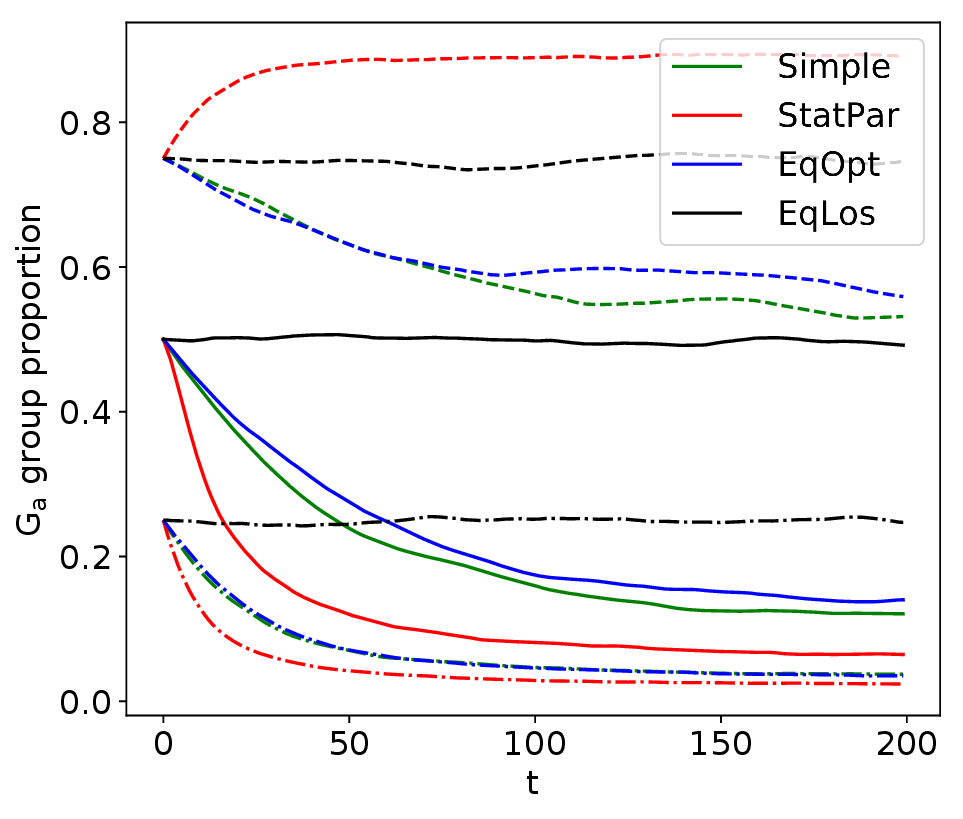}}
	\subfigure[Average total loss ]{\label{fig4:c}\includegraphics[trim={0cm 0cm 0cm 0cm},clip=true,width=0.45\textwidth]{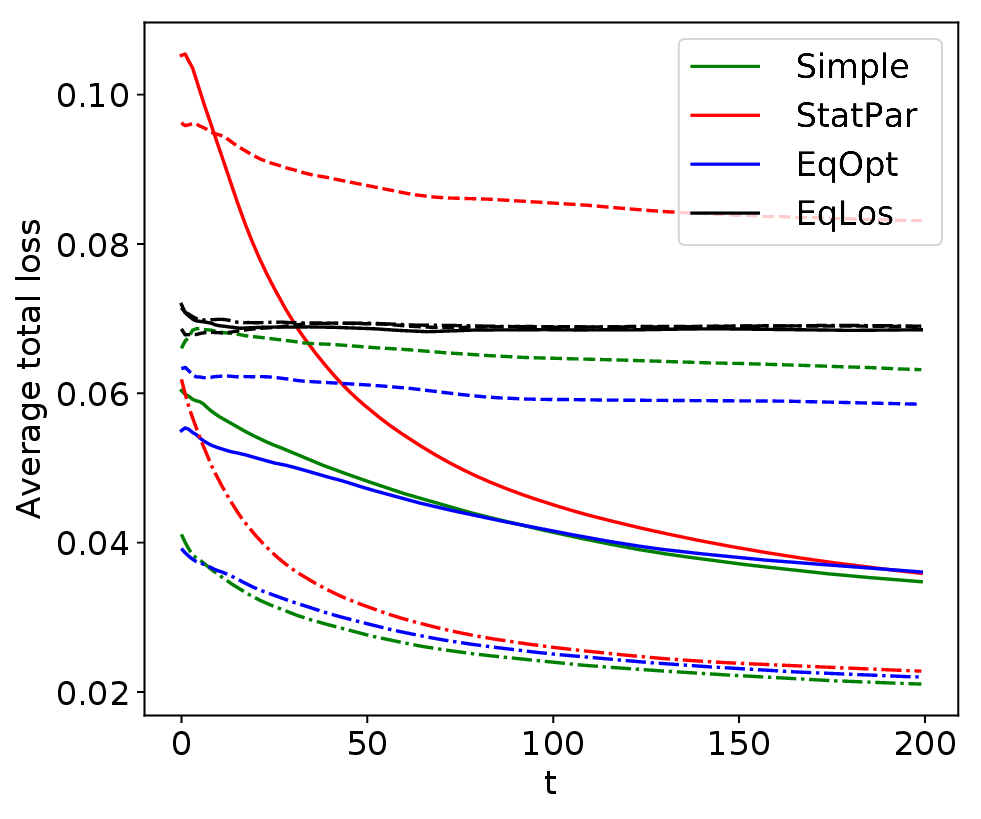}}
	\caption{Sample paths for truncated normal example under different fairness criteria when $\beta_a + \beta_b = 20000$. 
		Group proportion $\overline{\alpha}_a(t)$ and average total loss are shown in Fig.\ref{fig4:b}\ref{fig4:c} respectively: solid lines are for the case $\beta_a = \beta_b$, dashed lines for $\beta_a = 3\beta_b$, and dotted dashed lines for $\beta_a = \beta_b/3$.}
	\label{fig4}
\end{figure}

Fig. \ref{fig4} shows sample paths of the group proportion and average total loss using one-shot fair decisions and different combinations of 
$\beta_a$, $\beta_b$ under dynamics with $\pi_{k,t}(\cdot) = \nu(L_{k,t}(\cdot))$. In all cases convergence is reached (we did not include the decisions $\theta_k(t)$ but convergence holds there as well).  In particular, under \texttt{EqLos} fairness, the group representation is sustained throughout the horizon. By contrast, under other fairness constraints, even a ``major'' group (one with a larger arrival $\beta_k$) can be significantly marginalized over time (blue/green dashed line in Fig. \ref{fig4:b}). This occurs when the loss of the minor group happens to be smaller than that of the major group, which is determined by feature distributions of the two groups (see Fig. \ref{fig55}). Whenever this is the case, the one-shot fair decision will seek to increase the minor group's proportion in order to drive down the average loss. 

\begin{figure}[h]
	\centering  
	\subfigure[Feature distributions illustration]{\label{fig55:a}\includegraphics[trim={0cm 0cm 0cm 0cm},clip=true,width=0.45\textwidth]{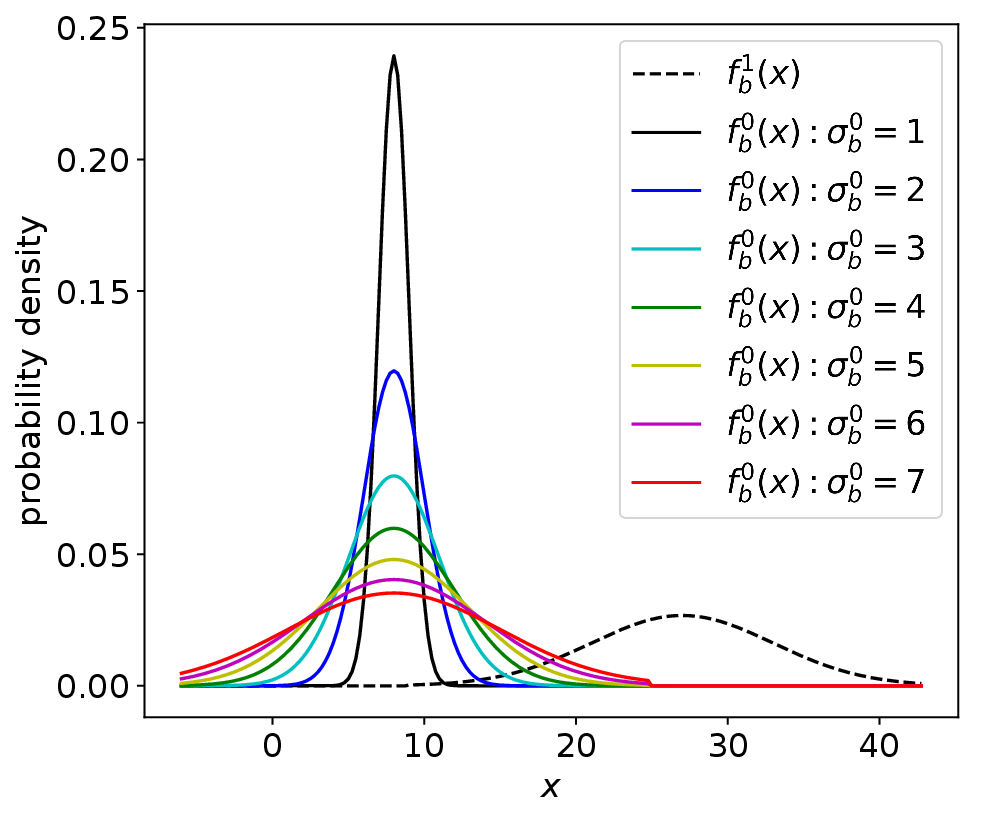}}
	\subfigure[Group proportion $\beta_a=\beta_b$ ]{\label{fig55:b}\includegraphics[trim={0cm 0cm 0cm 0cm},clip=true,width=0.45\textwidth]{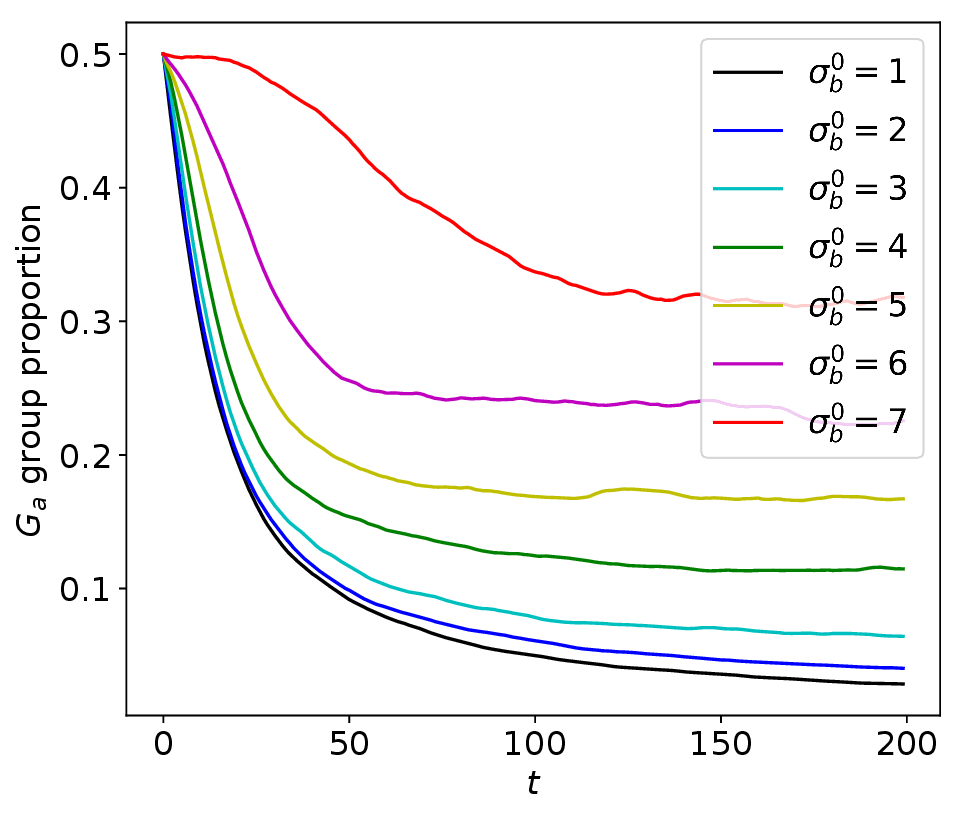}}
	\caption{Change $f_b^0(x)$ by varying $\sigma_b^0 \in \{1,2,3,4,5,6,7\}$. As $\sigma_b^0$ increases, the overlap area with $f_b^1(x)$ also increases as shown in Fig. \ref{fig55:a}.  Fig. \ref{fig55:b} shows the result under \texttt{StatPar} fairness. Given $\theta_a(t)$, the larger $\sigma_b^0$ results in the larger $L_b(\theta_b(t))$ and thus the smaller $G_b$'s retention rate. 
	}
	\label{fig55}
\end{figure}

\subsection{Dynamics driven by other factors}
To sustain the group representation, the key point is that the fairness definition should match the factors that drive user departure and arrival. If adopt different dynamic models, different fairness criteria should be adopted. Two examples with different dynamics and the performance of four fairness criteria are demonstrated in Fig. \ref{fig10}.

\begin{figure}[h]
	\centering   
	\subfigure[Users from $G_k$ are driven by false negative rate]{\label{fig10:b}\includegraphics[width=0.45\textwidth]{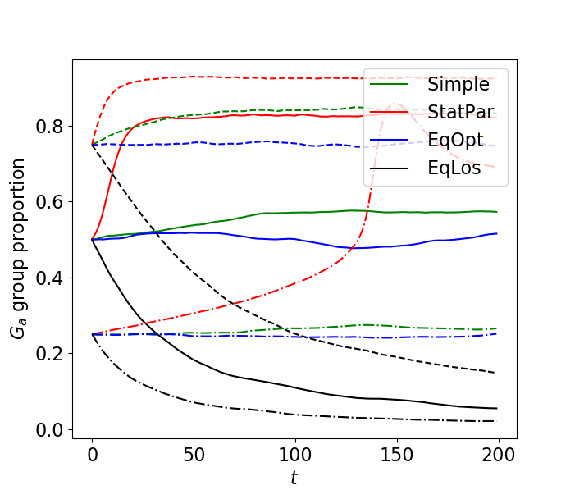}}
	\subfigure[Users from $G_k^j$ are driven by their own perceived loss ]{\label{fig10:a}\includegraphics[trim={0cm 0cm 0cm 0cm},clip=true,width=0.43\textwidth]{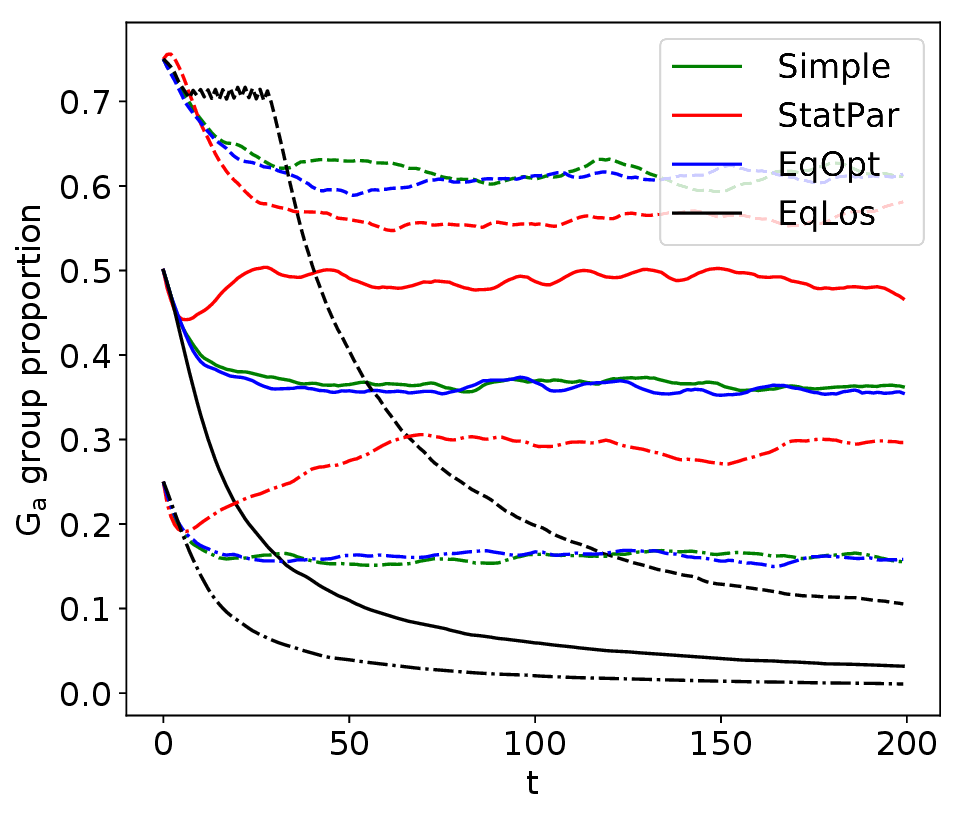}}
	\rev{	\caption{Sample paths under different dynamic models: Three cases are demonstrated including $\beta_a = \beta_b$ (solid curves); $\beta_a = 3\beta_b$ (dashed curves); $\beta_a = \beta_b/3$ (dotted dash curves). Fig. \ref{fig10:b} illustrates the model where the user departure is driven by false negative rate: $N_k(t+1) = N_k(t)\nu(\text{FN}_k(\theta_k(t))) + \beta_k$,  with $\text{FN}_k(\theta_k(t)) = \int_{\theta_k(t)}^{\infty}f_k^0(x)dx $. Under this model \texttt{EqOpt} is better at maintaining representation. 
			Fig. \ref{fig10:a} illustrates the model where the users from each sub-group $G_k^j$ are driven by their own perceived loss: $N_k^j(t+1) = N_k^j(t)\nu(L_k^j(\theta_k(t))) +g_k^j\beta_k$, with $L_k^j(\theta_k)$ being false positives for $j=0$ and false negatives for $j=1$.  Under this model none of the four criteria can maintain group representation. 
		}
		\label{fig10}}
\end{figure}

\subsection{When distributions are learned from users in the system}

If $f_k^j(x)$ is unknown to the decision maker and the decision is learned from users in the system, then as users leave the system  the decision can be more inaccurate and the exacerbation could potentially get more severe. In order to illustrate this, we first modify the dynamic model such that the users' arrivals are also effected by the model accuracy,\footnote{The size of one group can decrease in this case, while the size of two groups is always increasing for the dynamic \eqref{eq:dynamic}. } i.e., $N_k(t+1) = (N_k(t)+\beta_k )\nu(L_k(\theta_k(t))) $. We compare the performance of two cases: \textit{(i)} the Bayes optimal decisions are applied in every round; and \textit{(ii)} decisions in $(t+1)$th round are learned from the remaining users in $t$th round.  The empirical results are shown in Fig. \ref{fig7} where each solid curve (resp. dashed curve) is a sample path of case \textit{(i)} (resp. case \textit{(ii)}). Although $\beta_a=\beta_b$, $G_b$ suffers a smaller loss at the beginning and starts to dominate the overall objective gradually. It results in the less and less users from group $G_a$ than $G_b$ in the sample pool and the model trained from minority group $G_a$ suffers an additional loss due to its insufficient samples. In contrast, as $G_b$ dominates more in the objective and its loss may be decreased compared with the case \textit{(i)} (See Fig. \ref{fig7:c}). As a consequence, the exacerbation in group representation disparity gets more severe (See Fig. \ref{fig7:a}).

\begin{figure}[h]
	\centering   
	\subfigure[Group proportion]{\label{fig7:a}\includegraphics[trim={0cm 0cm 0cm 0cm},clip=true,width=0.32\textwidth]{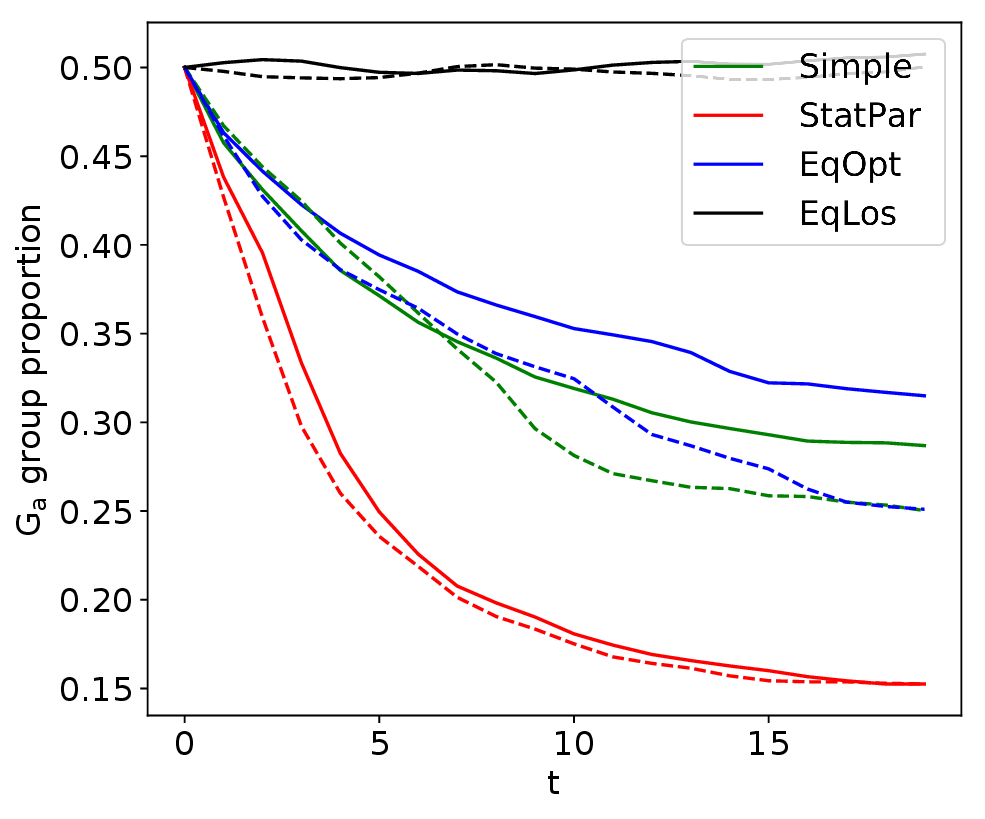}}
	\subfigure[$G_a$'s total population]{\label{fig7:b}\includegraphics[trim={0cm 0.1cm 0cm 0.2cm},clip=true,width=0.33\textwidth]{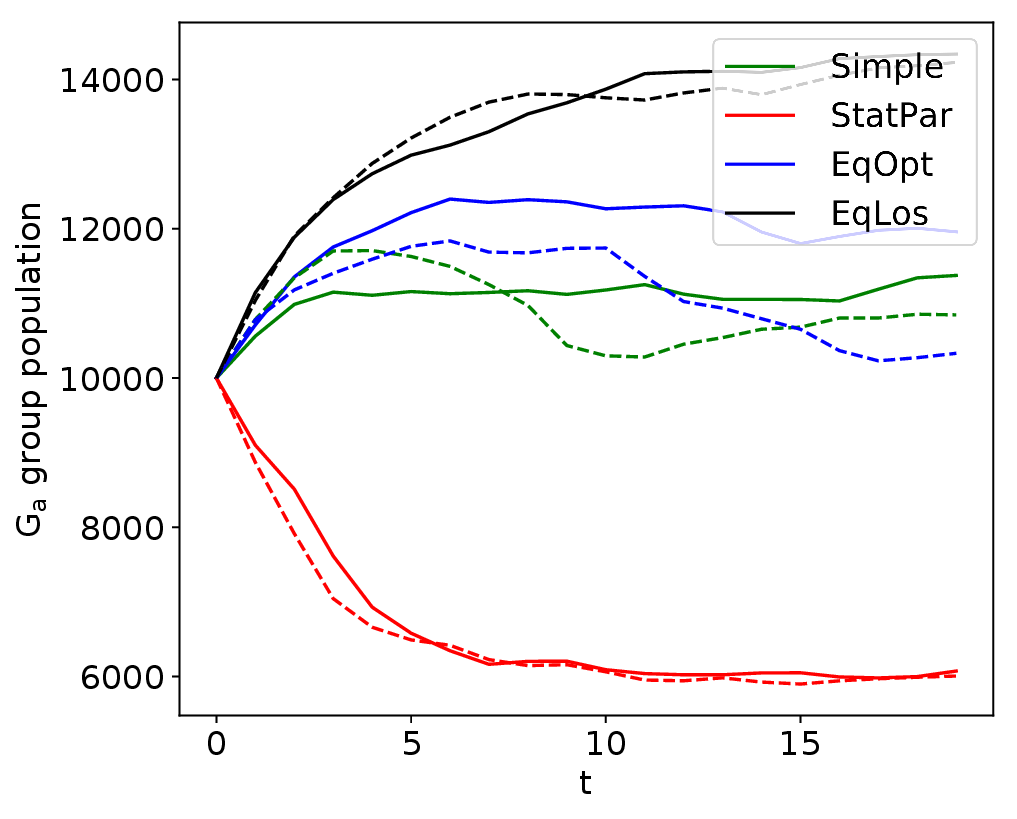}}
	\subfigure[$G_b$'s total population]{\label{fig7:c}\includegraphics[trim={0cm 0.1cm 0cm 0.2cm},clip=true,width=0.33\textwidth]{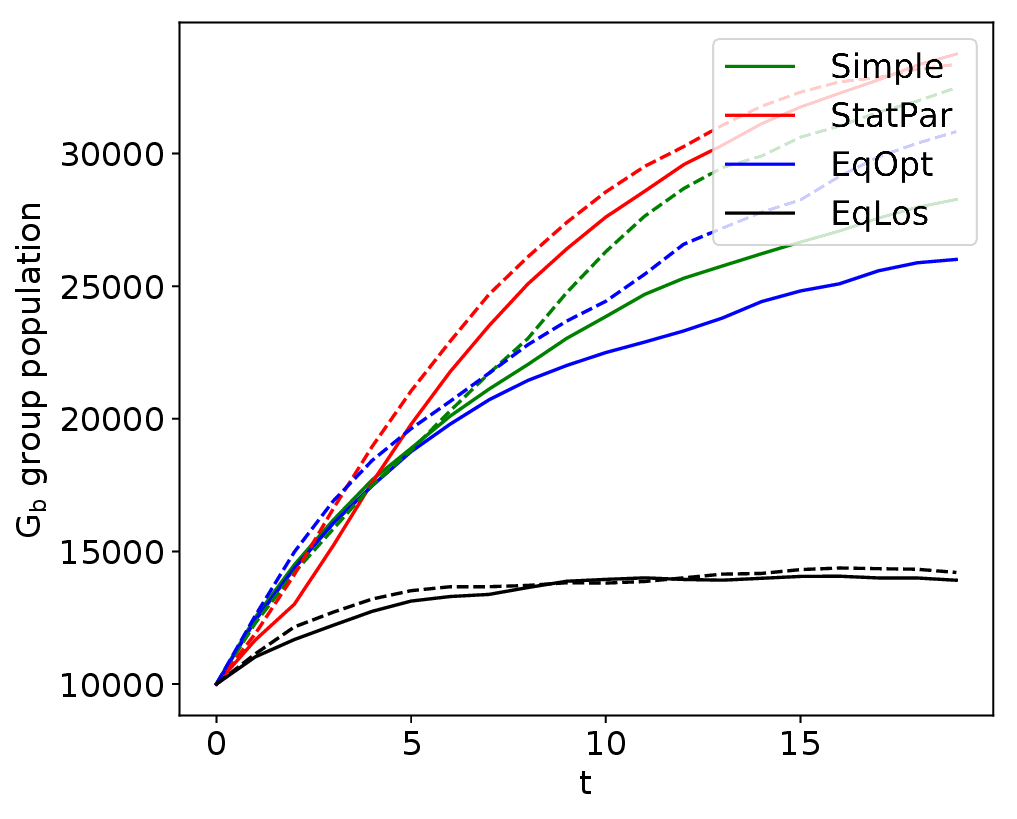}}
	\rev{\caption{Impact of the classifier's quality: dashed curves represent the results for decisions learned from users (case \textit{(ii)}), solid curves represent the results for Bayes optimal decisions (case \textit{(i)}). It shows the exacerbation of group disparity get more severe under case \textit{(ii)} for \texttt{Simple}, \texttt{EqOpt} and \texttt{StatPar} criteria.}
		\label{fig7}}
\end{figure}

\end{document}